\documentclass[10pt]{article}
\usepackage{amsmath,amsfonts,amsthm,mathrsfs,url}
\usepackage{color}
\usepackage{fancyhdr}
\usepackage{amscd}
\usepackage[top=2.3cm, bottom=2.5cm, left=2.1cm,
right=2.1cm]{geometry}
\usepackage{bm}
\usepackage{enumitem,amsthm,amsmath,amsfonts,amssymb}
\usepackage{float}
\usepackage{mathtools}
\usepackage{algorithmic}
\usepackage{algorithm}
\usepackage{xcolor}
\usepackage{graphicx}
\usepackage{subcaption}
\usepackage{tikz}
\usepackage{multirow}
\usepackage{multicol}
\usepackage{hyperref}
\usepackage{mathrsfs, fancyhdr}
\usepackage{multirow}
\usepackage{booktabs}

\newcommand{\Z}{\mathbb{Z}}

\newcommand{\N}{\mathbb{N}}

\newcommand{\D}{\mathcal{D}}
\newcommand{\E}{\mathbb{E}}

\newcommand{\R}{\mathbb{R}}

\newcommand\numberthis{\addtocounter{equation}{1}\tag{\theequation}}

\newtheorem{cor}{Corollary}

\newtheorem{prop}{Proposition}
\newtheorem{rem}{Remark}

\def\bw{\mathbf{w}}
\def\O{\mathcal{O}}

\def\W{\mathcal{W}}

\def\Z{\mathcal{Z}}
\def\N{\mathcal{N}}

\def\bz{\mathbf{z}}

\def\EX{\mathbb{E}}
\def\S{\mathcal{S}}
\def\bw{\mathbf{w}}

\def\A{\mathcal{A}}

\def\bI{\mathbb{I}}
\def\0{\mathbf{0}}

\def\bv{\mathbf{v}}

\def\proj{\text{Proj}}

\def\ebb{\mathbb{E}}
\def\ibb{\mathbb{I}}

\newtheorem{theorem}{Theorem}
\newtheorem{lemma}[theorem]{Lemma}
\newtheorem{proposition}[theorem]{Proposition}
\newtheorem{corollary}[theorem]{Corollary}
\theoremstyle{definition}
\newtheorem{definition}{Definition}

\newtheorem{assumption}{Assumption}
\newtheorem{example}{Example}
\theoremstyle{definition}

\newtheorem{remark}{Remark}

\def\begeqn{\begin{equation}}
\def\endeqn{\end{equation}}
\def\begth{\begin{theorem}}
\def\endth{\end{theorem}}
\def\begprop{\begin{proposition}}
\def\endprop{\end{proposition}}
\def\begcor{\begin{corollary}}
\def\endcor{\end{corollary}}
\def\begdef{\begin{definition}}
\def\enddef{\end{definition}}
\def\beglemm{\begin{lemma}}
\def\endlemm{\end{lemma}}
\def\begexm{\begin{example}}
\def\endexm{\end{example}}
\def\begrem{\begin{remark}}
\def\endrem{\end{remark}}

\def\begdef{\begin{definition}}
\def\enddef{\end{definition}}

\def\Pr{\text{Pr}}

\def\proj{\text{Proj}}
\def\emp{\text{emp}}
\def\bwi{\bw^{(i)}}
\def\bvi{\bv^{(i)}}
\def\V{\mathcal{V}} 
\def\wcal{\mathcal{W}}
\def\vcal{\mathcal{V}}

\begin{document}

\title{Stability and Generalization for Markov Chain Stochastic Gradient Methods$^\dag$}

%\maketitle
\author{Puyu Wang$^1 $\quad  Yunwen Lei$^2$\quad Yiming Ying$^{3*}$\quad Ding-Xuan Zhou$^4$ \\ 
$^{1}$ Liu Bie Ju Centre for Mathematical Sciences, City University of Hong Kong \\
$^{2}$ School of Computer Science, University of Birmingham \\
$^{3}$ Department of Mathematics and Statistics, State University of New York at Albany \\
$^{4}$ School of Mathematics and Statistics, University of Sydney \\
puyuwang@cityu.edu.hk,  y.lei@hbam.ac.uk, \\
yying@albany.edu, dingxuan.zhou@sydney.edu.au}

 \date{}

\maketitle

\begin{abstract}
Recently there is a large amount of work devoted to the study of Markov chain stochastic gradient methods (MC-SGMs)  which mainly focus on their convergence analysis   for solving minimization problems. In this paper, we provide a comprehensive generalization analysis of MC-SGMs for both minimization and minimax problems through the lens of algorithmic stability in the framework of statistical learning theory. For empirical risk minimization (ERM) problems, we establish the optimal excess population risk bounds for both smooth and non-smooth cases by introducing on-average argument stability. For minimax problems, we develop a quantitative connection between on-average argument stability and generalization error which extends the existing results for uniform stability \cite{lei2021stability}. 
We further develop the first nearly optimal convergence rates for convex-concave problems both in expectation and with high probability, which, combined with our stability results, show that the optimal generalization bounds can be attained for both smooth and non-smooth cases. 
To the best of our knowledge, this is the first generalization analysis of SGMs when the gradients are sampled from a Markov process.  \let\thefootnote\relax\footnote{$^\dag$To appear in NeurIPS 2022. 
$^{*}$The corresponding author is Yiming Ying. }  

\end{abstract}

\bigskip

%%%%%%%%%%%%%%%%%%%%%%%%%%%%%%%%%%%%%%%%%%%%%%%%%%%%%%%%%%%%%%%%%%%%%%
%%%%%%%%%%%%%%%%%%%%%%%%%%%%%%%%%%%%%%%%%%%%%%%%%%%%%%%%%%%%%%%%%%%%%
\parindent=0cm
%%%%%%%%%%%%%%%%%%%%%%%%%%%%%%%%%%%%%

\section{Introduction}
 
Stochastic gradient methods (SGMs)  have been the workhorse behind the success of many machine learning (ML) algorithms due to their simplicity and high efficiency. As opposed to the deterministic (full) gradient methods, SGMs only require a small batch of random example(s) to update the model parameters at each iteration,  making them amenable for solving large-scale problems. 

There are mainly two notable types of SGMs which are inherent for different learning problems.  In particular, stochastic gradient descent (SGD) is widely used for solving the  empirical risk minimization (ERM) problem and the  theoretical convergence has been extensively studied   \cite{bach2013non,dieuleveut2016nonparametric,duchi2009efficient,lan2020first,li2019convergence,lin2016optimal,nedic2014stochastic,rakhlin2012making,vaswani2019fast, ying2008online,ying2006online,zhang2004solving}. Concomitantly, %the minimax problem involves the minimization over the primal variable $\bw$ and maximization over the dual variable $\bv.$  
the minimax problems instantiate many ML problems such as Generative Adversarial Networks (GANs) \cite{arjovsky2017wasserstein,goodfellow2014generative}, AUC maximization \cite{gao2013one,liu2019stochastic,ying2016stochastic}, and algorithmic fairness \cite{diana2021minimax,li2019fair,mohri2019agnostic,martinez2020minimax}.     Stochastic gradient descent ascent (SGDA) is an off-the-shelf algorithm for solving minimax problems. The convergence of SGDA and its variants is also widely studied in the literature \cite{lin2020gradient,madry2017towards,nemirovski2009robust,rafique2021weakly}.

On the other important front, the ultimate goal of learning is to achieve good generalization from the training data to the unknown test data. Along this line, generalization analysis of SGMs has attracted considerable attention using the algorithmic stability approach \cite{bousquet2002stability,elisseeff2005stability}. In particular,  stability and generalization of SGD have been studied using the uniform argument stability \cite{bassily2020stability,bassily2019private,charles2018stability,hardt2016train,kuzborskij2018data} and on-average  stability \cite{lei2020fine,lei2021sharper}.  In \cite{farnia2021train,lei2021stability,zhang2021generalization}, different stability and generalization measures are investigated for SGDA  under both convex-concave and non-convex-non-concave settings. A critical assumption in most of the above studies about SGD and SGDA is the {\em i.i.d. sampling scheme} where the randomly sampled mini-batch or datum at each iteration  is i.i.d.  drawn from the given training data, guaranteeing that the stochastic gradient is an unbiased estimator of the true   gradient. 

Markov chain naturally appears in many important problems, such as decentralized consensus optimization, which finds applications in various areas including wireless sensor networks, smart grid implementations and distributed statistical learning \cite{ayache2021private,chen2014dictionary,duchi2011dual,lopes2007incremental,mao2020walkman,ram2009incremental,Sayed,shah2018linearly} as well as pairwise learning  \cite{yang2021simple} which instantiates AUC maximization \cite{agarwal2009generalization,gao2013one,liu2018fast,ying2016stochastic,zhao2011online} and metric learning \cite{kulis2013metric,weinberger2009distance,xing2002distance,ying2012distance}. 
A common example is a distributed system in which each node stores a subset of the whole data, and one aims to train a global model based on these data. We let a central node that stores all model parameters walk randomly over the system, in which case the samples are accessed according to a Markov chain. Several works studied this kind of model \cite{johansson2007simple,johansson2010randomized,lopes2007incremental,mao2020walkman,ram2009incremental}.  
Markov chains also arise extensively in thermodynamics, statistical mechanics dynamic systems and so on \cite{rauch1965maximum,smyth1994hidden}. 
In addition, it was   observed in \cite{Sun2018Markov,yang2021simple} that SGD with Markov chain sampling (MC-SGD) performs more efficiently than SGD with the common i.i.d.  sampling scheme  in various cases.
Hence, studying the performance of MC-SGMs has  certain theoretical and application values. 
The key difference from the i.i.d. sampling scheme is that the stochastic gradient at each iteration is sampled on the trajectory of a Markov chain, in which the stochastic gradient estimators are neither  unbiased  nor independent.    Recent studies \cite{ayache2021private,doan2020convergence,duchi2012ergodic,johansson2007simple,johansson2010randomized,ram2009incremental,sun2019decentralized} overcame this technical hurdle and provided the convergence rates of MC-SGD.   However, to the best of our knowledge, there is no work on the generalization performance of SGMs with Markov sampling. 

\noindent{\bf Main contribution:}
In this paper, we provide a comprehensive study of the stability and generalization for both SGD and SGDA with Markov sampling in the framework of statistical learning theory \cite{vapnik1999nature,bousquet2003introduction}.  Our main contribution can be summarized as follows. 
 
\hspace*{2mm}$\bullet$  We develop stability and generalization results of MC-SGD for solving ERM problems in both smooth and non-smooth cases. In particular,  we show that MC-SGD can achieve competitive stability results as SGD with i.i.d.  sampling scheme.  By trading off the generalization and optimization errors appropriately,  we establish the first-ever-known excess generalization bound $\O(1/\sqrt{n})$ for MC-SGD     where $n$ is the size of training data. The key idea for handling Markov sampling structure of MC-SGD is to use the concept of on-average argument stability.  
  
 \hspace*{2mm} $\bullet$ We first establish the connection between on-average argument stability and generalization for minimax optimization algorithms, which extends the existing work on uniform argument stability \cite{lei2021stability}. We further develop stability bounds of SGDA with Markov sampling (MC-SGDA) for both smooth and non-smooth cases and obtain the nearly optimal convergence rates $\tilde{\O}(1/\sqrt{T})$ for convex-concave problems in the form of both expectation and high probability, where  $T$ is the number of iterations, from which its   optimal population risk bound  is established. Specifically,  we consider several measures of generalization performance and show that the optimal population risk bounds $\O(1/\sqrt{n})$ can be derived even in the non-smooth case. %To the best of our knowledge, this is the first work on  stability and generalization  of SGDA under Markov chain setting.

 \hspace*{2mm}$\bullet$ To the best of our knowledge, this is the first-ever-known  work on  stability and generalization  of SGD and SGDA under the Markov chain setting. Our results show that, despite the stochastic gradient estimator is biased and dependent across iterations due to the Markov sampling scheme, the generalization performance of MC-SGD and MC-SGDA enjoys the same optimal excess generalization rates as the i.i.d.  sampling setting. 

%\textbf{  Stability and Generalization Bounds for MC-SGD.}\textbf{Stability and Generalization Bounds for MC-SGDA.}\textbf{MC-SGMs is competitive compared to SGMs with i.i.d.  sampling.}

\noindent{\bf Organization of the paper:} %The paper is organized as follows. 
We discuss the related work in Subsection \ref{sec:relatedwork} and formulate the problem in Section \ref{sec:preliminary}. 
Section~\ref{sec:MC-SGD} presents the stability and generalization results of MC-SGD for both smooth and non-smooth losses. 
Section~\ref{sec:MC-SGDA} develops the first nearly optimal convergence rates  for convex-concave problems
of MC-SGDA, and show that the optimal risk bounds   can be derived  in both smooth and  non-smooth cases. Section \ref{sec:conclu} concludes the paper.

\subsection{Related Work}\label{sec:relatedwork}
 
In this subsection, we review some further works which are closely related to our paper. 

\noindent\textbf{Algorithmic Stability}.
Algorithmic stability characterizes the sensitivity of a learning algorithm  when the inputs to the algorithm are slightly perturbed.   
The framework of algorithmic stability was established in a seminal paper \cite{bousquet2002stability} for the exact minimizer of the ERM problem, where the uniform stability was established for strongly convex objective functions.  Recent work \cite{bousquet2020sharper,feldman2018generalization,feldman2019high} derived sharper generalization bounds for uniformly stable algorithms with high probability. Several other stability measures were later developed for studying the generalization of different learning algorithms  including the hypothesis stability~\cite{bousquet2002stability}, on-average stability~\cite{shalev2010learnability}, argument stability~\cite{liu2017algorithmic} and total variation stability~\cite{bassily2021algorithmic,ullah2021machine}.
%The generalization performance of stochastic optimization algorithms mainly controlled by optimization error and generalization error \cite{bousquet2008tradeoffs}. 
%Optimization error is induced by running algorithms to minimize the empirical objective, which can be handled by tools in optimization theory \citep{duchi2011adaptive,lacoste2012simpler,nemirovski2009robust,rakhlin2012making,shamir2013stochastic,balamurugan2016stochastic,liu2020firstorder,rafique2018non}. 
%Generalization errors arise from the approximation of the unknown distribution based on sampling, which can be estimated by uniform-convergence approach \cite{lin2016generalization,lei2018stochastic} or stability analysis \citep{hardt2016train,bassily2020stability,lei2020fine,lei2021stability,feldman2019high,yang2021stability}. Specifically, 

\noindent\textbf{Stability and Generalization Analysis of SGMs.} \cite{hardt2016train}   established  generalization error bounds of order $\O(1/\sqrt{n})$ in expectation for SGD for convex and smooth problems  using uniform stability.
The on-average variance of stochastic gradients was used to refine the generalization analysis of SGD for non-convex problems \cite{zhou2018generalization}. The results were improved and refined by \cite{kuzborskij2018data}  using a data-dependent notion of algorithmic stability for SGD. \cite{lei2020fine} introduced on-average argument stability and studied the stability and generalization of SGD for a general class of non-smooth convex losses, i.e., the gradient of the loss function is $\alpha$-H\"older continuous. They also  established fast generalization bounds $\O(1/n)$ for smooth convex losses in a low-noise setting.  The same authors also extended the analysis to the non-convex loss functions in \cite{lei2021sharper}. 
Meanwhile, \cite{bassily2020stability} addressed uniform argument stability of SGD with Lipschitz-continuous convex losses.  Optimal generalization bounds were also developed for SGD in different settings \cite{bassily2020stability,bassily2019private,feldman2020private,wang2022differentially,yang2021stability}.  Stability and generalization for SGMs have been studied for  pairwise learning \cite{shen2019stability,yang2021simple} where the loss involves a pair of examples. In particular,  \cite{yang2021simple} introduced a simple MC-SGD algorithm for pairwise learning where   pairs of examples form a special Markov chain $\{ \xi_t = (z_{i_t}, z_{i_{t-1}}): t\in \N\}$. Here, $z_{i_t}$ and $z_{i_{t-1}}$ are i.i.d.  sampled from the training data of size $n$ at time $t$ and $t-1$, respectively.   The uniform argument stability and generalization have been established (see more discussion on the difference between our work and \cite{yang2021simple} in Remark 8 below).    %\yiming{discussion on our previous paper: puyu Wang and Yunwen}
 
For minimax problems, \cite{zhang2021generalization} studied the weak generalization and strong generalization bounds  in the strongly-convex-concave setting.  \cite{farnia2021train} established the optimal generalization bounds for proximal point method, while gradient descent ascent (GDA) is not guaranteed to have a vanishing excess risk in convex-concave case.  
\cite{lei2021stability} proved that SGDA can achieve the optimal excess risk bounds of order $\O(1/\sqrt{n})$ for both smooth and non-smooth problems in the convex-concave setting. They also extended their work to the nonconvex-nonconcave problems.  However, all the above studies the stability and generalization of SGD and SGDA under the assumption of the i.i.d. sampling scheme. %Under this i.i.d. sampling scheme,  the stochastic gradient is an unbiased estimator of the true gradient which is critical for the involved analysis.

\noindent\textbf{Convergence Analysis of MC-SGMs.}
The convergence analysis of SGD and its variants when the gradients are sampled from a Markov chain have been studied in different settings \cite{atchade2017perturbed,doan2020convergence,doan2020finite,duchi2012ergodic,johansson2010randomized,ram2009incremental,smale2009online,Sun2018Markov,tadic2017asymptotic}. Specifically, \cite{johansson2010randomized,ram2009incremental} studied the Markov subgradient incremental methods in a distributed system under time homogeneous and time non-homogeneous settings, respectively. \cite{duchi2012ergodic} studied the convergence of stochastic mirror descent under the ergodic assumption.  \cite{Sun2018Markov} established the convergence rate $\O(1/T^{1-q})$ with some $q\in(1/2,1)$ for convex problems. They also developed the convergence result for non-convex problems. %Further, decentralized SGD methods when the gradient sampled from a non-reversible Markov chain over a connected network has been studied in \citet{sun2019decentralized}. \citet{doan2020convergence} considered an accelerated ergodic Markov chain SGD for both convex and non-convex problems, and established the same convergence rates as the ones with uniform sampling, up to  logarithmic term. \citet{doan2020finite} proved the same convergence rates as independent sampling case can be achieved when the bounded gradient assumption is removed.
In addition, decentralized SGD methods with the gradients sampled from a non-reversible Markov chain have been studied in \cite{sun2019decentralized}. \cite{doan2020convergence} considered an accelerated ergodic Markov chain SGD for both convex and non-convex problems. \cite{doan2020finite} further studied the convergence rates without the bounded gradient assumption.
All these studies focused on the convergence analysis of MC-SGD for solving the ERM problems. 

\section{Problem Setting and Target of Analysis}\label{sec:preliminary}
 
In this section, we introduce the SGD for ERM and SGDA for solving minimax problems with Markov Chain, and describe the target of generalization analysis for both optimization algorithms.

{\bf Target of Generalization Analysis.} {\bf }Let $\W$ be a parameter space in $\R^d$  
%Let the input space $\X$ be a domain in some Euclidean space, the output space $\Y \in \R$, and $\Z=\X\times \Y$. 
and $\D$ be a population distribution defined on a sample space $\Z$. Let $f:\W \times \Z \rightarrow [0, \infty)$ be a loss function. 
In the standard framework of Statistical Learning Theory (SLT) \cite{bousquet2003introduction,vapnik1999nature}, one aims to minimize the expected population risk, i.e.,
$F(\bw):=\E_z[ f(\bw;z) ]$, where the model parameter $\bw$ belongs to $\W$, and the expectation is taken with respect to (w.r.t.) $z$ according to $\D$.  However, the population distribution is often  unknown. Instead, we have access to a training dataset $S=\{z_i\in \Z\}_{i=1}^n$ with size $n$, where $z_i$ is independently drawn from $\D$.  Then consider the following ERM problem  
\begin{equation}\label{eq:ERM}
    \min_{\bw \in \W} \big\{F_S(\bw) := \frac{1}{n} \sum_{i=1}^n f(\bw; z_i) \big\}. 
\end{equation}
For a randomized algorithm $\A$ to solve the above problem, let $\A(S)$ be the output of algorithm $\A$ based on the dataset $S$. Then its   statistical generalization performance (prediction ability) is measured by its \textit{excess population risk}  $F(\A(S)) - F(\bw^{*})$, i.e., the discrepancy between the expected risks of the model  $\A(S)$  and the best model  $\bw^{*}\in \W$. We are interested in studying the excess population risk.   Let $ \E_{S,\A}[\cdot]$ denote the expectation w.r.t. both the randomness of data $S$ and the internal randomness of $\A$.  
To analyze the excess population error, we use the following error decomposition 
  \begin{align}\label{eq:error-decomposition}
 \E_{S,\A}[ F(\A(S)) ] - F(\bw^{*}) =  \E_{S,\A}[ F(\A(S)) -  F_S(\A(S))  ]  +\E_{S,\A} [F_S(\A(S)) - F_S(\bw^{*}) ].
  \end{align}
The first term is called the generalization error of the algorithm $\A$ measuring the difference between the expected risk and empirical one, for which we will handle using stability analysis as shown soon.  
The second term is the optimization error, which is induced by running the randomized algorithm $\A$ to minimize the empirical objective. It can be estimated by tools from optimization theory.

As discussed in the introduction, many machine learning problems can be formulated as minimax problems %\citep{chen2017robust,dai2018sbeed,du2017stochastic,gao2013one,goodfellow2014generative,liu2018fast,namkoong2017variance,ying2016stochastic,zhao2011online}. 
including adversarial learning~\cite{goodfellow2014generative}, %robust optimization~\citep{chen2017robust,namkoong2017variance}, 
reinforcement learning~\cite{dai2018sbeed,du2017stochastic} and AUC maximization~\cite{gao2013one,liu2018fast,ying2016stochastic,zhao2011online}. 
We are also interested in solving this type of problem. 
Let $\W$ and $\V$ be parameter spaces in $\R^d$. Let $\D$ be a population distribution defined on a sample space $\Z$, and $f:\W \times \V \times \Z \rightarrow [0, \infty)$.   
We consider the minimax optimization problems: 
 %\begin{align}\label{eq:minmax-pop}
     $\min_{\bw\in \W}\max_{\bv\in\V} \big\{F(\bw, \bv) := \EX_{\bz\sim \mathcal{D} } [f(\bw,\bv;z)]  \}$. 
     %\end{align}
In practice, we only have a training dataset $S = \{z_1,\ldots, z_n\}$ independently drawn from $\D$ and hence the minimax problem is reduced to the following empirical version: 
\begin{align}\label{eq:minmax-emp}
    \min_{\bw\in \W}\max_{\bv\in\V} \big\{F_S(\bw, \bv) := {1\over n} \sum_{i=1}^n f(\bw, \bv; z_i)\big\}.
 \end{align}
Since minimax problems involve the primal variable and dual variable,   we have different measures of generalization  \cite{lei2021stability,zhang2021generalization}. For a randomized algorithm $\A(S)$ solving the problem \eqref{eq:minmax-emp},  we denote the output of $\A$ as $\A(S) = (\A_\bw(S), \A_\bv(S))$ for notation simplicity. 
Let $\EX[\cdot]$ denote the expectation w.r.t. the randomness of both $\A$ and $S$. We are particularly  interested in the following two metrics.  

\begin{definition}[Weak Primal-Dual (PD) Risk]
The weak Primal-Dual population risk of $\A(\S)$,  denoted by $\triangle^w(\A_\bw, \A_\bv)$,  is defined as
$\max_{\bv\in\V}\EX\big[F(\A_\bw(\S),\bv)\big]\!-\!\min_{\bw\in\W}\EX\big[F(\bw,\A_\bv(\S))\big].$
The corresponding (expected) weak PD empirical risk, denoted by $\triangle^w_\emp(\A_\bw, \A_\bv)$, is defined by
$
\max_{\bv\in\V}\EX\big[F_\S(\A_\bw(\S),\bv)\big]-\min_{\bw\in\W}\EX\big[F_\S(\bw,\A_\bv(\S))\big].
$ 
We refer to $\triangle^w(\A_\bw, \A_\bv)-\triangle^w_\emp(\A_\bw, \A_\bv)$ as the {\em weak PD generalization error} of the model $(\A_\bw(\S), \A_\bv(\S))$. 
\end{definition}

\begin{definition}[Primal Risk] The primal population and empirical risks of $\A(\S)$ are respectively  defined by
$R(\A_\bw(\S))=\max_{\bv\in\V}F(\A_\bw(\S),\bv),$ and $R_\S(\A_\bw(\S))=\max_{\bv\in\V}F_\S(\A_\bw(\S),\bv).$
We refer to $R(\A_\bw(\S))-R_S(\A_\bw(\S))$ as the primal generalization error of the model $\A_\bw(\S)$, and  $R(\A_\bw(\S))-\min_{\bw\in\W}R(\bw)$
as the \emph{excess primal population risk}.
\end{definition}

{\bf SGD and SGDA with Markov Sampling.} One often considers SGD to solve the ERM problem \eqref{eq:ERM}. Specifically, 
let $\W \subseteq \R^d$ be convex, $\proj_{\W}(\cdot)$ denote the projection to $\W$, and $\partial f(\bw ;z)$ denote a subgradient of $f(\bw ;z)$ at $\bw $.
Let $\bw_0 \in \W$ be an initial point, and $\{\eta_t\}$ is a stepsize sequence. For any $t\in\N$, the update rule of SGD is given by  
%\noindent \textit{MC-SGD update rule:} Initialize $\bw_0 \in \W$, for any $t\in \mathbb{N}$
\begin{equation}\label{eq:MCSGD-update-rule}
    \bw_{t} = \proj_{\W}\big( \bw_{t-1}- \eta_{t} \partial f(\bw_{t-1};z_{i_t})  \big),
\end{equation}
where $\{i_t\}$ is generated from $[n] = \{1,2,\ldots, n\}$ with some sampling scheme. A typically sampling scheme is the uniform i.i.d. sampling, i.e., $i_t$ is drawn randomly from $[n]$ according to a uniform distribution with/without replacement. 

In this paper, we are particularly interested in the case when $i_t\in [n]$ is drawn from a {\em Markov Chain} which is widely used in practice \cite{atchade2017perturbed,ayache2021private,doan2020convergence,duchi2012ergodic,johansson2010randomized,smale2009online,Sun2018Markov,tadic2017asymptotic}.  Let $P$ be an $n \times n$-matrix with real-valued entries.  We say a Markov chain $\{X_k\}$ with finite state $[n]$ and transition matrix $P$ is time-homogeneous if, for $k\in \N $, $i,j \in [n]$, and $ i_1,\ldots,i_{k-1} \in[n]$, there holds $ \Pr(X_{k+1}=j |   X_1=i_1, \ldots, X_k=i)   = \Pr( X_{k+1}=j | X_k=i ) = [P]_{i,j} .  $ Likewise,  the SGDA algorithm with Markov sampling scheme is defined as follows.  
%A notable algorithm  for solving the minimax problem \eqref{eq:minmax-emp} is the stochastic gradient descent ascent (SGDA) algorithm. 
Specifically, let $\partial_{\bw} f$ and  $\partial_{\bv} f$ denote the subgradients of $f$ w.r.t. the arguments $\bw$ and $\bv$, respectively.  
We initialize $(\bw_0, \bv_0)\in \W \times \V$, for any $t\in \mathbb{N}$, let  $\{i_t\}$ is drawn from $[n]$ according to a Markov Chain. The update rule of SGDA is given by
\vspace*{-1mm}
\begin{equation}\label{eq:MCSGDA}
\begin{cases}
\bw_{t}  = \proj_{\W}\bigl( \bw_{t-1}- \eta_{t} \partial_\bw f(\bw_{t-1}, \bv_{t-1};z_{i_t})  \bigr) &\\
    \bv_{t} = \proj_{\V}\bigl( \bv_{t-1}+ \eta_{t} \partial_\bv f(\bw_{t-1}, \bv_{t-1};z_{i_t})  \bigr).&
\end{cases}
\end{equation}
%In the literature \citep[e.g.,][]{bassily2020stability, bousquet2002stability,charles2018stability,farnia2021train,hardt2016train,kuzborskij2018data,lei2020fine,lei2021sharper,lei2021stability,zhang2021generalization}, stability and generalization analysis of SGMs mainly focus on the i.i.d. sampling scheme where the randomly sampled datum at each iteration is uniformly drawn from the dataset.  
%Recently, SGMs with Markov sampling has received extensive attention due to its wide applicability    . Specifically, in the update rules \eqref{eq:MCSGD-update-rule} and \eqref{eq:MCSGDA}, the sampled datum $z_{i_t}$ at iteration $t$ is generated from a Markov process.% instead of the i.i.d. sampling. 
For  brevity, we refer to the above algorithms as   Markov chain-SGD (MC-SGD) and Markov chain-SGDA (SGDA), respectively.  There are two types of randomness in MC-SGD/MC-SGDA. The first randomness is due to training dataset $S$ which  is i.i.d. from the population distribution $\D$. The other randomness arises from the internal randomness of the MC-SGD/MC-SGDA algorithm, i.e., the randomness of the indices $\{i_t\}$, which is a Markov chain.

\begin{remark}Convergence analysis mainly considers the empirical optimization gap, i.e., the discrepancy between $F_S(\A(S))$ and $F_S(\bw^{*})$. Here, we are mainly interested in the generalization error which measures the prediction ability of the trained model on the test (future) data.  {\em As such, the purpose of this paper is to provide a comprehensive generalization analysis of MC-SGD and MC-SGDA in the framework of statistical learning theory.} Specifically, given a finite training data $S$, let $\A(S)$ be the output of the MC-SGD for solving the ERM problem \eqref{eq:ERM}. Our target is to analyze the excess population risk $\E_{S,\A}[ F(\A(S)) ] - F(\bw^{*}).$ Let $\A(S) = (\A_\bw(S), \A_\bv(S))$ be the output of MC-SGDA for solving the empirical minimax problem \eqref{eq:minmax-emp}, our aim is to analyze the {\em weak PD population risk} $\triangle^w(\A_\bw, \A_\bv)$ and the \emph{excess primal population risk} $ R(\A_\bw(\S))-\min_{\bw\in\W}R(\bw).$ In both cases, the generalization analysis will be conducted using the algorithmic stability \cite{bousquet2002stability,hardt2016train}. As we show soon below, the final rates are obtained through trade-offing the optimization error (convergence rate) and the generalization error (stability results). 
\end{remark}

\vspace*{-1mm}
{\bf Properties of Markov Chain.} Denote the probability distribution of  $X_k$ as the non-negative row vector $\pi^k=(\pi^k(1), \pi^k(2),\ldots,\pi^k(n))$, i.e., $\Pr(X_{k}=j)=\pi^k(j)$. Further, we have $\sum_{i=1}^n \pi^k(i)=1$. For the time-homogeneous Markov chain, it holds $\pi^k=\pi^{k-1}P=\dots=\pi^1 P^{k-1}$ for all $k\in \N$. Here, $\pi^1$ is an initial distribution and $P^k$ denotes the $k$-th power of $P$. A Markov chain is irreducible if, for any $i,j \in[n]$, there exists $k$ such that $[P^k]_{i,j}>0$.  
That is, the Markov  process can go from any state to any other state. State $i \in[n]$ is said to have a period $\tau$ if $[P^k]_{i,i}=0$ whenever $k$ is not a multiple of $\tau$ and $\tau$ is the greatest integer with this property. If $\tau=1$ for every state $i\in[n]$, then we say the Markov chain is aperiodic.
We say a Markov chain with stationary distribution $\Pi^{*}$ is reversible if $\Pi^{*}(i)[P]_{i,j}=\Pi^{*}(j)[P]_{j,i}$ for all $i,j\in [n]$. 

We need the following assumption  for studying optimization error of MC-SGMs.    
\begin{assumption}\label{ass:Markov-chain}
Assume the Markov chain $\{i_t\}$ with finite state $[n]$ is  time-homogeneous, irreducible and aperiodic. It starts from an initial distribution $\pi^1$, and has transition matrix $P$ and stationary distribution $\Pi^{*}$ with $\Pi^{*}(i)=\frac{1}{n}$ for any $i\in[n]$, i.e.,
$\lim_{k\to \infty} P^k =\frac{1}{n}\mathbf{1}_n\mathbf{1}_n^\top,$ %\Big[\big(\frac{1}{n}\big);\ldots;\big(\frac{1}{n}\big)\Big] \in \R^{n\times n}.
where $\mathbf{1}_n\in\mathbf{R}^n$ is the vector with each entry being $1$ and $\mathbf{1}_n^\top$ denotes its transpose.
\end{assumption}
\begin{remark}
     Our assumptions on Markov chains listed above  are standard in the literature \cite{ doan2020convergence,johansson2010randomized,mao2020walkman,Sun2018Markov,sun2019decentralized,yang2021simple}. For instance, Markov chain-type SGD was proposed for pairwise learning which can apply to various learning task such as  AUC maximization and bipartite ranking  \cite{agarwal2009generalization,ying2016online,zhao2011online,gao2013one,liu2018fast} and metric learning \cite{kulis2013metric,weinberger2009distance,xing2002distance,ying2012distance}. This pairwise learning algorithm  forms  pairs of examples following a special Markov chain $\{ \xi_t = (z_{i_t}, z_{i_{t-1}}): t\in \N\}$ where $z_{i_t}$ and $z_{i_{t-1}}$ are i.i.d. sampled from the training data of size $n$ at time $t$ and $t-1$, respectively and, at time $t$, the model parameter is updated using gradient descent based on $\xi_t.$  As mentioned in Remark 3 of \cite{yang2021simple}, $\{\xi_t: t\in \N\}$ is a Markov Chain satisfying all of our assumptions. Another notable example is the decentralized consensus optimization in a multi-agent network, where 
     %all agents collaboratively find a minimizer for the sum of their private functions. To solve this problem, one often considers a token that carries a variable $\bw$ randomly moving through a (random) succession of nodes in the network. When the token reaches a node $i$, node $i$ reads $\bw$ from the token and updates $\bw$ based on its own data. Then, the token walks away to a random neighbor of node $i$. 
     the samples are accessed according to a Markov chain and the number of states of the Markov chain equals the number of nodes in the network, which is finite. One always considers the same transition matrix $P$ for each node and assumes the Markov chain is irreducible and aperiodic \cite{mao2020walkman,zeng2021finite}.  
\end{remark}
 
\section{Results for Markov Chain SGD}\label{sec:MC-SGD}

In this section, we present the stability and generalization results of MC-SGD. Our analysis requires the following definition and assumptions. Let $G,L>0$ and $\|\cdot\|_2$ denote the Euclidean norm. 
\begin{definition}%\label{ass:convex}
We say $f$ is convex w.r.t. the first argument if, for any $z\in \Z$ and $\bw,\bw' \in \W$, there holds $f(\bw;z)\ge f(\bw';z)+\langle \partial f(\bw';z),\bw-\bw'\rangle$.
\end{definition}
 
 \begin{assumption}\label{ass:G-lipschitz}
Assume $f$ is $G$-Lipschitz continuous, i.e., for any $z\in \Z$ and $\bw, \bw' \in \W$, there holds $|f(\bw;z)-f(\bw';z)|\le G\| \bw-\bw' \|_2$.  
 \end{assumption}

\begin{assumption}\label{ass:beta-smooth}
Assume $f$ is $L$-smooth, i.e., for any $z\in \Z$ and $\bw, \bw' \in \W$, there holds $ f(\bw;z)-f(\bw';z) \le \langle \partial f(\bw';z), \bw-\bw'\rangle +\frac{L}{2} \|\bw -\bw'\|_2^2 $. 
\end{assumption}
 
\subsection{Stability and Generalization of MC-SGD}
\vspace*{-1mm}
Let $\bw^{*}=\arg\min_{\bw\in \W} F(\bw)$ be the best model in $\W$ and $\bar{\bw}_T= \sum_{j=1}^T \eta_j \bw_j / \sum_{j=1}^T \eta_j$ be the output of MC-SGD with $T$ iterations. We will use algorithmic stability to study the generalization errors, which measures the sensitivity of the output model of an algorithm.     
%we first study on-average argument  stability of MC-SGD, which measures the on-average sensitivity of  models by traversing the perturbation of each single coordinate. Then we use the on-average argument stability to control generalization errors for both smooth and non-smooth problems. 
Below we give the definition of on-average argument stability~\cite{lei2020fine}. 
\begin{definition}(On-average argument stability)\label{def:on-average-stability}
Let $S=\{ z_1,\ldots,z_n \}  $ and $\widetilde{S}=\{ \tilde{z}_1, \ldots, \tilde{z}_n \}$ be drawn independently from $\D$. For any $i\in[n]$, define $S^{(i)}=\{z_1,\ldots,z_{i-1},\tilde{z}_i, z_{i+1},\ldots,z_n  \}$ as the set formed from $S$ by replacing the $i$-th element with $\tilde{z}_i$. %Let $\EX[\cdot]$ denote the expectation with respect to the randomness of  algorithm $\A$ and data $S$ and $S^{(i)}$. 
We say a randomized algorithm $\A$ is on-average  $\epsilon$-argument-stable if 
$ \E_{S,\widetilde{S},\A} \big[ \frac{1}{n}\sum_{i=1}^n \|\A(S)-\A(S^{(i)})\|_2 \big] \le \epsilon.$
%\vspace*{-3mm}
\end{definition}

To obtain on-average argument stability bounds of MC-SGD, our idea is to first write the stability as a deterministic function according to whether the different data point is selected, and then take the expectation w.r.t. the randomness of the algorithm.  The detailed proofs are given in Appendix~\ref{proof:stability}. 
\begin{theorem}[Stability bounds]\label{thm:stability}
	Suppose  $f$ is convex and Assumption~\ref{ass:G-lipschitz} holds.  Let {  $\W=\R^d$} and let $\A$ be MC-SGD  with $T$ iterations.  
	\begin{enumerate}[label=(\alph*), leftmargin=*]\setlength\itemsep{-2mm}
		\item (Smooth case) Suppose Assumption \ref{ass:beta-smooth} holds and $\eta_j\le 2/L$. Then  $\A$ is on-average  $\epsilon$-argument-stable with  $\epsilon \le \frac{2G}{n}\sum_{j=1}^T \eta_j$.
		% \vspace*{-1mm}
		\item (Non-smooth case) $\A$ is   on-average $\epsilon$-argument-stable with  $ \epsilon \le 2G \sqrt{ \sum_{j=1}^T \eta_j^2   }  + \frac{4G }{n} \sum_{j=1}^T \eta_j $. 
	\end{enumerate}
\end{theorem}

\begin{remark}\label{rmk:on-ave}
{Without any assumption on Markov chain,}
 Theorem~\ref{thm:stability} shows    that argument stability bounds of MC-SGD are in the order of $\O(T\eta/n)$ and $\O(\sqrt{T}\eta + T\eta/n)$ with a constant stepsize $\eta$  for smooth and non-smooth losses, respectively. Both of them match the corresponding bounds for SGD with  i.i.d.  sampling~\cite{bassily2020stability,hardt2016train,lei2020fine,wang2022differentially}, which imply that stability of MC-SGD is   at least not worse than that of the i.i.d. sampling case. 
The technical novelty here is to observe that, in the sense of on-average argument stability, we can use the calculation of $\E_\A[\sum_{i=1}^n \mathbb{I}_{[i_t=i]} ]$ to replace that of $\E_\A[ \mathbb{I}_{[i_t=i]} ]$, where $\mathbb{I}_{[\cdot]}$ is the indicator function. This key step avoids the complicated calculations about $\E_\A[ \mathbb{I}_{[i_t=i]} ]$. %due to the dependence on transition matrix $P$, the number of iterations $t$ and the initial distribution $\pi^1$.
{Taking the uniform stability as   example, we need to consider neighboring datasets differing by the $i$-th data, and can get $\mathbb{E}_{\mathcal{A}}[\| \bw_t - \bw_t' \|_2]=\O(\eta \sum_{j=1}^t \mathbb{E}_{\mathcal{A}}[\mathbb{I}_{i_j=i} ])=\O(\eta \sum_{j=1}^t \sum_{k=1}^n [P^{j-1}]_{k,i}\pi^1(k))$, which depends on the  transition matrix $P$ and is not easy to control. In contrast, with the on-average stability we get stability bounds depending on $\sum_{i=1}^n\mathbb{I}_{[i_t=i]}$, which is always $1$, i.e.,  the on-average stability allows us to ignore the effect of sampling process.}
 %Since $\sum_{i=1}^n \mathbb{I}_{[i_t=i]}=1$ for any $t\in[T]$, the dependence effect caused by the Markov Chain sampling is removed using the on-average stability. 
\end{remark}

The following theorem presents generalization bounds for MC-SGD in both smooth and non-smooth cases, which directly follows from Lemma~\ref{thm:gen-model-stab} and Theorem \ref{thm:stability}.
\begin{theorem}[Generalization error bounds]\label{thm:generalization-error}
Suppose  $f$ is convex and Assumption    \ref{ass:G-lipschitz} holds. Let { $\W=\R^d$} and let $\A$ be MC-SGD with $T$ iterations. 
 \vspace*{-2mm}
\begin{enumerate}[label=(\alph*), leftmargin=*]\setlength\itemsep{-1mm}
    \item (Smooth case)  Suppose Assumption~\ref{ass:beta-smooth} holds and let  $\eta_j\equiv\eta\le 2/L$. Then there holds
    \[ \E_{S,\A}[ F(\bar{\bw}_T) - F_S(\bar{\bw}_T) ]\le \frac{2G^2T\eta}{n}.\]
     %\vspace*{-1mm}
    \item (Non-smooth case) If $\eta_j\equiv\eta$, then there holds
    \[\E_{S,\A}[ F(\bar{\bw}_T) - F_S(\bar{\bw}_T) ]= \O\big( \sqrt{T}\eta + \frac{T\eta}{n} \big) .\]
\end{enumerate}
\end{theorem}

\subsection{Excess Population Risk of MC-SGD}

In this subsection, we present excess population risk bounds for MC-SGD in both smooth and non-smooth cases. 
%We now combine the generalization  and optimization results to establish excess population risk bounds for MC-SGD.  
%The following two theorems study the smooth and non-smooth problems for the convex case, respectively.  
The proofs are given in Appendix~\ref{proof:excess}. %Let $\bar{\bw}_T= \sum_{j=1}^T \eta_j \bw_j / \sum_{j=1}^T \eta_j$.
We use the notation $B\asymp \tilde{B}$ if there exist universal constants $c_1, c_2 > 0$ such that $c_1\tilde{B} \le B \le c_2 \tilde{B}$. Let $\lambda_i(P)$ be the $i$-th largest eigenvalue of the transition matrix $P$ and $\lambda(P)=(\max\{|\lambda_2(P)|,|\lambda_n(P)|\}+1)/2 \in [1/2,1)$. Let $K_P$ be the mixing time and $C_P$ be a constant depending on $P$ and its Jordan canonical form (detailed expressions are given in Lemma~\ref{lem:difference-stationary}).   We assume $\sup_{z\in\Z}f(0;z)$ and $\|\bw^{*}\|_2$ are bounded. 
%We introduce the following assumption to handle the extreme setting (i.e., the mixing time of Markov chain is too large) when estimating optimization errors. 
\begin{assumption}\label{ass:reversible}
Assume the Markov chain $\{i_t\}$ is reversible with $P=P^\top$. 
\end{assumption}

\begin{theorem}[Excess  population risk for smooth losses]\label{thm:excess-smooth}
Suppose  $f$ is convex and Assumptions \ref{ass:Markov-chain},  \ref{ass:G-lipschitz},  \ref{ass:beta-smooth} and \ref{ass:reversible} hold.  Let $\W\in\R^d$. Let $\A$ be MC-SGD with $T$ iterations, 
and $\{ \bw_j\}_{j=1}^T$ be produced by $\A$ with $\bw_0=0$ and $\eta_j\equiv\eta \le 2/L$. %Let $D_0=\|\bw^{*}\|_2$, $D=\big((G^2+2\sup_{z\in\Z}f(0;z))\sum_{k=1}^T\eta_k\big)^{1/2}+D_0$  and \vspace*{-1mm}\begin{equation}\label{eq:kj-sgd-convex}\!\!\!\!\!k_j\!=\!\min\!\Big\{\!\max\Big\{\Big\lceil \frac{\log(2C_PD n j)}{\log(1/\lambda(P))}\Big\rceil, K_P \Big\}, j\Big\}, j\in[T].\end{equation}  %Then\[ \E_{S,\A} [ F(\bar{\bw}_T) ] - F(\bw^{*}) = \O\Big( \frac{\eta T}{n} +  \frac{1 + \big(T + \sum_{j=1}^T k_j\big) \eta^2 }{T\eta} + \frac{K_P}{T} \Big).\]
%\vspace*{-1mm}Further, suppose Assumption~\ref{ass:reversible} holds. 
If we select $T\asymp n$ and $\eta=(T\log(T))^{-1/2}$, then  
\[ \E_{S,\A} [ F(\bar{\bw}_T) ] - F(\bw^{*}) = \O\big( \sqrt{\log(n)}/(\sqrt{n}\log(1/\lambda(P)))\big).  \]
\end{theorem}
\begin{remark}
  A term ${K_P}/\big(M_0{n^{\frac{3}{4}}\log^{\frac{1}{4}}(n)  }\big)$ with $M_0= \min\{   \sqrt{n\log(n)} C_P  n\lambda(P)^{K_P} , 1 \}$ appears in the excess risk bound of MC-SGD (see the proof of Theorem~\ref{thm:excess-smooth}), which will be worse than $\sqrt{\log(n)}/\sqrt{n}$ when $K_P$ is large.   Note Lemma A.1 implies that this term will disappear if $P$ is symmetric. Hence,  we introduce Assumption 4 to get the nearly optimal  rate. 
\end{remark} 

\begin{theorem}[Excess population risk for non-smooth losses]\label{thm:excess-nonsmooth}
Suppose  $f$ is convex and Assumptions \ref{ass:Markov-chain},  \ref{ass:G-lipschitz}, \ref{ass:reversible} hold. Let $\W\in\R^d$.   Let $\A$ be MC-SGD with $T$ iterations, 
and $\{ \bw_j\}_{j=1}^T$ be produced by $\A$ with $\eta_j\equiv\eta $. %Let $k_j$ be defined in \eqref{eq:kj-sgd-convex}. %Then \[ \E_{S,\!\A}\![ F(\bar{\bw}_T) ]\!-\! F(\bw^{*})\!=\!\O\Big(\!\sqrt{T}\eta + \frac{T\eta}{n}\!+\!\frac{1\!+\!\sum_{j=1}^Tk_j  \eta^2 }{T\eta}\!+\!\frac{K_P}{T}\!     \Big).\]
%Further, suppose Assumption~\ref{ass:reversible} holds.  
If we select $T\asymp n^2$ and $\eta= {T^{-3/4}  }$, then  
\[  \E_{S,\A} [ F(\bar{\bw}_T) ] - F(\bw^{*}) = \O\big( 1/\big({\sqrt{n}}\log(1/\lambda(P))\big) \big).\]
\end{theorem}

\begin{remark}
To estimate the excess population risk, 
we need the convergence rates of MC-SGD which can be found in Appendix~\ref{appendix-opt-sgd}. \cite{Sun2018Markov}  established a convergence rate of $\O(1/T^{1-q})$ with some $q\in(1/2,1)$ under the  bounded parameter domain assumption. We remove this  assumption by showing  $\|\bw_t\|_2^2=\O(\sum_{k=1}^T\eta_k)$ and obtain  the nearly optimal  convergence rate $\tilde{\O}(1/\sqrt{T})$ with a careful choice of $\eta=1/ \sqrt{T \log(T) } $. 
%The convergence proof follows from \cite{Sun2018Markov}, who established a convergence rate of $\O(1/T^{1-q})$ with some $1/2< q<1$. Indeed,  with a careful choice of $\eta=1/ \sqrt{T \log(T) } $,  a faster rate $\tilde{\O}(1/\sqrt{T})$ can be obtained under the same assumption. Here the notation $\tilde{\O}(\cdot)$ means $\O(\cdot)$ up to some logarithmic terms. 
%The above convergence rate matches that of SGD with uniform sampling~\cite{bottou2018optimization}, up to a logarithmic term.
%We further provide the convergence analysis for non-convex problems in Appendix~\ref{subsec:nonconvex}.  
%Since the convergence in terms of objective values cannot be given, we only measure the convergence rate in terms of gradient norm. With a suitable choice of stepsize $\eta_j$, we can show in Appendix~\ref{subsec:nonconvex} that  $\min_{1\le j\le T} \E_\A\big[\| \partial F_S( \bw_j)\|_2^2 ]=\O\big(  \log(T)/({\sqrt{T} \log^2(1/\lambda(P)) }) \big).$ 
To understand the variation of the algorithm,  we present a confidence-based bound for optimization error, which matches the bound in expectation up to a constant factor. We also provide the convergence analysis for non-convex problems in Appendix~\ref{appendix-opt-sgd}.%With a suitable choice of stepsize $\eta_j$, we can show that  $\min_{1\le j\le T} \E_\A\big[\| \partial F_S( \bw_j)\|_2^2 ]=\O\big(  \log(T)/({\sqrt{T} \log^2(1/\lambda(P)) }) \big).$ 
\end{remark}
 
\begin{remark}
Theorems~\ref{thm:excess-smooth} and \ref{thm:excess-nonsmooth} show, after carefully selecting the iteration number $T$ and stepsize $\eta$, that  the excess population risk rate $\O(1/\sqrt{n})$ is achieved in both  smooth and non-smooth cases. 
{
Note \cite{bassily2020stability,hardt2016train,lei2020fine} show that the excess population risk  rate $\O(1/\sqrt{n})$ is optimal for the i.i.d. sampling case. Therefore, our results for MC-SGD are also optimal since the i.i.d. sampling is a special case of Markov sampling. 
Our results imply that despite the gradients are biased and dependent across iterations in Markov sampling, the  generalization performance of SGD is competitive with the i.i.d.  sampling case. } 
Theorems~\ref{thm:excess-smooth} and \ref{thm:excess-nonsmooth} also show the impact of the smoothness in achieving the optimal rate. 
The rate for the non-smooth case in Theorem~\ref{thm:excess-nonsmooth} looks slightly better than the smooth case (Theorem~\ref{thm:excess-smooth}) with a logarithmic term. However, the optimal rate can be achieved with a linear gradient complexity (i.e., the total number of computing the gradient) for smooth losses, while Theorem~\ref{thm:excess-nonsmooth} implies that gradient complexity $\O(n^2)$ is required for non-smooth losses.
 %The additional $\sqrt{\log T}$ in Thm 6 comes from the artifact of the analysis with the choice of a small stepsize ~$\eta\!=\!1/\!\sqrt{T\!\log T}. $
\end{remark}

\begin{remark}\label{rmk:discuss-P}
According to Theorems~\ref{thm:excess-smooth} and \ref{thm:excess-nonsmooth}, we can further observe how the transition matrix $P$ affects the excess population risks. 
Indeed, the excess population risk rates are monotonically increasing w.r.t $\lambda(P)$. 
%Recall $\lambda(P)=(\max\{|\lambda_2(P)|,|\lambda_n(P)|\}+1)/2$. 
Particularly, the closer $\lambda(P)$ is to $1/2$, the better the rate is. Let us consider two extreme examples. Suppose the Markov chain starts from the uniform distribution and has transition matrix $P=\frac{1}{n}\mathbf{1}_n\mathbf{1}_n^T$. MC-SGD degenerates to SGD with i.i.d. sampling in this case. The excess population risk rate $\O(1/\sqrt{n})$ is obtained from Theorem~\ref{thm:excess-nonsmooth} with $\lambda(P)=1/2$. For a Markov chain moving on a circle (i.e., if the chain is currently at state $i$, then it goes to states $i+1$, $i$ and $i-1$ with equal probability), we can verify that $\lambda(P)=\O(1-1/n^2)$,  which implies a bad rate in this case. % our rate  will be extremely bad in this case. 
\end{remark}

\begin{remark}
\cite{yang2021simple} proposed a simple MC-SGD algorithm for pairwise learning associated with a pairwise loss $f(\bw, z, z')$. Specifically, at iteration $t$, the algorithm update the model parameter as follows:  $\bw_{t} = \bw_{t-1} - \eta_t \nabla_\bw f(\bw_{t-1}, z_{i_t}, z_{i_{t-1}})$ where $z_{i_t}$ and $z_{i_{t-1}}$ are i.i.d.  sampled from the training data of size $n$ at time $t$ and $t-1$, respectively.
In Remark 3 of \cite{yang2021simple}, it was shown that $\{\xi_t = (i_t, i_{t-1}) \in [n] \times [n]\}$ does form a time-homogeneous, irreducible and aperiodic Markov chain. There are two key differences between our work and \cite{yang2021simple}. Firstly, the work \cite{yang2021simple} used uniform stability directly due to $\E_\A[\mathbb{I}_{i_t=i}]=1/n$, while this term is not easy to control in our general setting (see Remark~\ref{rmk:on-ave} for details).  To overcome this hurdle, we resort to the on-average stability and show that MC-SGD achieves the optimal excess risk rate. Secondly, the proofs there critically rely on the fact  $f(\bw_{t-1};\bz_{i_t},\bz_{i_{t-1}})=f(\bw_{t-2};\bz_{i_t},\bz_{i_{t-1}})+\O(\eta_{t-1})$ and the independence of $\bw_{t-2}$ w.r.t. ${i_t}$ and $i_{t-1}$.  However, these specially tailored techniques  for pairwise learning do not apply to the general Markov setting as we considered here.   
%for Lipschitz-continuous convex  losses by using uniform stability directly due to $\E_\A[\mathbb{I}_{i_t=i}]=1/n$, while this term is not easy to control in our general setting (see Remark~\ref{rmk:on-ave} for details).
%\yiming{Puyu and Yunwen:  could you add the detailed comparison with our neurips paper on simple SGD for pairwise learning? what is the exact algorithm in pairwise learning case? rates   and assumptions comparison on the loss;  Here some sentences "For multi-pass SGD (Algorithm 1 in finite-sum or offline setting) in the work \cite{yang2021simple}, the sampling scheme $\{\xi_t = (i_t, i_{t-1}) \in [n] \times [n]\}$ does form a time homogeneous Markov chain with transition matrix $P$ with $n^2$ states. We can verify that such a Markov chain is irreducible and aperiodic"}
\end{remark}

\section{Results for Markov Chain SGDA}\label{sec:MC-SGDA}
%\vspace*{-1mm}

In this section, we study the generalization analysis of MC-SGDA for minimax optimization problems. Let $(\bar{\bw}_T, \bar{\bv}_T)$ be the output of MC-SGDA with $T$ iterations, where 
%\vspace*{-1mm}
\begin{equation}\label{eq:avgoutput-sgda}
    \bar{\bw}_T= {\sum_{j=1}^T \eta_j\bw_j}\big /{\sum_{j=1}^T \eta_j}  \text{ and }   \bar{\bv}_T= {\sum_{j=1}^T \eta_j\bv_j}\big /{\sum_{j=1}^T \eta_j}.
\end{equation} 
%\vspace*{-1mm}
We first introduce some necessary definitions and assumptions.

\begin{definition}
Let $\rho\geq 0$ and $g:\W\times \V \mapsto \R$.  We say $g$ is $\rho$-strongly-convex-strongly-concave ($\rho$-SC-SC) if, 
    for any $\bv\in\V$, the function $\bw\mapsto g(\bw,\bv)$ is $\rho$-strongly-convex and, for any $\bw\in\W$, the function $\bv\mapsto g(\bw,\bv)$ is $\rho$-strongly-concave. We say $g$ is convex-concave if $g$ is $0$-SC-SC. 
 
\end{definition}
%The case of $\rho$-SC-SC with $\rho=0$ corresponds to the standard convex-concave case.
\vspace*{-2mm}

The following two assumptions are standard~\cite{farnia2021train,zhang2021generalization}. Assumption \ref{ass:lipschitz} amounts to saying $f$ is Lipschitz continuous w.r.t. both $\bw$ and $\bv$, while Assumption \ref{ass:smooth} considers smoothness conditions.  %Let $G,L>0$.
\begin{assumption}\label{ass:lipschitz}
  Assume for all $\bw\in\W, \bv\in\V$ and $z\in\Z$, 
 $ \big\|\partial_{\bw}f(\bw,\bv;z)\big\|_2\leq G\;\text{ and }\; \big\|\partial_{\bv}f(\bw,\bv;z)\big\|_2\leq G.$
\end{assumption}
\begin{assumption}\label{ass:smooth}
  For any $z$, assume the function $(\bw,\bv)\mapsto f(\bw,\bv;z)$ is  $L$-smooth, i.e., the following inequality holds for all $\bw\in\W,\bv\in\V$ and $z\in\Z$ 
  \[
  \left\|\begin{pmatrix}
           \partial_{\bw}f(\bw,\bv;z)-\partial_{\bw}f(\bw',\bv';z) \\
           \partial_{\bv}f(\bw,\bv;z)-\partial_{\bv}f(\bw',\bv';z)
         \end{pmatrix}\right\|_2\!\leq\! L\left\|\begin{pmatrix}
                                             \bw\!-\!\bw' \\
                                             \bv\!-\!\bv'
                                           \end{pmatrix}\right\|_2.
  \]
  \vspace*{-2mm}
\end{assumption}

%\vspace*{-1mm}
\subsection{Stability and Generalization Measures}
\vspace*{-1mm}
We use algorithmic stability to study the generalization of minimax learners. To this end, we first introduce the stability for minimax optimization problems.

\begin{definition}[Argument stability for minimax problems]
Let $S,\widetilde{S}$ and ${S}^{(i)}$ be defined as Definition \ref{def:on-average-stability}.
Let $\A$ be a randomized algorithm and $\epsilon>0$. We say $\A$ is on-average $\epsilon$-argument-stable for minimax problems if 
$
\frac{1}{n}\sum_{i=1}^{n}\ebb\big[\|\A_{\bw}(S^{(i)})\!-\!\A_{\bw}(S)\|_2+\|\A_{\bv}(S^{(i)})\!-\!\A_{\bv}(S)\|_2\big]\leq \epsilon.
$
\end{definition}

The following theorem establishes a connection between stability and generalization. Part (a) shows that on-average argument stability implies generalization measured by the weak PD risk, while Part (b) shows that on-average argument stability guarantees a strong notion of generalization in terms of the primal risk under a strong concavity assumption. Theorem \ref{thm:stab-gen-min-max} will be proved in Appendix \ref{proof:sgda-stab}. % of the Appendix. 
\begin{theorem}[Generalization via argument  stability]\label{thm:stab-gen-min-max}
Let $\A$ be a randomized algorithm and $\epsilon>0$.
 \vspace*{-1mm}
\begin{enumerate}[label=(\alph*), leftmargin=*]\setlength\itemsep{-1mm}
  \item If $\A$ is on-average $\epsilon$-argument-stable and Assumption \ref{ass:lipschitz} holds, then there holds
  \[  \triangle^w(\A_{\bw},\A_{\bv})-\triangle^w_\emp(\A_{\bw},\A_{\bv})\leq G\epsilon.\]
  \item
  If $\A$ is on-average $\epsilon$-argument-stable, the function $\bv\mapsto F(\bw,\bv)$ is $\rho$-strongly-concave and Assumptions \ref{ass:lipschitz}, \ref{ass:smooth} hold, then we have %the primal generalization error satisfies
 \[ \ebb_{S,\A}\big[R(\A_{\bw}(S))-R_S(\A_{\bw}(S))\big]\leq \big(1+L/\rho\big)G\epsilon.\]
\end{enumerate}  
\end{theorem}

In the following theorem we develop stability bounds for MC-SGDA applied to convex-concave problems. 
%Part (a) considers the smooth problems, while Part (b) considers nonsmooth problems. 
The proof is given in Section \ref{proof:sgda-stab} of the Appendix. 
\begin{theorem}[Stability bounds\label{thm:stab-sgda}]
Assume for all $z$, the function $(\bw,\bv)\mapsto f(\bw,\bv;z)$ is convex-concave. Let {  $\W=\R^d$} and Assumption \ref{ass:lipschitz} hold, and let $\A$ be MC-SGDA with $T$ iterations. 
 \vspace*{-1mm}
\begin{enumerate}[label=(\alph*), leftmargin=*]\setlength\itemsep{-2mm}
  \item (Smooth case)  If Assumption \ref{ass:smooth} holds and $\sum_{j=1}^{T}\eta_j^2\leq1/(2L^2)$, then $\A$ is on-average $\epsilon$-argument stable with  
$
\epsilon\leq 4G\big(\frac{1}{n}\sum_{j=1}^{T}\eta_j^2\big)^{1/2}+\frac{8\sqrt{2}G}{n}\sum_{j=1}^{T}\eta_j.
$
 %\vspace*{-1mm}
 \item (Non-smooth case)  $\A$ is on-average $\epsilon$-argument stable with 
$
\epsilon \leq 2 G\sqrt{2\sum_{j=1}^{T} \eta_j^2} + \frac{4\sqrt{2}G}{n} \sum_{j=1}^{T} \eta_j.
$
\end{enumerate}
\end{theorem}
%\vspace*{-1mm}
\begin{remark}
For convex-concave and Lipschitz problems, the stability bound of the order $\O(\eta(\sqrt{T}+T/n))$ was established for SGDA with a constant stepsize under the uniformly i.i.d.  sampling setting. Under a further smoothness assumption, the stability bound was improved to the order of $O(\eta T/n)$~\cite{lei2021stability}. Our stability bounds in Theorem \ref{thm:stab-sgda} match these results up to a constant factor and extend them to the Markov sampling case. 
\end{remark}
%\vspace*{-1mm}
\begin{remark}
Let $\{(\bw_t^{(i)},\bv_t^{(i)})\}$ be the SGDA sequence based on $S^{(i)}$.
The existing stability analysis~\cite{lei2021stability} builds a recursive relationship for $\ebb_{\A}\big[\|\bw_t-\bw_t^{(i)}\|_2^2+\|\bv_t-\bv_t^{(i)}\|_2^2\big]$, which crucially depends on the i.i.d. sampling property of $i_t
\in [n]$. This strategy does not apply to MC-SGDA since the conditional expectation over $i_t$ is in a much complex manner due to the Markov Chain sampling. We bypass this difficulty by building a recursive relationship for $\|\bw_t-\bw_t^{(i)}\|_2^2+\|\bv_t-\bv_t^{(i)}\|_2^2$ in terms of a sequence of random variables $\ibb_{[i_t=i]}$. A key observation is that the effect of randomness would disappear if we consider on-average argument stability since $\sum_{i=1}^n\ibb_{[i_t=i]}=1$ for any $t\in\mathbb{N}$.
\end{remark}
We can combine the stability bounds in Theorem \ref{thm:stab-sgda} and Theorem \ref{thm:stab-gen-min-max} to develop generalization bounds. We first establish weak PD risk bounds in Theorem \ref{thm:gen-sgda}, and then move on to primal population risk bounds in Theorem \ref{thm:gen-sgda-primal}. The proofs are given in Section \ref{proof:sgda-stab} of the Appendix.

%\yiming{where the proofs are provided in the Appendix??}
\begin{theorem}[Weak PD risk bounds\label{thm:gen-sgda}]
Suppose Assumption \ref{ass:lipschitz} holds. 
Assume for all $z$, the function $(\bw,\bv)\mapsto f(\bw,\bv;z)$ is convex-concave.  
Let { $\W=\R^d$} and $\{ \bw_j,\bv_j\}_{j=1}^T$ be produced by MC-SGDA  with $\eta_j\equiv\eta$. Let $\A$ be defined by $\A_\bw(S)=\bar{\bw}_T$ and $\A_\bv(S)=\bar{\bv}_T$ for $(\bar{\bw}_T,\bar{\bv}_T)$ in \eqref{eq:avgoutput-sgda}. Denote $\epsilon_{gen}^w:=\triangle^w(\bar{\bw}_T,\bar{\bv}_T)-\triangle^w_\emp(\bar{\bw}_T,\bar{\bv}_T)$.
\vspace*{-1mm}
\begin{enumerate}[label=(\alph*), leftmargin=*]\setlength\itemsep{-1mm}
 \item (Smooth case) If Assumption \ref{ass:smooth} holds and $\sum_{j=1}^{T}\eta_j^2\leq1/(2L^2)$, then 
\[ \epsilon_{gen}^w\leq 4G^2\big(\frac{1}{n}\sum_{j=1}^{T}\eta_j^2\big)^{1/2}+\frac{8\sqrt{2}G^2}{n}\sum_{j=1}^{T}\eta_j.\]
\item (Non-smooth case) The weak PD risk satisfies 
\[\epsilon_{gen}^w\leq 2\sqrt{2}G^2\big(\sum_{j=1}^{T}\eta_j^2\big)^{1/2}+\frac{2\sqrt{2}G^2}{n}\sum_{j=1}^{T}\eta_j.\]
\end{enumerate}
\end{theorem}

\begin{theorem}[Primal risk bounds\label{thm:gen-sgda-primal}]
Suppose Assumption \ref{ass:lipschitz} holds.  Assume for all $z$, the function $(\bw,\bv)\mapsto f(\bw,\bv;z)$ is convex-concave, and the function $\bv\mapsto F(\bw,\bv)$ is $\rho$-strongly-concave.  
Let {$\W=\R^d$} and let $\{ \bw_j,\bv_j\}_{j=1}^T$ be produced by MC-SGDA  with $\eta_j\equiv\eta$. Let $\A$ be defined by $\A_\bw(S)=\bar{\bw}_T$ and $\A_\bv(S)=\bar{\bv}_T$ for $(\bar{\bw}_T,\bar{\bv}_T)$ in \eqref{eq:avgoutput-sgda}. Denote $\epsilon_{gen}^p:=\ebb_{S,\A}\big[R(\bar{\bw}_T)-R_S(\bar{\bw}_T)\big]$.
 %\vspace*{-1mm}
\begin{enumerate}[label=(\alph*), leftmargin=*]\setlength\itemsep{-1mm}
\item (Smooth case) If Assumption \ref{ass:smooth} holds and $\sum_{j=1}^{T}\eta_j^2\leq1/(2L^2)$, then 
\[\epsilon_{gen}^p\leq 4G^2(1+L/\rho)\big(\big(\frac{1}{n}\sum_{j=1}^{T}\eta_j^2\big)^{\frac{1}{2}}+\frac{2\sqrt{2}}{n}\sum_{j=1}^{T}\eta_j\big).\]
 %\vspace*{-1mm}
\item (Non-smooth case) The primal population risk satisfies 
\[\epsilon_{gen}^p\leq 2\sqrt{2}G^2(1+L/\rho)\big(\big(\sum_{j=1}^{T}\eta_j^2\big)^{1/2}+\frac{2}{n}\sum_{j=1}^{T}\eta_j\big).\]
\end{enumerate}
\end{theorem}

\subsection{Population Risks of MC-SGDA}
 
Now we establish the population risk bounds for MC-SGDA.
%Let $(\bw^{*},\bv^{*})$ be a saddle point of $F$, i.e., for any $\bw\in \W, \bv\in \V$, there holds $F(\bw^{*},\bv)\le F(\bw^*,\bv^*)\le F(\bw,\bv^{*})$. 
The following theorem establishes the weak PD population risk of MC-SGDA for both smooth and non-smooth problems. Let $D_\bw$ and $D_\bv$ be the diameters of $\W$ and $\V$. The proof for Theorem \ref{thm:weakPD-risk} is provided in Appendix \ref{sec:thm10-11}.  
 \begin{theorem}[Weak PD population  risk]\label{thm:weakPD-risk}
Suppose Assumptions~\ref{ass:Markov-chain}, \ref{ass:reversible} and \ref{ass:lipschitz} hold. Assume for all $z$, the function $(\bw,\bv)\mapsto f(\bw,\bv;z)$ is convex-concave. 
Let $\{ \bw_j,\bv_j\}_{j=1}^T$ be produced by MC-SGDA with $\eta_j\equiv\eta$. Let $\A$ be defined by $\A_\bw(S)=\bar{\bw}_T$ and $\A_\bv(S)=\bar{\bv}_T$ for $(\bar{\bw}_T,\bar{\bv}_T)$ in \eqref{eq:avgoutput-sgda}. 
%\vspace*{-1mm}
\begin{enumerate}[label=(\alph*), leftmargin=*]\setlength\itemsep{-1mm}
\item (Smooth case)  Let Assumption \ref{ass:smooth} hold. If $T\asymp n$ and $\eta\asymp  ( T \log(T))^{-\frac{1}{2}}$, then \[\triangle^w(\bar{\bw}_T,\bar{\bv}_T)= \O\big(   \log(n)/\big({\sqrt{n}\log(1/\lambda(P)) }\big)\big). \]  
\item (Non-smooth case)  If we select $T\asymp n^2$ and $\eta\asymp T^{-\frac{3}{4}}$, then we have \[\triangle^w(\bar{\bw}_T,\bar{\bv}_T)= \O\big( 1/{\big(\sqrt{n}\log(1/\lambda(P))\big)}\big) .\] 
\end{enumerate}
\end{theorem}
\begin{remark}
{The above excess population risk bounds are obtained through the trade-off between the optimization errors (convergence analysis) and stability results of MC-SGDA. The   convergence rates $\tilde{\O}(1/\sqrt{T})$ of MC-SGDA for minimax problems in both expectation and high probability are provided in 
Theorem~\ref{thm:opt-sgda} and \ref{thm:opt-hp-sgda} in Appendix~\ref{appendix-opt-sgda}. }   %These match the optimal rate up to a logarithmic factor.  
With gradient complexity $\O(n)$, 
the minimax optimal excess risk bound $\O(1/\sqrt{n})$ for SGDA with uniform sampling for smooth problems was established in \cite{lei2021stability}. We show SGDA with Markov sampling can achieve the nearly optimal bound with the same gradient complexity.   For non-smooth problems,  part (b) shows that the optimal excess risk bound can be exactly achieved with the gradient complexity $\O(n^2)$. 
\end{remark}
Finally, we establish the following bounds for excess primal population risk under a strong concavity condition on $\bv\mapsto F(\bw,\bv)$, which measures the performance of the primal variable. The proof for Theorem \ref{thm:excess-primal} is provided in Appendix \ref{sec:thm10-11}.  
\begin{theorem}[Excess primal population risk]\label{thm:excess-primal}
Suppose Assumptions~\ref{ass:Markov-chain},  \ref{ass:reversible}, \ref{ass:lipschitz} and \ref{ass:smooth} hold. Assume for all $z$, the function $(\bw,\bv)\mapsto f(\bw,\bv;z)$ is convex-concave. Assume  $\bv\mapsto F(\bw,\bv)$ is $\rho$-strongly-concave. 
Let $\{ \bw_j,\bv_j\}_{j=1}^T$ be produced by MC-SGDA with $\eta_j\equiv\eta$. Let $\A$ be defined by $\A_\bw(S)=\bar{\bw}_T$ and $\A_\bv(S)=\bar{\bv}_T$ for $(\bar{\bw}_T,\bar{\bv}_T)$ in \eqref{eq:avgoutput-sgda}. If we choose $T\asymp n,\eta\asymp (T\log(T))^{-1/2}$,  then 
\[\E_{S,\A}[R(\bar{\bw}_T)]-\min_{\bw\in\W}R(\bw)=\O\big( {(L/\rho)\sqrt{\log(n)}}/({\sqrt{n}\log(1/\lambda(P))})\big).\]
\end{theorem}

\begin{remark}
We show MC-SGDA attains population risk bounds of the order $\tilde{O}(1/\sqrt{n})$ with a linear gradient complexity $\O(n)$, which are minimax optimal up to a logarithmic factor. This implies that considering sampling with a Markov chain does not weaken the learnability. Theorems~\ref{thm:weakPD-risk} and \ref{thm:excess-primal} also show the effect of $P$ on the population risk rates, i.e., the rates get better as $\lambda(P)$ decreases. 
\end{remark}

\section{Conclusion}\label{sec:conclu}

\vspace*{-2mm}
We develop the first-ever-known stability and generalization analysis of Markov chain stochastic gradient methods for both minimization and minimax objectives. In particular, 
we establish the optimal excess population bounds $\O(1/\sqrt{n})$ for MC-SGD for both smooth and non-smooth cases. We also develop the first nearly optimal convergence rates $\tilde{\O}(1/\sqrt{T})$ for convex-concave problems of MC-SGDA, and  show that the optimal risk bounds $\O(1/\sqrt{n})$ can be derived even in the non-smooth case.  Although the gradients from Markov sampling are biased and not independent across the iterations, we show the performance of MC-SGMs is competitive compared to SGMs with the classical i.i.d.  sampling scheme. An interesting direction is to consider other variants of SGMs with variance reduction techniques and  differentially private SGMs under the Markov sampling scheme. 

\medskip 

\noindent{\bf Acknowledgement.}  The work described in this paper is partially done when the last author, Ding-Xuan Zhou, worked at City University of Hong Kong, supported by
the Laboratory for AI-Powered Financial Technologies under the InnoHK scheme, the Research Grants Council of Hong Kong [Projects No. CityU 11308121, No. N\_CityU102/20, and No. C1013-21GF], the National Science Foundation of China [Project No. 12061160462], and the Hong Kong Institute for Data Science. The corresponding author is Yiming Ying whose work is supported by SUNY-IBM AI Alliance Research and NSF grants (IIS-2103450, IIS-2110546 and DMS-2110836).

%\newpage
\onecolumn
\appendix
\numberwithin{equation}{section}
\numberwithin{theorem}{section}
\numberwithin{remark}{section}
\numberwithin{figure}{section}
\numberwithin{table}{section}
\renewcommand{\thesection}{{\Alph{section}}}
\renewcommand{\thesubsection}{\Alph{section}.\arabic{subsection}}
\renewcommand{\thesubsubsection}{\Roman{section}.\arabic{subsection}.\arabic{subsubsection}}
\setcounter{secnumdepth}{-1}
\setcounter{secnumdepth}{3}

\vspace*{0cm}
\begin{center}
  \Large \textbf{Appendix for ``Stability and Generalization of Markov Chain Stochastic Gradient Methods''}
\end{center}

\section{Technical Lemmas}

Starting from a deterministic and arbitrary initialization $\bw_0$, the iteration of MC-SGMs is illustrated by the following diagram:
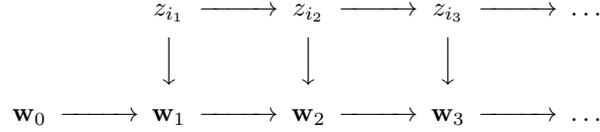
\begin{equation*} 
    \CD
  @.       z_{i_1} @>  >>z_{i_2} @> >> z_{i_3}  @> >> \ldots\\
  @.       @V  VV @V  VV @V    VV @.   \\
  \bw_0 @> >> \bw_1 @> >> \bw_2 @> >> \bw_3 @>>>\ldots
    \endCD
\end{equation*}

The Jordan normal form of transition matrix  $P$ \cite{oldenburger1940infinite} is
\[  P=U\begin{bmatrix}
    %\left(
    %     \begin{array}{cccc}
           1 &   &   &   \\
             & J_2 &   &   \\
             &   & \ddots &   \\
             &   &   & J_m
%             \\
%         \end{array}
%       \right)
\end{bmatrix} U^{-1}, \]
where $m$ is the number of the blocks, $d_i\ge 1$ is the dimension of the $i$-th block submatrix $J_i$, $i=2,3,\ldots,m$, which satisfy $\sum_{i=1}^m d_i=n$,  and matrix $J_i:=\lambda_i(P)\cdot\mathbf{I}_{d_i}+\mathbf{D}(-1,d_i)$ with $\mathbf{D}(-1,d_i):=
%\left(
%             \begin{array}{cccc}
\begin{bmatrix}
               0 & 1  &   &   \\
                & \ddots & \ddots   &   \\
                 &  & \ddots & 1  \\
                 &   &   & 0 \\
             %\end{array}
           %\right)
          \end{bmatrix}_{d_i\times d_i}
$. Here $\mathbf{I}_{d_i}$ is the identity matrix of size $d_i$. In particular, if $P$ is symmetric, then it is double stochastic and there holds $d_i=1$ for any $i\in[m]$.

To establish the optimization error for MC-SGMs, we need the following lemma which  gives the mixing time of a Markov chain.   
\begin{lemma}[\cite{Sun2018Markov}]\label{lem:difference-stationary} Suppose Assumption 1 holds. Let $\lambda_i(P) $ be the $i$-th largest eigenvalue of $P$, $ \lambda(P)  = \frac{\max\{ |\lambda_2(P)|, |\lambda_n(P)|\} + 1 }{2} \in  [1/2,1 ) ,$ $C_P= \big(\sum_{i=2}^m d_i^2 \big)^{1/2} \|U\|_{F}\|U^{-1}\|_F   $ and \[K_P =\max\Big\{\max_{1\leq i\leq m}\big\{    \big\lceil\frac{2d_i(d_i-1)(\log(\frac{2d_i}{|\lambda_2(P)|\cdot \log (\lambda(P)/|\lambda_2(P)|)})-1)}{(d_i+1)\log(\lambda(P)/|\lambda_2(P)|)}\big\rceil\big\} ,~0\Big\}.\]
For any $j\ge K_P$, there holds
\[\big\|\Pi^{*}-P^j\big\|_{\infty} \le C_P \cdot \big(\lambda(P) \big)^j.\]
In addition, if $P$ is symmetric, then $K_P=0$ and
\[ \big\|\Pi^{*}-P^j\big\|_{\infty} \le n^{3/2} \cdot \big(\lambda(P) \big)^j,  \text{ for any } j\ge 0. \]
\end{lemma}

%We first introduce some useful lemmas. %Let $c_{\alpha,1}=(1+1/\alpha)^{\frac{\alpha}{1+\alpha}}L^{\frac{1}{1+\alpha}}$ if $\alpha>0$ and $c_{\alpha,1}=\sup_z\|\partial f(0;z)\|_2+L$ if $\alpha=0$.
%Note we require a convexity assumption in Part (c) by considering non-smooth loss functions.

The following lemma shows  the non-expansive behavior for the gradient mapping $\bw \mapsto \bw-\eta \partial f(\bw;z_i)$ associated with a smooth function.
\begin{lemma}[\cite{hardt2016identity}]\label{lem:non-expensive}
Suppose the loss $f$ is convex and $L$-smooth  w.r.t. the first argument. Then for all $\eta\le 2/L$ and $z\in \Z$ there holds
 \[\|\bw-\partial  f(\bw;z ) -\bw' + \eta \partial f(\bw';z )\|_2\le \|\bw-\bw'\|_2.\]
\end{lemma}

\begin{lemma}[\cite{schmidt2011}\label{eq:recursive0}] Assume that the non-negative sequence $\{u_t: t\in \mathbb{N}\}$ satisfies the following recursive inequality for all $t\in \mathbb{N}$, 
$$ u_t^2 \le S_t + \sum_{\tau=1}^{t-1} \alpha_\tau u_\tau.$$
where $\{S_\tau:  \tau\in\mathbb{N}\}$ is an increasing sequence, $S_0 \ge u_0^2$ and $\alpha_\tau\ge 0$ for any $\tau\in \mathbb{N}.$  Then, the following inequality holds true: 
$$ u_t \le  \sqrt{S_t} +  \sum_{\tau=1}^{t-1} \alpha_\tau. $$
 \end{lemma}

The connection between  on-average argument stability  for Lipschitz continuous losses  and its generalization error has been studied in \cite{lei2020fine}.  
\begin{lemma}[Generalization via argument  stability\label{thm:gen-model-stab}]
Let $S,\widetilde{S}$ and $S^{(i)}$ be defined as Definition~\ref{def:on-average-stability}. If $\A$ is on-average $\epsilon$-argument-stable and Assumption \ref{ass:G-lipschitz} holds, then there holds $
  \big|\ebb_{S,\A}\big[F(\A(S))-F_S(\A(S))\big]\big|\leq G\epsilon.$
\end{lemma}

Finally,  we introduce the following lemma on concentration inequality of martingales. 
\begin{lemma}[\cite{boucheron2013concentration}]\label{lem:martingle}
	Let $z_1,\ldots,z_n$ be a sequence of random variables. Consider a sequence of functionals $\xi_k(z_1,\ldots,z_k)$, $k\in[n]$. Assume $|\xi_k-\E_{z_k}[\xi_k]|\le b_k$ for each $k$. Let $\gamma\in(0,1)$. With probability at least $1-\gamma$, there holds
 		\[ \sum_{k=1}^n \E_{z_k}[\xi_k] - \sum_{k=1}^n \xi_k \le \Big(2\log(1/{\gamma}) \sum_{k=1}^n b_k^2 \Big)^{1/2}. \]
\end{lemma} 

\section{Proofs of Markov Chain SGD}
%Let  $S$, $\widetilde{S}$,  $S^{(i)}$ be constructed as Definition~\ref{def:on-average-stability}, and $\{\bw_t\}$ and $\{\bw_t^{(i)}\}$ be produced by MC-SGD update rule \eqref{eq:MCSGD-update-rule} based on $S$ and $S^{(i)}$, respectively.  For simplicity, we denote by $\delta_t^{(i)}=\| \bw_t - \bw_t^{(i)} \|_2$ in this section.

\subsection{Proof of Theorem~\ref{thm:stability}   } \label{proof:stability} 
\begin{proof}[Proof of Theorem~\ref{thm:stability}]
For any $i\in[n]$, define
$S^{(i)}=\{z_1,\ldots,z_{i-1},\tilde{z}_i,z_{i+1},\ldots,z_n\}$ as the set formed from $S$ by replacing the $i$-th element with $\tilde{z}_i$. Let $\{\bw_t\}$ and $\{\bw_t^{(i)}\}$ be produced by MC-SGD based on $S$ and $S^{(i)}$, respectively.  For simplicity, we denote by $\delta_t^{(i)}=\| \bw_t - \bw_t^{(i)} \|_2$ here. Note that the projection step is nonexpansive.

We first prove part (a).  Consider the following two cases.

\noindent \textbf{Case 1.} If $i_t \neq i$,  then it follows from the $L$-smoothness of $f$ and  Lemma~\ref{lem:non-expensive}  that
\begin{align*}
    & \delta_t^{(i)}  \le \| \bw_{t-1} - \eta_t \partial f(\bw_{t-1};z_{i_t}) -  \bw^{(i)}_{t-1} + \eta_t \partial f(\bw^{(i)}_{t-1};z_{i_t}) \|_2 \le \delta_{t-1}^{(i)}.
\end{align*}

\noindent \textbf{Case 2.} If $i_t=i$, then it follows from the Lipschitz continuity of $f$ that
\begin{align*} 
   \delta_{t }^{(i)} &\le \| \bw_{t-1} - \eta_t \partial f(\bw_{t-1};z_{i_t}) -  \bw^{(i)}_{t-1} + \eta_t \partial f(\bw^{(i)}_{t-1};\tilde{z}_{i_t}) \|_2 \\
   &\le   \delta_{t -1}^{(i)}  + \eta_t \| \partial f(\bw_{t-1};z_{i}) -   \partial f(\bw^{(i)}_{t-1};\tilde{z}_{i}) \|_2\le \delta_{t -1}^{(i)} + 2G\eta_t.
\end{align*}
Combining the above two cases together, we know
\begin{align*}
    \delta_{t }^{(i)}  \le   \delta_{t-1 }^{(i)}  +  2G\eta_t    \mathbb{I}_{[i_t=i]},
\end{align*}
where $\mathbb{I}_{[i_t=i]}$ is the indicator function, i.e., $\mathbb{I}_{[i_t=1]}=1$ if $i_t=i$ and $0$ else.
Now, applying the above inequality recursively  we have
\begin{align*}
   \delta_{t }^{(i)}   \le   2G \sum_{j=1}^t \eta _j \mathbb{I}_{[i_j=i]}, %=  2G \eta \sum_{j=1}^t  \mathbb{I}_{[i_j=i]}  
\end{align*}
%here we use $\eta_t=\eta$. 
%Define $[P^0]_{k,i}=1$ if $k=i$ and $0$ otherwise. For any $j\ge 1$, there holds\begin{align}\label{eq:expectation-indicator}\E [\mathbb{I}_{[i_j=i]} ] %&= \sum_{i=1}^n \mathbb{P}( i_j=i   ) \mathbb{I}_{[i =1]} \\= \sum_{k=1}^n \mathbb{P}( i_j=i| i_1=k ) \pi_1(k) =  \sum_{k=1}^n [P^{j-1}]_{k,i} \pi_1(k).\end{align}
%Taking expectation over the randomness of the algorithm, we obtain\begin{align*}\E_\A[  \delta_{t }^{(i)}] \le  2G \sum_{j=1}^t \eta_j \E_\A[\mathbb{I}_{[i_j=i]} ].%\le 2G\eta \sum_{j=1}^t \sum_{k=1}^n [P^{j-1}]_{k,i}\pi_1(k),\end{align*}
%where  the last inequality follows from  Eq.\eqref{eq:expectation-indicator}. 
By the convexity of $\|\cdot\|_2$, there holds
\begin{align*}
       \|    \bar{\bw}_t-\bar{\bw}^{(i)}_t \|_2    \le \frac{1}{t}\sum_{j=1}^t   \delta_{j}^{(i)} \le  2G\sum_{j=1}^t\eta_j  \mathbb{I}_{[i_j=i]}  . % 2G\eta \sum_{j=1}^t \sum_{k=1}^n [P^{j-1}]_{k,i}\pi_1(k).
\end{align*} 
Taking an average over $i$ yields 
\begin{align*}
 \frac{1}{n} \sum_{i=1}^n \| \bar{\bw}_t-\bar{\bw}_t^{(i)} \|_2  \le   \frac{2G }{n} \sum_{j=1}^t \eta_j    \sum_{i=1}^n  \mathbb{I}_{[i_j=i]}   \le  \frac{2G }{n}\sum_{j=1}^t \eta_j ,
\end{align*} 
 where   the last inequality used the fact that  $\sum_{i=1}^n  \mathbb{I}_{[i_j=i]} =1$ for any $ j\in[t]$. 
 Taking  expectation w.r.t. $\A$, we have
 \begin{align*}
\E_{ \A} \Big[ \frac{1}{n} \sum_{i=1}^n \| \bar{\bw}_t-\bar{\bw}_t^{(i)} \|_2 \Big]\le  \frac{2G}{n}\sum_{j=1}^t \eta_j ,
\end{align*}
which completes the proof of Part (a).

\medskip
Now, we turn to the non-smooth case. 
Similar as before,  we consider the following two cases. 

\noindent \textbf{Case 1.} If $i_t\neq i$, then we have 
\begin{align*}
	 \big(\delta_t^{(i)}\big)^2 &\le \| \bw_{t-1} - \eta_t \partial f(\bw_{t-1};z_{i_t}) -  \bw^{(i)}_{t-1} + \eta_t \partial f(\bw^{(i)}_{t-1};z_{i_t}) \|^2_2 \\
	 &= \big(\delta_{t-1}^{(i)}\big)^2 + \eta_t^2\| \partial f(\bw_{t-1};z_{i_t})-\partial f(\bw^{(i)}_{t-1};z_{i_t}) \|_2^2 -2\eta_t\langle \bw_{t-1} -    \bw^{(i)}_{t-1},   \partial f(\bw_{t-1};z_{i_t}) - \partial f(\bw^{(i)}_{t-1};z_{i_t}) \rangle \\
	 &\le  \big(\delta_{t-1}^{(i)}\big)^2 + 2 \eta_t^2(\| \partial f(\bw_{t-1};z_{i_t})\|_2^2+\|\partial f(\bw^{(i)}_{t-1};z_{i_t}) \|_2^2)  \\
	 &\le \big(\delta_{t-1}^{(i)}\big)^2 + 4G^2\eta_t^2, 
\end{align*}
 where in the last second inequality we used $\langle \bw_{t-1}-\bw^{(i)}_{t-1},   \partial f(\bw_{t-1};z_{i_t}) - \partial f(\bw^{(i)}_{t-1};z_{i_t}) \rangle \ge 0  $ due to the convexity of $f$, and  the last inequality follows from the Lipschitz continuity of $f$. 
 
\noindent \textbf{Case 2.} If $i_t = i$, then 
\begin{align}\label{eq:stab-holder-1} 
   (\delta^{(i)}_t)^2 &\le\| \bw_{t-1} - \eta_t \partial f(\bw_{t-1};z_{i }) -  \bw^{(i)}_{t-1} + \eta_t \partial f(\bw^{(i)}_{t-1};\tilde{z}_{i }) \|^2_2 \nonumber\\
    &= (\delta^{(i)}_{t-1})^2 + \eta_t^2 \|\partial f(\bw_{t-1};z_{i }) - \partial f(\bw^{(i)}_{t-1};\tilde{z}_{i }) \|_2^2 -2 \eta_t \langle \bw_{t-1}  - \bw_{t-1}^{(i)}, \partial f(\bw_{t-1};z_{i }) - \partial f(\bw^{(i)}_{t-1};\tilde{z}_{i }) \rangle\nonumber\\
    &\le (\delta^{(i)}_{t-1})^2 + 2 \eta_t^2 \big( \|\partial f(\bw_{t-1};z_{i })\|_2^2 + \|\partial f(\bw^{(i)}_{t-1};\tilde{z}_{i })\|_2^2 \big) + 2 \eta_t \delta^{(i)}_{t-1} \big( \|\partial f(\bw_{t-1};z_{i })\|_2 + \| \partial f(\bw^{(i)}_{t-1};\tilde{z}_{i })\|_2 \big)\nonumber\\
    &\le (\delta^{(i)}_{t-1})^2 + 4G^2 \eta_t^2 + 4G \eta_t \delta^{(i)}_{t-1},
 \end{align}
where the last inequality holds since $f$ is $G$-Lipschitz. 
 
Combining Case 1 and Case 2 together, we have
\begin{align*}
   (\delta^{(i)}_t)^2  
    &\le  (\delta^{(i)}_{t-1})^2      +   4G^2 \eta_t^2 + 4G \eta_t \delta^{(i)}_{t-1}   \mathbb{I}_{[i_t=i]}. 
\end{align*}
Note that $\delta^{(i)}_0=\|\bw_0-\bw_0^{(i)}\|_2=0$, we get the following recursive inequality
\begin{align*} 
    (\delta^{(i)}_t)^2 &\le  4G^2 \sum_{j=1}^t \eta_j^2  + 4G\sum_{j=1}^t  \eta_j \delta^{(i)}_{j-1} \mathbb{I}_{[i_{j}=i]}=4G^2 \sum_{j=1}^t \eta_j^2  + 4G\sum_{j=1}^{t-1}  \eta_{j+1} \delta^{(i)}_{j} \mathbb{I}_{[i_{j+1}=i]}. 
\end{align*}
Lemma~\ref{eq:recursive0} with $u_t=\delta_t^{(i)}$ implies 
\begin{align*}
    \delta_t^{(i)}&\le 2G \sqrt{ \sum_{j=1}^t \eta_j^2  } + 4G\sum_{j=1}^{t-1}  \eta_{j+1}  \mathbb{I}_{[i_{j+1}=i]}.
\end{align*}
By the convexity of $\|\cdot\|_2,$ it follows 
\begin{align*}
     \|    \bar{\bw}_t-\bar{\bw}^{(i)}_t \|_2    \le \frac{1}{t}\sum_{j=1}^t   \delta_j^{(i)}\le 2G \sqrt{ \sum_{j=1}^t \eta_j^2  } + 4G\sum_{j=1}^{t-1}  \eta_{j+1}  \mathbb{I}_{[i_{j+1}=i]}.
\end{align*} 
Taking an average over $i$, we have
\begin{align*}
\frac{1}{n} \sum_{i=1}^n   \| \bar{\bw}_{t} - \bar{\bw}_{t}^{(i)}\|_2  & \le 2G \sqrt{ \sum_{j=1}^t \eta_t^2 } + \frac{4G}{n} \sum_{j=1}^{t-1} \eta_{j+1} \sum_{i=1}^n    \mathbb{I}_{[i_{j+1}=i]} \\ & \le 2G \sqrt{ \sum_{j=1}^t \eta_t^2 }  + \frac{4G }{n} \sum_{j=1}^t \eta_j,
\end{align*}
where the last inequality used the fact that $\sum_{i=1}^n  \mathbb{I}_{[i_{j+1}=i]} =1$. 
Taking the expectation w.r.t. $\A$  completes the proof.
\end{proof}

\bigskip

\subsection{Optimization Error of MC-SGD for Convex Problems}\label{appendix-opt-sgd}
 In this subsection, we establish the convergence rates of MC-SGD for convex and non-convex problems. We consider both upper bounds in expectation and with high probability.   

Recall $\bw^{*}= \arg\min_{\bw\in\W} F (\bw)$. 
Let $\lambda_i(P)$ be the $i$-th largest eigenvalue of transition matrix $P$ and $\lambda(P)=(\max\{|\lambda_2(P)|,|\lambda_n(P)|\}+1)/2 \in [1/2,1)$. 
Let $K_P$ be the mixing time and $C_P$ be a constant depending on $P$ and its Jordan canonical form (detailed expressions are given in Lemma~\ref{lem:difference-stationary} in the Appendix).  
 %To examine the optimization error,  we use the decomposition $ \E_\A[ F_S( \bw_j )- F_S(\bw^{*}) ]=   \E_\A\big[  F_S( \bw_j ) -F_S(\bw_{j-k_j}) \big]   +  \E_\A  \big[   F_S(\bw_{j-k_j})-F_S(\bw^{*})  \big] $, where $k_j$ is a constant defined in the following theorem. By a careful choice of $k_j$, the extra term $\sum_{i=1}^n \big( [P^{k_j}]_{i_{j-k_j},i} -1/n \big) \cdot [ f(\bw_{j-k_j};z_{i}) - f(\bw^{*} ; z_{i} ) ]$ due to Markov sampling can be controlled in the order $\O(1/j)$. So that the convergence rate $\O(1/\sqrt{n})$ can still be attained. 

%\yunwen{$\lambda(P)$ and $C_P$ are not defined}

Theorem~\ref{thm:opt-convex} gives optimization error bounds in expectation for MC-SGD in the convex case.  
\begin{theorem}[Convex case]\label{thm:opt-convex}   Suppose  $f$ is convex and Assumptions \ref{ass:Markov-chain}, \ref{ass:G-lipschitz} hold. Let $\A$ be MC-SGD with $T$ iterations,  and $\{\bw_j\}_{j=1}^T$ be produced by $\A$ with $\bw_0=0$ and $\eta_j\equiv\eta \le 2/L$. Let $D_0=\|\bw^{*}\|_2$, $D=\big((G^2+2\sup_{z\in\Z}f(0;z))\sum_{k=1}^T\eta_k\big)^{1/2}+D_0$  and 
\begin{equation}%\label{eq:kj-sgd-convex}
    \!\!\!\!\!k_j\!=\!\min\!\Big\{\!\max\Big\{\Big\lceil \frac{\log(2C_PD n j)}{\log(1/\lambda(P))}\Big\rceil, K_P \Big\}, j\Big\}, j\in[T].
\end{equation}  
%here, $K_P$ and $C_P$ are related to transition matrix $P$, which are defined in Lemma~\ref{lem:difference-stationary}.  
Then the following inequality holds
 \begin{align*}
    &\E_\A[ F_S(\bar{\bw}_T)\!-\!  F_S(\bw^{*}) ]\!\le\! \frac{D_0\!+\!G\big( 4 D\sum_{j=1}^{K_P-1}\eta_j \! +\!  \sum_{j=K_P}^{T}\!\frac{\eta_j}{ j}\!+ \! G   \sum_{j=1}^T ( 4 \eta_j      \sum_{k=j-k_j+1}^j \eta_k  +   \eta^2_j )\big) }{2\sum_{j=1}^T \eta_j}.
 %\frac{\|\bw_0 - \bw\|_2^2  +4 G^2  \sum_{j=1}^t  \eta_j      \sum_{k=j-k_j+1}^j \eta_k  + \sum_{j=1}^t (G^2\eta^2_j+ G\eta_j / j)}{2\sum_{j=1}^t  \eta_j} ,  
 \end{align*}
%where $C=\|\bw_0 -\bw^{*} \|_2^2+ 2GD\sum_{j=1}^{K_P-1} \eta_j$.
Furthermore, suppose Assumption~\ref{ass:reversible} holds.  Then selecting $\eta_j\equiv\eta=1/ \sqrt{T \log(T) }  $ implies $\E_\A[ F_S(\bar{\bw}_T) - F_S(\bw^{*})]=\O\big(\sqrt{\log(T)}/(\sqrt{T}  \log( 1/\lambda(P)  ) \big)$. 
\end{theorem}
\begin{remark}\label{rmk:assump4}
 {
For all $j\in[T]$, $k_j$ can be seen as the mixing time such that the distance between the distribution of the current state of the Markov chain and that of the stationary distribution can be controlled by $\O(1/(n j))$.}
\end{remark}
Now, we give the proof of Theorem~\ref{thm:opt-convex}. 
\begin{proof}[Proof of Theorem~\ref{thm:opt-convex}\label{subsec:proof-opt-convex}]
By the convexity of $f$ and Jensen's inequality, 
\begin{align}\label{eq:opt-1}
    &\Big( \sum_{j=1}^T  \eta_j \Big)  \E_\A[ F_S(\bar{\bw}_t )-F_S(\bw^{*}) ]\nonumber\\
    &\le    \sum_{j=1}^T  \eta_j    \E_\A[ F_S( \bw_j )-F_S(\bw^{*}) ] \nonumber\\ 
    & = \sum_{j=1}^T  \eta_j    \E_\A\big[  F_S( \bw_j ) -F_S(\bw_{j-k_j}) \big]  + \sum_{j=1}^T  \eta_j \E_\A  \big[   F_S(\bw_{j-k_j})-F_S(\bw^{*})  \big],  
\end{align}
where $k_j=\min\Big\{\max\Big\{\Big\lceil \frac{\log(2C_PD n j)}{\log(1/\lambda(P))}\Big\rceil, K_P \Big\}, j\Big\}$.

Consider the first term $\sum_{j=1}^T  \eta_j    \E_\A\big[  F_S( \bw_j ) -F_S(\bw_{j-k_j}) \big] $ in \eqref{eq:opt-1}. By Lipschitz continuity of $f$, we have   
\begin{align}\label{eq:opt-7}
    \sum_{j=1}^T  \eta_j    \E_\A\big[ F_S( \bw_j ) -F_S(\bw_{j-k_j}) \big] 
    %&=\sum_{j=1}^t  \eta_j   \E\big[  \frac{1}{n} \sum_{i=1}^n \big(   f(\bw_j;z_{i } )  - f(\bw_{j-k_j};z_{i })\big) \big] \nonumber\\
      \le     G \sum_{j=1}^T  \eta_j       \E_\A\big[ \|\bw_j   -  \bw_{j-k_j}\|_2 \big] \le   G^2  \sum_{j=1}^T  \eta_j        \sum_{k=j-k_j+1}^j    \eta_k, 
\end{align}
where  the last inequality used the fact that $\|\bw_j   -  \bw_{j-k_j}\|_2\le \sum_{k=j-k_j+1}^j \eta_k\|\partial f(\bw_{k-1};z_{i_k})\|_2\le G\sum_{k=j-k_j+1}^j \eta_k$.

Next, we estimate the term $\sum_{j=1}^T\eta_j\E_\A[F_S(\bw_{j-k_j})-F(\bw^{*})]$ in \eqref{eq:opt-1}.   Note that
\begin{align}\label{eq:opt-9}
    &\E_{i_j}[ f(\bw_{j-k_j};z_{i_j})\!-\!f(\bw^{*} ; z_{i_j} )| \bw_{0},\ldots,\bw_{j-k_j}, z_{i_1},\ldots, z_{i_{j-k_j}}  ]\nonumber\\
    &= \sum_{i=1}^n [ f(\bw_{j-k_j};z_{i}) - f(\bw^{*} ; z_{i} ) ] \cdot \Pr ( i_j=i | i_{j-k_j} )\nonumber\\
    & =  \sum_{i=1}^n [ f(\bw_{j-k_j};z_{i}) - f(\bw^{*} ; z_{i} ) ] \cdot [P^{k_j}]_{i_{j-k_j},i}\nonumber\\
    & = \frac{1}{n} \sum_{i=1}^n [ f(\bw_{j-k_j};z_{i}) - f(\bw^{*} ; z_{i} ) ] +  \sum_{i=1}^n \Big( [P^{k_j}]_{i_{j-k_j},i} - \frac{1}{n} \Big) \cdot [ f(\bw_{j-k_j};z_{i}) - f(\bw^{*} ; z_{i} ) ] \nonumber\\
    & = \big( F_S(\bw_{j-k_j}) - F_S(\bw^{*} ) \big)   +  \sum_{i=1}^n \Big( [P^{k_j}]_{i_{j-k_j},i} - \frac{1}{n} \Big) \cdot [ f(\bw_{j-k_j};z_{i}) - f(\bw^{*} ; z_{i} ) ].
\end{align}
Rearranging the above equality and taking total expectation give us
\begin{align*} 
     \E_\A[F_S(\bw_{j-k_j}) - F_S(\bw^{*} )] 
     = & \E_\A[ f(\bw_{j-k_j};z_{i_j}) - f(\bw^{*} ; z_{i_j} ) ]\\&
   +  \E_\A\Big[\sum_{i=1}^n \Big(  \frac{1}{n} - [P^{k_j}]_{i_{j-k_j},i}  \Big)   \big( f(\bw_{j-k_j};z_{i}) - f(\bw^{*} ; z_{i} )\big) \Big].
\end{align*}
Summing over $j$ yields
\begin{align} \label{eq:opt-2}
      \sum_{j=1}^T \eta_j \E_\A[ F_S(\bw_{j-k_j})-F_S(\bw^{*} ) ]
    %&= \sum_{j=1}^t\eta_j \E[ f(\bw_{j-k_j};z_{i_j}) - f(\bw ; z_{i_j} ) ] -  \sum_{j=1}^t\eta_j \Big( [P^{k_j}]_{i_{j-k_j},i} - \frac{1}{n} \Big) \cdot \sum_{i=1}^n [ f(\bw_{j-k_j};z_{i}) - f(\bw ; z_{i} ) ]\nonumber\\
       &= \sum_{j=1}^T\eta_j  \E_\A[ f(\bw_{j-k_j};z_{i_j}) - f(\bw^{*} ; z_{i_j} ) ]\nonumber\\
     &+ \sum_{j=1}^T \eta_j \E_\A \Big[  \sum_{i=1}^n \Big(\frac{1}{n} -  [P^{k_j}]_{i_{j-k_j},i} \Big) [  f(\bw_{j-k_j};z_{i}) - f(\bw^{*} ; z_{i} ) ]\Big].
    %&\le  \sum_{j=1}^t\eta_j \E[ f(\bw_{j-k_j};z_{i_j}) - f(\bw_S^{*} ; z_{i_j} ) ] +    G D \sum_{j=1}^t   \eta_j   \sum_{i=1}^n  \Big|  \frac{1}{n}- [P^{k_j}]_{i_{j-k_j},i}  \Big| , 
\end{align}
Now, we estimate the term $ \sum_{j=1}^T\eta_j \E_\A[ f(\bw_{j-k_j};z_{i_j}) - f(\bw^{*} ; z_{i_j} ) ]$ in \eqref{eq:opt-2}. 
According to update rule \eqref{eq:MCSGD-update-rule}, for any $j$ and $1\le k_j\le j$  
\begin{align*}
    &\|\bw_j - \bw^{*}\|_2^2 \\&\le \| \bw_{j-1} -\eta_j \partial f(\bw_{j-1};z_{i_j}) - \bw^{*} \|_2^2 \nonumber \\  & =\|\bw_{j-1} -  \bw^{*} \|_2^2 -2\eta_j \langle  \bw_{j-1} -  \bw^{*}, \partial f(\bw_{j-1};z_{i_j}) \rangle + \eta^2_j\|\partial f(\bw_{j-1};z_{i_j})\|_2^2\nonumber\\
    &\le \|\bw_{j-1} -  \bw^{*} \|_2^2 -2\eta_j \big( f(\bw_{j-1};z_{i_j}) - f(\bw^{*};z_{i_j}) \big) + G^2\eta^2_j \nonumber\\
    &= \|\bw_{j-1} -  \bw^{*}\|_2^2 -2\eta_j \big( f(\bw_{j-k_j};z_{i_j}) - f(\bw^{*};z_{i_j}) \big) + 2\eta_j \big( f(\bw_{j-k_j};z_{i_j}) - f(\bw_{j-1};z_{i_j}) \big)   + G^2\eta^2_j \nonumber\\
    &\le \|\bw_{j-1} -  \bw^{*} \|_2^2\! -\!2\eta_j \big( f(\bw_{j-k_j};z_{i_j}) \!-\!f(\bw^{*};z_{i_j}) \big)\!+\!2G^2\eta_j\sum_{k=j-k_j+1}^{j} \eta_k\!+\!G^2 \eta^2_j,
\end{align*}
where the second inequality is due to the convexity of $f$, and the last inequality used the fact $f$ is $G$-Lipschitz and $\|\bw_{j-k_j} -  \bw_{j-1} \|_2 \le   \sum_{k=j-k_j+1}^{j}\eta_k\|\partial f(\bw_k;z_{i_k}) \|_2 \le G \sum_{k=j-k_j+1}^{j}\eta_k$.  
Taking a summation of the both sides   over $j$ and noting $\bw_0=0$, we get
\begin{align}\label{eq:opt-3}
   \sum_{j=1}^T \eta_j  \big(f(\bw_{j-k_j};z_{i_j}) - f(\bw^{*};z_{i_j}) \big)\le \frac{\| \bw^{*}\|_2^2    +  G^2\sum_{j=1}^T \big( 2\eta_j\sum_{k=j-k_j+1}^{j} \eta_k +   \eta^2_j\big)}{2}.
\end{align}
Then we turn to estimate $\sum_{j=1}^T\eta_j     \sum_{i=1}^n \Big(\frac{1}{n}-  [P^{k_j}]_{i_{j-k_j},i}  \Big) [ f(\bw_{j-k_j};z_{i}) - f(\bw^{*} ; z_{i} ) ] $.  
Recall that $k_j= \min \Big\{ \max\Big\{\Big\lceil \frac{\log(2C_PD n j)}{\log(1/\lambda(P))}\Big\rceil, K_P\Big\}, j  \Big\}$.
For $j\ge K_P$, 
according to Lemma~\ref{lem:difference-stationary}, we know for any $i,i'\in[n]$
\begin{align*}
    \Big|  \frac{1}{n} -  [P^{k_j}]_{i, i'}  \Big|\le C_P \big(\lambda(P)\big)^{k_j} = C_P e^{k_j \log(\lambda(P))} \le \frac{1}{2Dnj}. 
\end{align*}
According to the update rule \eqref{eq:MCSGD-update-rule} we know
\[ \|\bw_{t}\|_2^2 \le \| \bw_{t-1} -\eta_t \partial f(\bw_{t-1};z_{i_{t }}) \|_2^2 \le   \| \bw_{t-1}\|_2^2 + \eta_t^2 \| \partial f(\bw_{t-1};z_{i_{t }}) \|_2^2 - 2\eta_t\langle \partial f(\bw_{t-1};z_{i_{t }}), \bw_{t-1} \rangle.\]
The convexity of $f$ implies
\begin{align*}
    \eta_t  \| \partial f(\bw_{t-1};z_{i_t}) \|_2^2 - 2\langle \partial f(\bw_{t-1};z_{i_t}), \bw_{t-1} \rangle&\le    \eta_t  \| \partial f(\bw_{t-1};z_{i_t}) \|_2^2 + 2\Big( f(0;z_{i_t}) - f( \bw_{t-1};z_{i_t} ) \Big)\\
    &\le G^2 + 2\sup_{z\in \Z} f(0;z),
\end{align*}
where we used $\eta_t<1$ and Lipschitz continuity and non-negativity of $f$. 
Combining the above two inequalities together, we get
\[ \|\bw_{t }\|_2^2 \le    \| \bw_{t-1}\|_2^2 + \big(G^2 + 2\sup_{z\in \Z} f(0;z)\big)\eta_t  .\]
Applying the above inequality recursively and noting that $\bw_0=0$,
we get
\[ \|\bw_{t }\|_2^2 \le    \big(G^2 + 2\sup_{z\in \Z} f(0;z)\big)\sum_{k=1}^{t-1} \eta_k\le \big(G^2 + 2\sup_{z\in \Z} f(0;z)\big)\sum_{k=1}^T \eta_k. \]
Recall that $D_0=\|\bw^{*}\|_2$ and $D=\big((G^2+2\sup_{z\in\Z}f(0;z))\sum_{k=1}^T\eta_k\big)^{1/2}+D_0$, it then follows that
\begin{align}\label{eq:opt-8}
     &\sum_{j=K_P}^{T} \eta_j   \sum_{i=1}^n   \Big(  \frac{1}{n}- [P^{k_j}]_{i_{j-k_j},i}  \Big) [ f(\bw_{j-k_j};z_{i}) - f(\bw^{*} ; z_{i} ) ] \nonumber\\
     &\le GD \sum_{j=K_P}^{T} \eta_j  \sum_{i=1}^n   \Big|  \frac{1}{n}- [P^{k_j}]_{i_{j-k_j},i}  \Big|\le \sum_{j=K_P}^{T} \frac{ G \eta_j}{2 j},
\end{align}
where the first inequality used  $|f(\bw_t;z_{i }) - f(\bw^{*} ;z_{i } )|  \le  G \|\bw_t -\bw^{*}  \|_2 \le G(\|\bw_t\|_2 + \|\bw^{*}\|_2 )\le GD$ for any $t\in[T]$.
On the other hand, note that $[P^{k_j}]_{i,i'}\geq 0$ for any  $i,i'\in[n]$ and $k_j$, then
\begin{align*}
    &\sum_{i=1}^n   \Big(  \frac{1}{n}- [P^{k_j}]_{i_{j-k_j},i}  \Big) [ f(\bw_{j-k_j};z_{i}) - f(\bw^{*}; z_{i} ) ]\\&\le  \sum_{i=1}^n   \Big(  \frac{1}{n}+ [P^{k_j}]_{i_{j-k_j},i}  \Big)  | f(\bw_{j-k_j};z_{i}) - f( \bw^{*} ; z_{i} ) |\le 2G\|\bw_{j-k_j} -  \bw^{*} \|_2 \le 2GD ,
\end{align*}
where the last second inequality used Lipschitz continuity of $f$ and $\sum_{j=1}^n [P^x]_{i,j}=1$ for any fixed $i\in[n]$ and $x\ge 1$.  
Therefore, there holds
\begin{align*}
 \sum_{j=1}^{K_P-1} \eta_j \sum_{i=1}^n   \Big(  \frac{1}{n}- [P^{k_j}]_{i_{j-k_j},i}  \Big) [ f(\bw_{j-k_j};z_{i}) - f(\bw^{*}; z_{i} ) ]\le 2GD \sum_{j=1}^{K_P-1} \eta_j.\end{align*}
Combining the above inequality and \eqref{eq:opt-8} together, we get
\begin{align}\label{eq:opt-5}
    \sum_{j=1}^T\eta_j   \sum_{i=1}^n \Big(\frac{1}{n}-  [P^{k_j}]_{i_{j-k_j},i}  \Big) [ f(\bw_{j-k_j};z_{i}) - f(\bw^{*} ; z_{i} ) ]\le \sum_{j=K_P}^{T} \frac{ G \eta_j}{2 j} + 2GD\sum_{j=1}^{K_P-1} \eta_j. 
\end{align}
%Indeed, the second term has no effect on the convergence rate since it is finite. 

Putting \eqref{eq:opt-3} and \eqref{eq:opt-5} back into \eqref{eq:opt-2}, we obtain
\begin{align} \label{eq:opt-6}
    &\sum_{j=1}^T \eta_j \E_\A[ F_S(\bw_{j-k_j}) - F_S(\bw^{*} ) ]\nonumber\\ & 
    \le  \frac{\|  \bw^{*} \|_2^2 +4GD\sum_{j=1}^{K_P-1} \eta_j   }{2}  + \frac{   G^2\sum_{j=1}^T\big(2\eta_j\sum_{k=j-k_j+1}^{j} \eta_k  +   \eta^2_j\big) }{2} + \frac{ G \sum_{j=K_P}^{T}   \eta_j/ j}{2}. 
    %&\le \frac{\|\bw_0 - \bw_S^{*}\|_2^2    + 2G^2   \sum_{j=1}^t \eta_j\sum_{k=j-k_j+1}^{j-1} \eta_k  + \sum_{j=1}^t (G^2\eta^2_j+ G\eta_j / j)}{2} .
\end{align}

Now, plugging \eqref{eq:opt-7} and \eqref{eq:opt-6} back into \eqref{eq:opt-1}, we have
\begin{align*}
     \E_\A[ F_S(\bar{\bw}_t )-F_S(\bw^{*}) ]\le & \frac{\| \bw^{*} \|_2^2+4GD\sum_{j=1}^{K_P-1} \eta_j+  G^2\sum_{j=1}^T\big(4\eta_j\sum_{k=j-k_j+1}^{j} \eta_k  +   \eta^2_j\big)  }{2\sum_{j=1}^T  \eta_j} \\
     &+ \frac{  G\sum_{j=K_P}^{T}    \eta_j/  j }{2\sum_{j=1}^T  \eta_j}.
\end{align*}
%where $C=\|\bw_0 -\bw^{*} \|_2^2+4GD\sum_{j=1}^{K_P-1} \eta_j$.  

Furthermore,  choosing $\eta_j\equiv\eta=\frac{1}{\sqrt{T\log(T)}}$ and noting $D=\O(\sqrt{\eta T})=\O\big( (T/\log(T))^{1/4} \big)$, we have
\begin{align}\label{eq:proof-opt-rate}
     \E_\A[ F_S(\bar{\bw}_T )-F_S(\bw^{*}) ]&= \O\big( \frac{1 + K_P\eta^{\frac{3}{2}}\sqrt{T}   + T \eta^2 +\sum_{j=1}^T k_j \eta^2 }{T\eta}   \big)\\
     &= \O\Big( \frac{\sqrt{\log(T)}\big(1 + \sum_{j=1}^T k_j \eta^2\big)}{\sqrt{T}} + \frac{K_P \sqrt{\eta} }{\sqrt{T}}\Big).\nonumber
\end{align}
Recall that $k_j=\min\Big\{ \max\Big\{ \Big\lceil \frac{\log(2C_P D n j)}{\log(1/\lambda(P))}\Big\rceil, K_P  \Big\}, j \Big\}$. 
Let $K=\frac{1}{ 2C_PDn\lambda(P)^{K_P} }$. If $j\le K$,   we have $k_j\le K_P$ and 
\[ \sum_{j=1}^{K} k_j\eta^2 \le K K_P \eta^2=\frac{K_P}{ T^{\frac{5}{4}}\log^{\frac{3}{4}}(T)  2C_P n\lambda(P)^{K_P}}. \]
If $j > K$, there holds $k_j\le \Big\lceil \frac{\log(2C_P D n j)}{\log(1/\lambda(P))}\Big\rceil$. Then we have
%Since the finite terms can be bounded by a constant, below we consider $\sum_{j=K}^{T} k_j \eta^2$ for a large enough $K$ such that $k_j=\Big\lceil \frac{\log(2C_P G^2 n j)}{\log(1/\lambda(P))}\Big\rceil$ for all $j\ge K$. Note that
\begin{align*}%\label{eq:kj-sum}
    \sum_{j=K+1}^{T} k_j \eta^2 &\le  \frac{1}{\log( 1/ \lambda(P) )} \Big[ \sum_{j=K+1}^{T}  \log( 2C_P D  )   \eta^2 + \sum_{j=K+1}^{T}\log(n) \eta^2 + \sum_{j=K+1}^{T} \log(j) \eta^2 \Big] +T\eta^2 \nonumber\\
    &=\O\Big(  \frac{\log(n) + \log(T)  }{\log( 1/ \lambda(P) )\log(T)} + 1   \Big)=\O\Big( \frac{1}{\log( 1/ \lambda(P) )}  \Big).
\end{align*}  
Here, we use a reasonable assumption $n=\O(T)$. 
Combining the above two cases,  
%With the reasonable assumption $T\ge n$, 
we get
\begin{align}\label{eq:kj-bound}
    \sum_{j=1}^{T} k_j \eta^2 =\O\Big(  \frac{K_P}{  T^{\frac{5}{4}}\log^{\frac{3}{4}}(T)    C_P  n\lambda(P)^{K_P}}+  \frac{1}{\log( 1/ \lambda(P) )}  \Big).
\end{align}   
Putting \eqref{eq:kj-bound} back into \eqref{eq:proof-opt-rate} yields 
\[   \E_\A[ F_S(\bar{\bw}_T )-F_S(\bw^{*})]
=\O\Big( \frac{\sqrt{\log(T)}}{\sqrt{T}\log( 1/ \lambda(P)) }  + \frac{K_P}{T^{\frac{3}{4}}\log^{\frac{1}{4}}(T)  \min\{   \sqrt{T\log(T)} C_P  n\lambda(P)^{K_P} , 1 \} }\Big).\]
Note Assumption~\ref{ass:reversible} implies   $K_P=0$, then
\[   \E_\A[ F_S(\bar{\bw}_T )-F_S(\bw^{*})]
 =\O\Big( \frac{\sqrt{\log(T)}}{\sqrt{T}\log ( 1/\lambda(P)  )}   \Big),\]
which completes the proof. 
\end{proof}
\iffalse
\begin{remark}
The convergence rate of MC-SGD mainly depends on the number of iterations $T$,  eigenvalues of transition matrix $P$ and mixing time $K_P$. 
If $T$ is large enough such that $K_P/T \min\{   \sqrt{T\log(T)} C_P n\lambda(P)^{K_P} , 1 \} = \O(\sqrt{\log(T)/T})$, then we have $\E_\A[ F_S(\bar{\bw}_T )-F_S(\bw^{*})]
=\O\big( \frac{\sqrt{\log(T)}}{\sqrt{T}} \max\big\{1, \frac{1}{\log( 1/ \lambda(P)) } \big\}\big)$.

\end{remark}
\fi

To understand the variation of the algorithm,  we present a confidence-based bound for optimization error, which matches the bound in expectation up to a constant factor. %Recall that $\sup_{z\in \Z}f(0;z)\le M_0$. 
%We first introduce the following lemma on concentration inequality of martingales.  

\begin{theorem}[High-probability bound]\label{thm:opt-hp}
  Suppose  $f$ is convex and Assumptions \ref{ass:Markov-chain},  \ref{ass:G-lipschitz} and \ref{ass:reversible} hold. Let $\A$ be MC-SGD   with $T$ iterations and $\{\bw_j\}_{j=1}^T$ be produced by $\A$ with $\eta_j\equiv\eta=1/\sqrt{T\log(T)}$.  Assume $\sup_{z\in \Z}f(\bw;z)\le B$ for some $B>0$. Let $\gamma\in (0,1)$,  then with probability at least $1-\gamma$ 
\begin{align*}
 \!\!F_S(\bar{\bw}_T)\!-\!F_S(\bw^{*})\!=\!\O\!\Big( \!\frac{\sqrt{\log(T)}\!\big(  \frac{1}{\log(1/\lambda(P))}\!+\!B\!\sqrt{ {\log(\frac{1}{\gamma}} )} \big)}{\sqrt{T}}\! 
 \Big).
 \end{align*}
\end{theorem}
%The proofs for Theorems \ref{thm:opt-convex}  and \ref{thm:opt-hp} are given in Appendix  \ref{subsec:proof-opt-convex}. 

\begin{proof}[Proof of Theorem~\ref{thm:opt-hp}]\label{subsec:proof-opt-convexhp} 
We decompose the optimization error as follows
	\begin{align}\label{eq:opt-hp-1}
		  \sum_{j=1}^T \eta_j[F_S( \bar{\bw}_t)-F_S(\bw^{*})] & \le\sum_{j=1}^T \eta_j[F_S( {\bw}_j) - F_S(\bw_{j-k_j})]+  \sum_{j=1}^T  \eta_j[F_S(\bw_{j-k_j})-F_S(\bw^{*})] \nonumber \\ & \le G \sum_{j=1}^T \eta_j \| {\bw}_j - \bw_{j-k_j}\|_2  +  \sum_{j=1}^T \eta_j[F_S(\bw_{j-k_j})-F_S(\bw^{*})]  \nonumber \\ &\le G^2 \sum_{j=1}^T \eta_j \sum_{k=j-k_j+1}^j \eta_k + \sum_{j=1}^T \eta_j[F_S(\bw_{j-k_j})-F_S(\bw^{*})],
	\end{align}
	where the last second inequality used the Lipschitz continuity of $f$ and  the last inequality follows from the update rule \eqref{eq:MCSGD-update-rule}. 
	
Consider the second term in \eqref{eq:opt-hp-1}. Let 
$ 
\xi_j = \eta_j [ f(\bw_{j-k_j}; z_{i_j}) - f(\bw^{*};z_{i_j} ) ] .$
Observe that  $|\xi_j-\E_{i_j}[ \xi_j]|\le 2 B \eta_j$. %Note that convexity and Lipschitz continuity of $f$ imply $f(\bw;z)\le f(0;z) +\langle \partial f(\bw;z),\bw  \rangle\le  f(0;z) + \|\partial f(\bw;z)\|_2\| \bw \|_2 $ .  	
Then, applying Lemma~\ref{lem:martingle} implies, with probability at least $1-\gamma$, that
\begin{equation}\label{eq:opt-hp-2}
	\sum_{j=1}^T \E_{i_j}[\xi_j] - \sum_{j=1}^T \xi_j \le 2B \Big(2 \sum_{j=1}^T\eta_j^2 \log(1/{\gamma})\Big)^{1/2}. 
\end{equation}  
Note \eqref{eq:opt-9} implies
\begin{align*}
	&\sum_{j=1}^T \eta_j [F_S(\bw_{j-k_j})  - F_S(\bw^{*})] +  \sum_{j=1}^T \eta_j \sum_{i=1}^n  \Big( [P^{k_j}]_{i_{j-k_j},i} -  \frac{1}{n}\Big) [ f(\bw_{j-k_j};z_{i}) - f(\bw^{*} ; z_{i} ) ]\\
	&= \sum_{j=1}^T\E_{i_j}[ \eta_j  f(\bw_{j-k_j};z_{i_j}) - f(\bw^{*} ; z_{i_j} )| \bw_{0},\ldots,\bw_{j-k_j}, z_{i_1},\ldots, z_{i_{j-k_j}}  ],
	\end{align*}
which combines with	\eqref{eq:opt-hp-2} yields 
\begin{align*}
	&\sum_{j=1}^T \eta_j [F_S(\bw_{j-k_j})  - F_S(\bw^{*})] +  \sum_{j=1}^T \eta_j \sum_{i=1}^n  \Big(  [P^{k_j}]_{i_{j-k_j},i} - \frac{1}{n}  \Big) [ f(\bw_{j-k_j};z_{i}) - f(\bw^{*} ; z_{i} ) ] \\
	& \le   \sum_{j=1}^T \eta_j [f(\bw_{j-k_j};z_{i_j}) - f(\bw^{*} ; z_{i_j} )]  + 2B \Big(2 \sum_{j=1}^T\eta_j^2 \log(1/{\gamma})\Big)^{1/2}.
	\end{align*}
Putting \eqref{eq:opt-3} and \eqref{eq:opt-5} back into the above inequality, we obtain
\begin{align}\label{eq:opt-hp-3}	&\sum_{j=1}^T \eta_j [F_S(\bw_{j-k_j}) - F_S(\bw^{*})]\nonumber\\
&\le  \sum_{j=1}^T\eta_j    \sum_{i=1}^n \Big(  \frac{1}{n}-[P^{k_j}]_{i_{j-k_j},i} \Big) [ f(\bw_{j-k_j};z_{i})-f(\bw^{*} ; z_{i} ) ] \nonumber\\
&\quad  + \sum_{j=1}^T \eta_j [f(\bw_{j-k_j};z_{i_j})-f(\bw^{*} ; z_{i_j} )]  + 2B \Big(2 \sum_{j=1}^T\eta_j^2 \log(1/{\gamma})\Big)^{1/2} \nonumber\\
	& \le \frac{C+ G^2\sum_{j=1}^T\big(2\eta_j\sum_{k=j-k_j+1}^{j} \eta_k+\eta^2_j \big) }{2}  + \frac{  G\sum_{j=K_P}^T  \eta_j/j + 4B \Big(2 \sum_{j=1}^T\eta_j^2 \log(1/{\gamma})\Big)^{1/2}}{2} ,
	\end{align}	
where $C=\|  \bw^{*}\|_2^2 + 4GD\sum_{j=1}^{K_P-1}\eta_j$. 	

Now, plugging \eqref{eq:opt-hp-3} back into \eqref{eq:opt-hp-1}, with probability at least $1-\gamma$, there holds
 \begin{align*}
		\sum_{j=1}^T \eta_j[F_S( \bar{\bw}_t)-F_S(\bw^{*})]\le&  \frac{C+ G^2\sum_{j=1}^T\big(4\eta_j\sum_{k=j-k_j+1}^{j} \eta_k  +  \eta^2_j \big) }{2}\\
		&+ \frac{  G\sum_{j=K_P}^T  \eta_j/j + 4B \Big(2 \sum_{j=1}^T\eta_j^2 \log(1/{\gamma})\Big)^{1/2}}{2}.
		\end{align*}
By Jensen's inequality,  there holds
 \begin{align*}
	F_S( \bar{\bw}_t)-F_S(\bw^{*})\le&  \frac{C+ G^2\sum_{j=1}^T\big(4\eta_j\sum_{k=j-k_j+1}^{j} \eta_k  +  \eta^2_j \big) }{2	\sum_{j=1}^T \eta_j}\\
	&+ \frac{  G\sum_{j=K_P}^T  \eta_j/j + 4B \Big(2 \sum_{j=1}^T\eta_j^2 \log(1/{\gamma})\Big)^{1/2}}{2	\sum_{j=1}^T \eta_j}.
		\end{align*}
Similar as the discussion in Theorem~\ref{thm:opt-convex}, by choosing $\eta_j\equiv\eta=1/\sqrt{T\log(T)}$, there holds
\begin{align*}
 F_S(\bar{\bw}_T)-F_S(\bw^{*})= \O\Big(& \frac{\sqrt{\log(T)} \big(  \log^{-1}(1/\lambda(P)) +   B\sqrt{ {\log(1/\gamma)} } \big)}{\sqrt{T}}\\
 &+ \frac{K_P}{T^{\frac{3}{4}}\log^{\frac{1}{4}}(T)  \min\{   \sqrt{T\log(T)} C_P  n\lambda(P)^{K_P} , 1 \} }\Big)
 \Big).
 \end{align*}
If we further assume the Markov chain is reversible with $P=P^\top$, then we have
\begin{align*}
 F_S(\bar{\bw}_T)-F_S(\bw^{*})= \O\Big( \frac{\sqrt{\log(T)} \big(  \frac{1}{\log(1/\lambda(P))} +   B\sqrt{ {\log(\frac{1}{\gamma})} } \big) }{\sqrt{T}} 
 \Big).
 \end{align*}
\iffalse
\begin{align*}
 F_S(\bar{\bw}_T)-F_S(\bw^{*})= \O\Big( \frac{\sqrt{\log(T)} + B\sqrt{\log(1/\gamma)}}{\sqrt{T}}  
 \Big).
 \end{align*}
 \fi
 The proof is completed. 
\end{proof}

%\subsection{Optimization Error for Non-Convex Problems}\label{subsec:nonconvex}
The following theorem provides the convergence analysis for non-convex problems. Since the convergence in terms of objective values cannot be given, we only measure the convergence rate in terms of gradient norm. The proof follows from \cite{Sun2018Markov}.

\begin{theorem}[Non-convex case]\label{thm:opt-nonconvex}
Suppose Assumptions \ref{ass:Markov-chain}, \ref{ass:G-lipschitz} and \ref{ass:beta-smooth}  hold.  Let $\A$ be MC-SGD with $T$ iterations  and $\{\bw_j\}_{j=1}^T$ be produced by $\A$. Let $D$ be the diameter of $\W$, and
\[\!\!\!\!\!k_j\!=\!\min\!\Big\{\!\max\Big\{\Big\lceil \frac{\log(2C_PD n j)}{\log(1/\lambda(P))}\Big\rceil, K_P \Big\}, j\Big\}, j\in[T].\]
Then
\begin{align*}
     &\min_{1\le j\le T} \E_\A\big[\| \partial F_S(\bw_j)\|_2^2 ]\le \frac{ C+  \sum_{j=K_P}^T \eta_j/j }{2 \sum_{j=1}^T \eta_j} 
     + \frac{ G^2 L \sum_{j=1}^T (\eta^2_j+ k_j \sum_{k=j-k_j}^j \eta_k^2 + 6\eta_j \sum_{k=j-k_j}^j \eta_k) }{2 \sum_{j=1}^T \eta_j}, 
\end{align*}
where $C= 2(F_S(\bw_0)+ 2G^2 \sum_{j=1}^{K_P-1} \eta_j )$. Furthermore, suppose Assumption~\ref{ass:reversible} holds.   Selecting $\eta_j\equiv\eta=1/\big(\log(T)\sqrt{T}\big)$ implies $\min_{1\le j \le T}\E_\A\big[\| \partial F_S(\bw_j)\|_2^2 ]=\O\big(\log(T)/ (\sqrt{   T } \log^2(1/(\lambda(P)))\big)$. 
\end{theorem}
\begin{proof}[Proof of Theorem~\ref{thm:opt-nonconvex}]
 \label{subsec:proof-opt-nonconvex} 
 Let $k_j=\min\Big\{ \max\Big\{ \Big\lceil \frac{\log(2C_P D n j)}{\log(1/\lambda(P))}\Big\rceil, K_P  \Big\}, j \Big\}$. 
Consider the following decomposition
\begin{align}\label{eq:opt-nc-decom}
    \sum_{j=1}^T \eta_j\E_\A\big[\| \partial F_S(\bw_j)\|_2^2 ]= \sum_{j=1}^T \eta_j \E_\A\big[\| \partial F_S(\bw_j)\|_2^2 - \| \partial F_S(\bw_{j-k_j}) \|_2^2 \big] + \sum_{j=1}^T \eta_j \E_\A\big[ \| \partial F_S(\bw_{j-k_j}) \|_2^2 \big].
\end{align}
Note that  
\begin{align}\label{eq:opt-nc-1}
    &\sum_{j=1}^T \eta_j \E_\A\big[\| \partial F_S(\bw_j)\|_2^2 - \| \partial F_S(\bw_{j-k_j}) \|_2^2 \big]\nonumber\\
    %&=\sum_{j=1}^t \eta_j \E\big[ \langle \partial F_S(\bw_j) + \partial F_S(\bw_{j-k_j}), \partial F_S(\bw_j)-\partial F_S(\bw_{j-k_j}) \rangle \big]\nonumber\\
    &\le \sum_{j=1}^T \eta_j \E_\A\big[ \big(\|\partial F_S(\bw_j) \|_2+\| \partial F_S(\bw_{j-k_j})\|_2\big) (\|\partial F_S(\bw_j) \|_2- \|\partial F_S(\bw_{j-k_j})\|_2 \big) \big]  \nonumber\\
    &\le 2G  \sum_{j=1}^T \eta_j \E_\A[ \|\partial F_S(\bw_j) - \partial F_S(\bw_{j-k_j})\|_2 ] \le 2 G L \sum_{j=1}^T \eta_j \E_\A\big[ \| \bw_j -\bw_{j-k_{j}}\|_2 \big]\nonumber\\
    &\le 2G^2 L \sum_{j=1}^T \eta_j \sum_{k=j-k_j}^j \eta_k ,
\end{align}
where the second inequality used the fact that  $f$ is $G$-Lipschitz, the third inequality follows from the smoothness of $f$, and the last inequality used the update rule \eqref{eq:MCSGD-update-rule}. The first term in \eqref{eq:opt-nc-decom} is bounded.

Now, we turn to estimate the second term in \eqref{eq:opt-nc-decom}. Note that 
\begin{align}\label{eq:opt-nc-0}
    &\E_{i_j}\big[ \langle \partial f( \bw_{j-k_j};z_{i_j} ), \partial F_S(\bw_{j-{k_j}} )  \rangle | \bw_0, \ldots, \bw_{j-{k_j}}, i_1,\ldots, i_{j-k_j} \big]\nonumber\\
    &=\sum_{i=1}^n  \langle \partial f( \bw_{j-k_j};z_i ), \partial F_S(\bw_{j-{k_j}} )  \rangle \cdot \Pr( i_j=i | i_{j-k_j} )\nonumber\\
    &= \sum_{i=1}^n  \langle \partial f( \bw_{j-k_j};z_i ), \partial F_S(\bw_{j-{k_j}} )  \rangle \cdot [P^{k_j}]_{i_{j-k_j},i}\nonumber\\
    &= \frac{1}{n}\sum_{i=1}^n  \langle \partial f( \bw_{j-k_j};z_i ), \partial F_S(\bw_{j-{k_j}} )  \rangle  + \sum_{ i=1 }^n \Big( [P^{k_j}]_{i_{j-k_j},i}- \frac{1}{n} \Big) \langle \partial f( \bw_{j-k_j};z_i ), \partial F_S(\bw_{j-{k_j}} )  \rangle \nonumber\\
    &= \| \partial F_S( \bw_{j-k_j} )\|_2^2 +  \sum_{ i=1 }^n \Big( [P^{k_j}]_{i_{j-k_j},i}- \frac{1}{n} \Big) \langle \partial f( \bw_{j-k_j};z_i ), \partial F_S(\bw_{j-{k_j}} )  \rangle .
\end{align}
Taking total expectations on both sides and summing over $j$ yields
\begin{align}\label{eq:opt-nc-2}
    \sum_{j=1}^T \eta_j \E_\A\big[  \| \partial F_S( \bw_{j-k_j} )\|_2^2  \big] =&\sum_{j=1}^T \eta_j \E_\A\big[\langle \partial f( \bw_{j-k_j};z_{i_j} ), \partial F_S(\bw_{j-{k_j}} )  \rangle \big] \nonumber\\
    &+ \sum_{j=1}^T \eta_j \E_\A\Big[ \sum_{ i=1 }^n \Big(  \frac{1}{n} -  [P^{k_j}]_{i_{j-k_j},i} \Big)  \langle \partial f( \bw_{j-k_j};z_i ), \partial F_S(\bw_{j-{k_j}} )  \rangle\Big].
\end{align}
Consider the first term in \eqref{eq:opt-nc-2}. By the smoothness of $f$, we have
\begin{align*}
     F_S(\bw_j )  
    & \le F_S(\bw_{j-1}) + \langle \bw_j -\bw_{j-1}, \partial F_S(\bw_{j-1}) \rangle    + \frac{L}{2}\|\bw_j-\bw_{j-1}\|_2^2  \nonumber\\
    & \le F_S(\bw_{j-1}) + \langle \bw_j -\bw_{j-1}, \partial F_S(\bw_{j-k_j}) \rangle  + \langle \bw_j -\bw_{j-1}, \partial F_S(\bw_{j-1}) - \partial F_S(\bw_{j-k_j})  \rangle  \\
    &\quad + \frac{L}{2}\|\bw_j-\bw_{j-1}\|_2^2\nonumber\\
    &\le F_S(\bw_{j-1}) + \langle \bw_j -\bw_{j-1}, \partial F_S(\bw_{j-k_j}) \rangle  + \frac{1}{2}\|\bw_j -\bw_{j-1}\|^2_2 \\
    &\quad + \frac{1}{2}\| \partial F_S(\bw_{j-1}) - \partial F_S(\bw_{j-k_j})\|_2^2 + \frac{LG^2 \eta_j^2}{2}\nonumber\\
    &\le  F_S(\bw_{j-1}) + \langle \bw_j -\bw_{j-1}, \partial F_S(\bw_{j-k_j}) \rangle   + \frac{ (L +1) G^2 \eta^2_j }{2} + \frac{L^2 \|\bw_j-\bw_{j-k_j}\|_2^2}{2}\nonumber\\
    &\le F_S(\bw_{j-1}) + \langle \bw_j -\bw_{j-1}, \partial F_S(\bw_{j-k_j}) \rangle  + \frac{ (L + 1)  G^2 \eta^2_j }{2} + \frac{ G^2 L^2 k_j \sum_{k=j-k_j}^j \eta_k^2}{2},
\end{align*}
where the third inequality used $ab\le a^2/2 + b^2/2$, and  the last inequality follows from  $\|\bw_j-\bw_{j-k_j}\|_2^2=\|\sum_{k=j-k_j}^j \eta_k \partial f(\bw_{k-1};z_{i_{k-1}})\|_2^2\le k_j G^2 \sum_{k=j-k_j}^j \eta_k^2 $.
Rearrangement of the above inequality and taking expectation over $\A$ give us 
\begin{align}\label{eq:opt-nc-3}
     \E_\A [ \langle \bw_{j-1} -\bw_{j}, \partial F_S(\bw_{j-k_j}) \rangle ] \le & \E_\A [ F_S(\bw_{j-1}) - F_S(\bw_{j}) ] +   \frac{ (L +1) G^2 \eta^2_j }{2} \nonumber\\
     & + \frac{ G^2 L^2 k_j \sum_{k=j-k_j}^j \eta_k^2}{2}.
\end{align}
Note that
\begin{align*}
    &\E_\A \big[\langle \bw_{j-1} -\bw_{j}, \partial F_S(\bw_{j-k_j}) \rangle    \big]\\ &= \eta_j \E_\A \big[\langle \partial f(\bw_{j-1};z_{i_j}), \partial F_S(\bw_{j-k_j}) \rangle  \big]\\
    &= \eta_j \E_\A \big[\langle \partial f(\bw_{j-k_j};z_{i_j}), \partial F_S(\bw_{j-k_j}) \rangle   \big]  + \eta_j \E \big[\langle \partial f(\bw_{j-1};z_{i_j}) - \partial f(\bw_{j-k_j};z_{i_j}), \partial F_S(\bw_{j-k_j}) \rangle    \big]\nonumber\\
    &\ge  \eta_j \E_\A \big[\langle \partial f(\bw_{j-k_j};z_{i_j}), \partial F_S(\bw_{j-k_j}) \rangle   \big]\!-\!G^2 L\eta_j\!\sum_{k=j-k_j}^j\!\eta_k.
\end{align*}
Combining \eqref{eq:opt-nc-3} with the above inequality together, we obtain
\begin{align}\label{eq:opt-nc-8}
    &\eta_j \E_\A \big[\langle \partial f(\bw_{j-k_j};z_{i_j}), \partial F_S(\bw_{j-k_j}) \rangle   \big]\nonumber\\ 
    &\le  \E_\A [F_S(\bw_{j-1}) - F_S(\bw_{j}) ]  +    \frac{ G^2 L (\eta^2_j+ L k_j \sum_{k=j-k_j}^j \eta_k^2 + 2\eta_j \sum_{k=j-k_j}^j \eta_k) + G^2 \eta_j^2 }{2} .
\end{align}
Summing over $j$ yields
\begin{align}\label{eq:opt-nc-4}
     &\sum_{j=1}^T \eta_j \E_\A[ \langle \partial f(\bw_{j-k_j};z_{i_j}), \partial F_S(\bw_{j-k_j}) \rangle ] \nonumber\\
     &\le     F_S(\bw_0)+ \frac{ G^2 L \sum_{j=1}^T(\eta^2_j+ k_j \sum_{k=j-k_j}^j \eta_k^2 + 2\eta_j \sum_{k=j-k_j}^j \eta_k)  + G^2 \sum_{j=1}^T\eta_j^2 }{2}.
\end{align}

Now, we consider the second term in \eqref{eq:opt-nc-2}.  Similar as the proof of Theorem~\ref{thm:opt-convex}, it is easy to obtain the following bound by using Lemma~\ref{lem:difference-stationary}
\begin{align}\label{eq:opt-nc-5}
    & \sum_{j=1}^T \eta_j \E_\A\Big[ \sum_{ i=1 }^n  \Big(  \frac{1}{n} -  [P^{k_j}]_{i_{j-k_j},i} \Big) \langle \partial f( \bw_{j-k_j};z_{i_j} ), \partial F_S(\bw_{j-{k_j}} )  \rangle\Big]\nonumber\\ 
    &\le G^2 \sum_{j=1}^{K_P-1} \eta_j  \sum_{ i=1 }^n  \Big(  \frac{1}{n} +  [P^{k_j}]_{i_{j-k_j},i} \Big)   + G^2 \sum_{j=K_P}^{T} \eta_j  \sum_{ i=1 }^n  \Big|  \frac{1}{n} -  [P^{k_j}]_{i_{j-k_j},i} \Big|   \nonumber\\
    &\le 2G^2 \sum_{j=1}^{K_P-1} \eta_j + \sum_{j=K_P}^T \eta_j/2j.
\end{align}
Plugging \eqref{eq:opt-nc-4} and  \eqref{eq:opt-nc-5} back into \eqref{eq:opt-nc-2}, we have
\begin{align}
    \label{eq:opt-nc-6}
       &\sum_{j=1}^T \eta_j \E_\A\big[  \| \partial F_S( \bw_{j-k_j} )\|_2^2  \big]\nonumber \\
       &\le \frac{2(F_S(\bw_0)+ 2G^2 \sum_{j=1}^{K_P-1} \eta_j ) + \sum_{j=K_P}^T \eta_j/j    +  G^2\sum_{j=1}^T  \eta_j^2     }{2} \nonumber\\ &\quad + \frac{G^2 L  \sum_{j=1}^T (\eta^2_j+ L k_j \sum_{k=j-k_j}^j \eta_k^2 + 2\eta_j \sum_{k=j-k_j}^j \eta_k)}{2} .
\end{align}
Finally, putting \eqref{eq:opt-nc-1} and \eqref{eq:opt-nc-6} back into \eqref{eq:opt-nc-decom}, we obtain
\begin{align*}
   \sum_{j=1}^T \eta_j\min_{1\le j\le T}\E_\A\big[\| \partial F_S(\bw_j)\|_2^2 ] 
  & \le 
     \sum_{j=1}^T \eta_j\E_\A\big[ \| \partial F_S(\bw_j)\|_2^2 ] \\
     &\le \frac{  2(F_S(\bw_0)+ 2G^2 \sum_{j=1}^{K_P-1} \eta_j ) +    \sum_{j=K_P}^T \eta_j/j +  G^2\sum_{j=1}^T  \eta_j^2 }{2} \\ &\quad +  \frac{   G^2 L \sum_{j=1}^T (\eta^2_j+ L k_j \sum_{k=j-k_j}^j \eta_k^2 + 6\eta_j \sum_{k=j-k_j}^j \eta_k)    }{2}.
\end{align*}
Dividing both sides of the above inequality by $\sum_{j=1}^T \eta_j$ yields
\begin{align*}
   \min_{1\le j\le T} \E_\A\big[\| \partial F_S(\bw_j)\|_2^2 ]\le & \frac{ 2(F_S(\bw_0)+ 2G^2 \sum_{j=1}^{K_P-1} \eta_j ) +     \sum_{j=K_P}^T \eta_j/j + \sum_{j=1}^T  G^2 \eta_j^2  }{2 \sum_{j=1}^T \eta_j}  \\
   &+ \frac{   G^2 L \sum_{j=1}^T (\eta^2_j+ L k_j \sum_{k=j-k_j}^j \eta_k^2 + 6\eta_j \sum_{k=j-k_j}^j \eta_k)   }{2 \sum_{j=1}^T \eta_j}. 
\end{align*}
If we set $\eta_j\equiv\eta=\frac{1}{\sqrt{T} \log(T) }$,  then 
\[ \min_{1\le j\le T} \E_\A\big[\| \partial F_S(\bw_j)\|_2^2 ]\!=\!\O\Big(\frac{K_P}{T}+ \frac{1}{T\eta} + \eta + \frac{ \sum_{j=1}^T k_j^2 \eta^2 }{T\eta}   \Big)\!=\!\O\Big(\frac{K_P}{T}+ \frac{\log(T)\big(1 +    \sum_{j=1}^T k_j^2 \eta^2 \big)  }{\sqrt{T}  }   \Big).  \]
It suffices to estimate $\sum_{j=1}^{T} k_j^2 \eta^2$. Recall that $k_j=\min\Big\{ \max\Big\{ \Big\lceil \frac{\log(2C_P D n j)}{\log(1/\lambda(P))}\Big\rceil, K_P  \Big\}, j \Big\}$. 
Let $K=\frac{1}{ 2C_PDn\lambda(P)^{K_P} }$. If $j\le K$,   we have $k_j\le K_P$ and 
\[ \sum_{j=1}^{K} k^2_j\eta^2 \le K K^2_P \eta^2=\frac{K^2_P}{ T\log(T)  2C_PDn\lambda(P)^{K_P}}. \]
If $j > K$, there holds $k_j\le \Big\lceil \frac{\log(2C_P D n j)}{\log(1/\lambda(P))}\Big\rceil$. Then with a reasonable assumption $n=\O(T)$ we have

%Since the finite terms can be controlled by a constant, below we consider $\sum_{j=K}^{T} k_j^2 \eta^2$ for a large enough $K$ such that $k_j=\Big\lceil \frac{\log(2C_P G^2 n j)}{\log(1/\lambda(P))}\Big\rceil$ for all $j\ge K$. Note that
\begin{align*}
    \sum_{j=K+1}^{T} k_j^2 \eta^2 &\le  \frac{6}{\log^2( 1/ \lambda(P) )} \Big[ \sum_{j=K+1}^{T} \big(\log( 2C_P D  )   \big)^2  \eta^2 + \sum_{j=K+1}^{T} \log^2(n) \eta^2 + \sum_{j=K+1}^{T} \log^2(j) \eta^2 \Big]\\
    &\quad +2T\eta^2 \\
    &=\O\Big(  \frac{1}{\log^2(1/\lambda(P))}  \Big). 
\end{align*} 
%It is obvious that the first term can be bounded by a constant since $(\eta_j^2)_{j\ge 1} $ is summable.  Under a reasonable assumption $T\ge n$, there holds  $\sum_{j=K}^{T} \log^2(n) \eta^2 \le \sum_{j=K}^{T} \frac{1}{T}\le 1. $  Also, there holds  $\sum_{j=K}^{T} \log^2(j) \eta^2\le \sum_{j=K}^{T} \frac{1}{T }\le 1$. Hence, we have
Combining the above two cases together yields
\[ \sum_{j=1}^{T} k_j^2 \eta^2 =\O\Big( \frac{K^2_P}{ T\log(T) C_P n\lambda(P)^{K_P}} + \frac{1}{\log^2(1/\lambda(P))} \Big). \]
Therefore, 
\[  \min_{1\le j\le T} \E_\A\big[\| \partial F_S( {\bw}_j)\|_2^2 ]=\O\Big(\frac{K_P}{T}+ \frac{\log(T)  }{\sqrt{T}   }  \big( \frac{K^2_P}{ T\log(T)  C_P n\lambda(P)^{K_P}} + \frac{1}{\log^2(1/\lambda(P))} \big) \Big).\]
The stated bound then follows from $K_P=0$. 
\end{proof}

\subsection{Proofs of   Theorem~\ref{thm:excess-smooth} and Theorem~\ref{thm:excess-nonsmooth}}\label{proof:excess}

\begin{proof}[Proof of Theorem~\ref{thm:excess-smooth}]
Let $\eta_j\equiv\eta$. 
According to Part (a) in Theorem~\ref{thm:generalization-error} and   \eqref{eq:proof-opt-rate},  we know
\begin{align*}
      \E_{S,\A}[ F(\bar{\bw}_T) -  F_S(\bw^{*}) ]&=
    \E_{S,\A}[ F(\bar{\bw}_T) - F_S(\bar{\bw}_T) ] + \E_\A[ F_S(\bar{\bw}_T) -  F_S(\bw^{*}) ]\\
    &=\O\Big( \frac{T\eta}{n} + \frac{1+\big(   T  +  \sum_{j=1}^Tk_j \big)\eta^2 }{T\eta}   + \frac{K_P\sqrt{\eta}}{\sqrt{T}} \Big).
\end{align*}
Setting $\eta=\frac{1}{\sqrt{T\log(T)}}$ and choosing $T \asymp n$,  we have
\begin{align*}
   & \E_{S,\A}[ F(\bar{\bw}_T) -  F_{S,\A}(\bw^{*}) ]\\
&=\O\Big( \frac{\sqrt{T}}{n}+ \frac{\sqrt{\log(T)}}{\sqrt{T}\log(1/\lambda(P))}   + \frac{K_P}{T^{\frac{3}{4}}\log^{\frac{1}{4}}(T) \min\{ \sqrt{T\log(T)} C_P n \lambda(P)^{K_P} \}}\Big)\\
&= \O\Big(  \frac{\sqrt{\log(n)}}{\sqrt{n}\log(1/\lambda(P))}  + \frac{K_P}{n^{\frac{3}{4}}\log^{\frac{1}{4}}(n) \min\{ \sqrt{n\log(n)} C_P n \lambda(P)^{K_P},1 \} }  \Big),
\end{align*} 
where the first equality follows from Eq.\eqref{eq:kj-bound}. 
Note that $K_P=0$ when $P=P^\top$, we immediately obtain
\[ \E_{S,\A}[ F(\bar{\bw}_T) -  F_{S,\A}(\bw^{*}) ]= \O\Big(  \frac{\sqrt{\log(n)}}{\sqrt{n}\log(1/\lambda(P))}   \Big).\]
\end{proof}

\begin{proof}[Proof of Theorem~\ref{thm:excess-nonsmooth} ]

Part (b) in Theorem~~\ref{thm:generalization-error} and \eqref{eq:proof-opt-rate}
implies
\begin{align}\label{eq:excess-nonsm-1}
     \E_{S,\A}[ F(\bar{\bw}_T) -  F_S(\bw^{*}) ]&=
    \E_{S,\A}[ F(\bar{\bw}_T) - F_S(\bar{\bw}_T) ] + \E_\A[ F_S(\bar{\bw}_T) -  F_S(\bw^{*}) ]\nonumber\\
    %&=\O\big( \sqrt{T}\eta + \frac{T\eta}{n}+    \frac{1+\big(   T  +  \sum_{j=1}^Tk_j \big)\eta^2 }{T\eta}  + \frac{K_P}{T}   \big)\nonumber\\
    &=\O\big( \sqrt{T}\eta + \frac{T\eta}{n}+    \frac{1+  (T+\sum_{j=1}^Tk_j)  \eta^2 }{T\eta}   + \frac{K_P\sqrt{\eta}}{\sqrt{T}}    \big). 
\end{align}
Selecting $\eta= T^{-\frac{3}{4}}$. Similar as the discussion in Theorem~\ref{thm:opt-convex},  we know
\begin{align*}
    \sum_{j=1}^T k_j \eta^2 &=\sum_{j=1}^K k_j\eta^2 + \sum_{j=K+1}^T k_j\eta^2\\
    &\le K K_P \eta^2 + \frac{1}{\log(1/\lambda(P))}\big( \sum_{j=K+1}^T \log(2C_PD)\eta^2 + \sum_{j=K+1}^T \log(n) \eta^2 + \sum_{j=K+1}^T \log(j) \eta^2 \big)\\
    &\quad +  T\eta^2\\
    &=\O\Big( \frac{\log(T)}{\sqrt{T}\log(1/\lambda(P))}+ \frac{K_P}{T^{13/8} C_P n \lambda(P)^{K_P}}  \Big),
\end{align*}
where $K=\frac{1}{ 2C_PDn\lambda(P)^{K_P}}$ and $D=\O(\sqrt{\eta T})$.
Note transition matrix $P$ is symmetric implies $K_P=0$. 
Plugging the above estimation with $K_P=0$ back into \eqref{eq:excess-nonsm-1} and choosing $T\asymp n^2$,  we get 
\[ \E_{S,\A} [ F(\bar{\bw}_T) ] - F(\bw^{*}) = \O\Big( \frac{1}{\sqrt{n}\log(1/\lambda(P))} \Big).   \] 
The above results  
The proof is completed. 
\end{proof}

\section{Proofs of Markov Chain SGDA\label{sec:proof-mc-sgda}}
In this section, we present the proof on MC-SGDA. Let $(\bw^{*},\bv^{*})$ be a saddle point of $F$, i.e., for any $\bw\in \W, \bv\in \V$, there holds $F(\bw^{*},\bv)\le F(\bw^*,\bv^*)\le F(\bw,\bv^{*})$. 
\subsection{Proofs of   Theorem \ref{thm:stab-gen-min-max}-Theorem~\ref{thm:gen-sgda-primal}}\label{proof:sgda-stab}
We first prove Theorem \ref{thm:stab-gen-min-max} on the connection between stability and generalization for minimax problems.
\begin{proof}[Proof of Theorem \ref{thm:stab-gen-min-max}]
We follow the argument in \cite{lei2021stability} to prove Theorem \ref{thm:stab-gen-min-max}. %claim no contribution in proving this theorem.
For any function $g,\tilde{g}$, we have the basic inequalities
\begin{equation}\label{sup-inf}
  \begin{split}
      & \sup_{\bw}g(\bw)-\sup_{\bw}\tilde{g}(\bw)\leq \sup_{\bw}\big(g(\bw)-\tilde{g}(\bw)\big)\\
      & \inf_{\bw}g(\bw)-\inf_{\bw}\tilde{g}(\bw)\leq \sup_{\bw}\big(g(\bw)-\tilde{g}(\bw)\big).
  \end{split}
\end{equation}
According to Eq. \eqref{sup-inf}, we know
\begin{multline*}
  \triangle^w(\A_{\bw}(S),\A_{\bv}(S))-\triangle^w_{emp}(\A_{\bw}(S),\A_{\bv}(S)) \leq \sup_{\bv'\in\vcal}\ebb[F(\A_{\bw}(S),\bv')-F_S(\A_{\bw}(S),\bv')]\\+
  \sup_{\bw'\in\wcal}\ebb[F_S(\bw',\A_{\bv}(S))-F(\bw',\A_{\bv}(S))].
\end{multline*}
Recall that $S=\{z_1,\ldots,z_n\}$, $\tilde{S}=\{\tilde{z}_1, \ldots,\tilde{z}_n  \}$ and $S^{(i)}=\{z_1,\ldots,z_{i-1},\tilde{z}_i,z_{i+1},\ldots,z_n\}$. 
According to the symmetry between $z_i$ and $\tilde{z}_i$ we know
\begin{align*}
  \ebb[F(\A_{\bw}(S),\bv')-F_S(\A_{\bw}(S),\bv')] & =  \frac{1}{n}\sum_{i=1}^{n} \ebb[F(\A_{\bw}(S^{(i)}),\bv')]-\ebb[F_S(\A_{\bw}(S),\bv')]\\
  &=  \frac{1}{n}\sum_{i=1}^{n} \ebb\big[f(\A_{\bw}(S^{(i)}),\bv';z_i)-f(\A_{\bw}(S),\bv';z_i)\big]\\
  & \leq \frac{G}{n}\sum_{i=1}^{n}\ebb\big[\|\A_{\bw}(S^{(i)})-\A_{\bw}(S)\|_2\big],
\end{align*}
where the second identity holds since $z_i$ is not used to train $\A_{\bw}(S^{(i)})$ and the last inequality holds due to the Lipschitz continuity of $f$.
In a similar way, we can prove
\[
\ebb[F_S(\bw',\A_{\bv}(S))-F(\bw',\A_{\bv}(S))]
\leq
\frac{G}{n}\sum_{i=1}^{n}\ebb\big[\|\A_{\bv}(S^{(i)})-\A_{\bv}(S)\|_2\big].
\]
As a combination of the above three inequalities we get
\[
    \triangle^w(\A_{\bw}(S),\A_{\bv}(S))-\triangle^w_S(\A_{\bw}(S),\A_{\bv}(S))\! \leq\!\frac{G}{n}\sum_{i=1}^{n}\ebb\Big[\|\A_{\bw}(S^{(i)})-\A_{\bw}(S)\|_2+\|\A_{\bv}(S^{(i)})-\A_{\bv}(S)\|_2\Big].
\]
This proves Part (a). Part (b) was proved in \cite{lei2021stability}.
The proof is completed.
\end{proof}
\iffalse
\begin{comment}
\begin{lemma}[\cite{lei2021stability}]\label{lem:stab-gen}
Let $\A$ be a randomized algorithm and $\epsilon>0$. If for all neighboring datasets $S, S'$, there holds
\begin{align*}
\EX_\A[\|\A_\bw(S) - \A_\bw(S')\|_2] \leq \varepsilon.    
\end{align*}
Furthermore, if the function $F(\bw,\cdot)$ is $\rho$-strongly-concave and Assumptions \ref{ass:lipschitz}, \ref{ass:smooth}  hold, then the primal generalization error satisfies
  \begin{align*}
      \EX_{S,\A}\Big[R(
      A_{\bw}(S))-R_S(\A_{\bw}(S))\Big]\leq \big(1+L/\rho\big)G\varepsilon.
  \end{align*}%\begin{equation}\label{stab-gen-b}
%\end{equation}
\end{lemma}
We now study the stability of MC-SGDA, which relies on the properties of the gradient map
\[
G_{f,\eta}:\begin{pmatrix}
  \bw \\
  \bv
\end{pmatrix}\mapsto \begin{pmatrix}
           \bw-\eta\partial_{\bw}f(\bw,\bv) \\
           \bv+\eta\partial_{\bv}f(\bw,\bv)
         \end{pmatrix}. 
\] 
\end{comment}
\fi

To prove our stability bounds, we first introduce two useful lemmas. The first lemma is due to \cite{rockafellar1976monotone}, while the second lemma is elementary.
\begin{lemma}[\cite{rockafellar1976monotone}\label{lem:nonexpansive}]
  Let $f$ be $\rho$-SC-SC with $\rho\geq0$. For any $(\bw,\bv)$ and $(\bw',\bv')$, then   \begin{equation}\label{non-expansive-key}
  \bigg\langle\begin{pmatrix}
    \bw-\bw' \\
    \bv-\bv'
  \end{pmatrix},\begin{pmatrix}
                  \partial_{\bw}f(\bw,\bv)-\partial_{\bw}f(\bw',\bv') \\
                  \partial_{\bv}f(\bw',\bv')-\partial_{\bv}f(\bw,\bv)
                \end{pmatrix}\bigg\rangle \geq \rho\bigg\|\begin{pmatrix}
                                                           \bw-\bw' \\
                                                           \bv-\bv'
                                                         \end{pmatrix}\bigg\|_2^2.
  \end{equation}
\end{lemma}
\begin{lemma}\label{lem:quad}
Let $b,c\geq0$. If $x^2\leq bx+c$, then $x\leq b+\sqrt{c}$.
\end{lemma}
%Now we are ready to prove Theorem \ref{thm:stab-sgda}.
\begin{proof}[Proof of Theorem \ref{thm:stab-sgda}]
%For a set $S=\{z_1,\ldots,z_n\}$, we denote $S(j)$ to denote $z_j$. Then, we have
%\[
%\begin{cases}
%  \bw_t^{(i)} = \bw^{(i)}_{t-1}-\eta_t\nabla_{\bw}f(\bw^{(i)}_{t-1},\bv^{(i)}_{t-1};S^{(i)}(i_t))&  \\
%  \bv_t^{(i)} = \bv^{(i)}_{t-1}+\eta_t\nabla_{\bv}f(\bw^{(i)}_{t-1},\bv^{(i)}_{t-1};S^{(i)}(i_t))&
%\end{cases}
%\]
For any $i\in[n]$, define
$S^{(i)}=\{z_1,\ldots,z_{i-1},\tilde{z}_i,z_{i+1},\ldots,z_n\}$ as the set formed from $S$ by replacing the $i$-th element with $\tilde{z}_i$. Let $(\bw_t^{(i)},\bv_t^{(i)})$ be produced by MC-SGDA based on $S^{(i)}$ for $i\in[n]$.
Note that the projection step is nonexpansive.

We first prove Part (a).
We consider two cases at the $t$-th iteration. 

\noindent{\textbf{Case 1}.}
If $i_t\neq i$, then it follows from the $L$-smoothness of $f$ and  Lemma \ref{lem:nonexpansive} with $\rho=0$  that
\begin{align}
  &\left\|\begin{pmatrix}
            \bw_{t}-\bw_{t}^{(i)} \\
           \bv_{t}-\bv_{t}^{(i)}
         \end{pmatrix}\right\|_2^2\nonumber \\
         &\leq\left\|\begin{pmatrix}
                   \bw_{t-1}-\eta_t\partial_{\bw}f(\bw_{t-1},\bv_{t-1};z_{i_t})-\bw_{t-1}^{(i)}+\eta_t\partial_{\bw}f(\bw_{t-1}^{(i)},\bv^{(i)}_{t-1};z_{i_t}) \\
                   \bv_{t-1}+\eta_t\partial_{\bv}f(\bw_{t-1},\bv_{t-1};z_{i_t})-\bv_{t-1}^{(i)}-\eta_t\partial_{\bv}f(\bw_{t-1}^{(i)},\bv^{(i)}_{t-1};z_{i_t})
                 \end{pmatrix}\right\|_2^2\notag \\
         & = \left\|\begin{pmatrix}
           \bw_{t-1}-\bw_{t-1}^{(i)} \\
           \bv_{t-1}-\bv_{t-1}^{(i)}
         \end{pmatrix}\right\|_2^2+\eta_t^2\left\|\begin{pmatrix}
                                                    \partial_{\bw}f(\bw_{t-1},\bv_{t-1};z_{i_t}) - \partial_{\bw}f(\bw_{t-1}^{(i)},\bv^{(i)}_{t-1};z_{i_t}) \\
                                                    \partial_{\bv}f(\bw_{t-1},\bv_{t-1};z_{i_t}) - \partial_{\bv}f(\bw_{t-1}^{(i)},\bv^{(i)}_{t-1};z_{i_t})
                                                  \end{pmatrix}\right\|_2^2\notag\\
         & \quad  -2\eta_t\left\langle\begin{pmatrix}
           \bw_{t-1}-\bw_{t-1}^{(i)} \\
           \bv_{t-1}-\bv_{t-1}^{(i)}
         \end{pmatrix},\begin{pmatrix}
                                                    \partial_{\bw}f(\bw_{t-1},\bv_{t-1};z_{i_t}) - \partial_{\bw}f(\bw_{t-1}^{(i)},\bv^{(i)}_{t-1};z_{i_t}) \\
                                                    \partial_{\bv}f(\bw_{t-1}^{(i)},\bv^{(i)}_{t-1};z_{i_t}) - \partial_{\bv}f(\bw_{t-1},\bv_{t-1};z_{i_t})
                                                  \end{pmatrix}\right\rangle\notag\\
         & \leq (1+L^2\eta_t^2)\left\|\begin{pmatrix}
           \bw_{t-1}-\bw_{t-1}^{(i)} \\
           \bv_{t-1}-\bv_{t-1}^{(i)}
         \end{pmatrix}\right\|_2^2.\label{stab-gda-1}
\end{align}

\noindent{\textbf{Case 2}.}
If $i_t=i$, then it follows from the Lipschitz continuity of $f$ that
\begin{align}
  &\left\|\begin{pmatrix}
           \bw_{t}-\bw_{t}^{(i)} \\
           \bv_{t}-\bv_{t}^{(i)}
         \end{pmatrix}\right\|_2^2\nonumber\\
         &\leq\left\|\begin{pmatrix}
                   \bw_{t-1}-\eta_t\nabla_{\bw}f(\bw_{t-1},\bv_{t-1};z_i)-\bw_{t-1}^{(i)}+\eta_t\nabla_{\bw}f(\bw_{t-1}^{(i)},\bv^{(i)}_{t-1};\tilde{z}_i) \\
                   \bv_{t-1}+\eta_t\nabla_{\bv}f(\bw_{t-1},\bv_{t-1};z_i)-\bv_{t-1}^{(i)}-\eta_t\nabla_{\bv}f(\bw_{t-1}^{(i)},\bv^{(i)}_{t-1};\tilde{z}_i)
                 \end{pmatrix}\right\|_2^2\notag \\
                          & = \left\|\begin{pmatrix}
           \bw_{t-1}-\bw_{t-1}^{(i)} \\
           \bv_{t-1}-\bv_{t-1}^{(i)}
         \end{pmatrix}\right\|_2^2+\eta_t^2\left\|\begin{pmatrix}
                                                    \partial_{\bw}f(\bw_{t-1},\bv_{t-1};z_i) - \partial_{\bw}f(\bw_{t-1}^{(i)},\bv^{(i)}_{t-1};\tilde{z}_i) \\
                                                    \partial_{\bv}f(\bw_{t-1},\bv_{t-1};z_i) - \partial_{\bv}f(\bw_{t-1}^{(i)},\bv^{(i)}_{t-1};\tilde{z}_i)
                                                  \end{pmatrix}\right\|_2^2\notag\\
         & \quad  +2\eta_t\left\|\begin{pmatrix}
           \bw_{t-1}-\bw_{t-1}^{(i)} \\
           \bv_{t-1}-\bv_{t-1}^{(i)}
         \end{pmatrix}\right\|_2\left\|\begin{pmatrix}
                                                    \partial_{\bw}f(\bw_{t-1},\bv_{t-1};z_i) - \partial_{\bw}f(\bw_{t-1}^{(i)},\bv^{(i)}_{t-1};\tilde{z}_i) \\
                                                    \partial_{\bv}f(\bw_{t-1}^{(i)},\bv^{(i)}_{t-1};\tilde{z}_i) - \partial_{\bv}f(\bw_{t-1},\bv_{t-1};z_i)
                                                  \end{pmatrix}\right\|_2\notag\\
         & \leq \left\|\begin{pmatrix}
           \bw_{t-1}-\bw_{t-1}^{(i)} \\
           \bv_{t-1}-\bv_{t-1}^{(i)}
         \end{pmatrix}\right\|_2^2+8\eta_t^2G^2+4\sqrt{2}G\eta_t\left\|\begin{pmatrix}
           \bw_{t-1}-\bw_{t-1}^{(i)} \\
           \bv_{t-1}-\bv_{t-1}^{(i)}
         \end{pmatrix}\right\|_2\label{stab-gda-2}
\end{align}
We can combine the above two inequalities together and get the following inequality
\begin{align*}
    \left\|\begin{pmatrix}
           \bw_{t}-\bw_{t}^{(i)} \\
           \bv_{t}-\bv_{t}^{(i)}
         \end{pmatrix}\right\|_2^2\leq & (1+L^2\eta_t^2)\left\|\begin{pmatrix}
           \bw_{t-1}-\bw_{t-1}^{(i)} \\
           \bv_{t-1}-\bv_{t-1}^{(i)}
         \end{pmatrix}\right\|_2^2+8\eta_t^2G^2\ibb_{[i_t=i]} \\
         & + 4\sqrt{2}G\eta_t\left\|\begin{pmatrix}
           \bw_{t-1}-\bw_{t-1}^{(i)} \\
           \bv_{t-1}-\bv_{t-1}^{(i)}
         \end{pmatrix}\right\|_2\ibb_{[i_t=i]}.
\end{align*}
We can apply the above inequality recursively and derive
\begin{align*}
    \left\|\begin{pmatrix}
           \bw_{t}-\bw_{t}^{(i)} \\
           \bv_{t}-\bv_{t}^{(i)}
         \end{pmatrix}\right\|_2^2\leq& L^2\sum_{j=1}^{t-1}\eta_j^2\left\|\begin{pmatrix}
           \bw_{j}-\bw_{j}^{(i)} \\
           \bv_{j}-\bv_{j}^{(i)}
         \end{pmatrix}\right\|_2^2+8G^2\sum_{j=1}^{t}\eta_j^2\ibb_{[i_j=i]}\\
         &+4\sqrt{2}G\sum_{j=1}^{t}\eta_j\left\|\begin{pmatrix}
           \bw_{j-1}-\bw_{j-1}^{(i)} \\
           \bv_{j-1}-\bv_{j-1}^{(i)}
         \end{pmatrix}\right\|_2\ibb_{[i_j=i]}
\end{align*}
For simplicity, we let
\begin{equation}\label{delta-t-i}
\delta_t^{(i)}=\max_{j\in[t]}\left\|\begin{pmatrix}
           \bw_{j}-\bw_{j}^{(i)} \\
           \bv_{j}-\bv_{j}^{(i)}
         \end{pmatrix}\right\|_2.
\end{equation}
Then we have
\begin{align*}
\big(\delta_t^{(i)}\big)^2 & \leq L^2\big(\delta_t^{(i)}\big)^2\sum_{j=1}^{t-1}\eta_j^2+8G^2\sum_{j=1}^{t}\eta_j^2\ibb_{[i_j=i]}+4\sqrt{2}G\delta_t^{(i)}\sum_{j=1}^{t}\eta_j\ibb_{[i_j=i]}\\
& \leq \frac{\big(\delta_t^{(i)}\big)^2}{2}+8G^2\sum_{j=1}^{t}\eta_j^2\ibb_{[i_j=i]}+4\sqrt{2}G\delta_t^{(i)}\sum_{j=1}^{t}\eta_j\ibb_{[i_j=i]},
\end{align*}
where we have used $\sum_{j=1}^{t}\eta_j^2\leq 1/(2L^2)$. It then follows that
\[
\big(\delta_t^{(i)}\big)^2\leq 16G^2\sum_{j=1}^{t}\eta_j^2\ibb_{[i_j=i]}+8\sqrt{2}G\delta_t^{(i)}\sum_{j=1}^{t}\eta_j\ibb_{[i_j=i]}.
\]
We can apply Lemma \ref{lem:quad} with $x=\delta_t^{(i)}$ to show that
\[
\left\|\begin{pmatrix}
           \bw_{t}-\bw_{t}^{(i)} \\
           \bv_{t}-\bv_{t}^{(i)}
         \end{pmatrix}\right\|_2\leq\delta_t^{(i)}\leq 4G\Big(\sum_{j=1}^{t}\eta_j^2\ibb_{[i_j=i]}\Big)^{\frac{1}{2}}+8\sqrt{2}G\sum_{j=1}^{t}\eta_j\ibb_{[i_j=i]}.
\]
It then follows from the concavity of the function $x\mapsto \sqrt{x}$ that
\begin{align*}
  \frac{1}{n}\sum_{i=1}^{n} \left\|\begin{pmatrix}
           \bw_{t}-\bw_{t}^{(i)} \\
           \bv_{t}-\bv_{t}^{(i)}
         \end{pmatrix}\right\|_2 & \leq \frac{4G}{n}\sum_{i=1}^{n}\Big(\sum_{j=1}^{t}\eta_j^2\ibb_{[i_j=i]}\Big)^{\frac{1}{2}}+\frac{8\sqrt{2}G}{n}\sum_{i=1}^{n}\sum_{j=1}^{t}\eta_j\ibb_{[i_j=i]} \\
   & \leq 4G\Big(\frac{1}{n}\sum_{i=1}^{n}\sum_{j=1}^{t}\eta_j^2\ibb_{[i_j=i]}\Big)^{\frac{1}{2}}+\frac{8\sqrt{2}G}{n}\sum_{i=1}^{n}\sum_{j=1}^{t}\eta_j\ibb_{[i_j=i]}\\
   & = 4G\Big(\frac{1}{n}\sum_{j=1}^{t}\eta_j^2\Big)^{\frac{1}{2}}+\frac{8\sqrt{2}G}{n}\sum_{j=1}^{t}\eta_j,
\end{align*}
where we have used the identity $\sum_{i=1}^{n}\ibb_{[i_j=i]}=1$. 
Finally, the convexity of the norm implies
\begin{align*}
  \frac{1}{n}\sum_{i=1}^{n} \left\|\begin{pmatrix}
           \bar{\bw}_{T}- \bar{\bw}_{T}^{(i)} \\
            \bar{\bv}_{T}- \bar{\bv}_{T}^{(i)}
         \end{pmatrix}\right\|_2 & \leq \frac{4G}{n}\sum_{i=1}^{n}\Big(\sum_{j=1}^{t}\eta_j^2\ibb_{[i_j=i]}\Big)^{\frac{1}{2}}+\frac{8\sqrt{2}G}{n}\sum_{i=1}^{n}\sum_{j=1}^{t}\eta_j\ibb_{[i_j=i]} \\
   & \leq 4G\Big(\frac{1}{n}\sum_{i=1}^{n}\sum_{j=1}^{t}\eta_j^2\ibb_{[i_j=i]}\Big)^{\frac{1}{2}}+\frac{8\sqrt{2}G}{n}\sum_{i=1}^{n}\sum_{j=1}^{t}\eta_j\ibb_{[i_j=i]}\\
   & = 4G\Big(\frac{1}{n}\sum_{j=1}^{t}\eta_j^2\Big)^{\frac{1}{2}}+\frac{8\sqrt{2}G}{n}\sum_{j=1}^{t}\eta_j,
\end{align*}

The proof of part (a) is completed.

\medskip

We now move to the nonsmooth case. In a similar way, we  consider the following two cases. 

\noindent{\textbf{Case 1}.}
If $i_t\neq i$, analogous to Eq. \eqref{stab-gda-1}, we can use the Lipschitz continuity of $f$ to derive
\[
\left\|\begin{pmatrix}
           \bw_{t}-\bw_{t}^{(i)} \\
           \bv_{t}-\bv_{t}^{(i)}
         \end{pmatrix}\right\|_2^2
         \leq \left\|\begin{pmatrix}
           \bw_{t-1}-\bw_{t-1}^{(i)} \\
           \bv_{t-1}-\bv_{t-1}^{(i)}
         \end{pmatrix}\right\|_2^2+8G^2\eta_t^2.
\]
\noindent{\textbf{Case 2}.}
For the case $i_t=i$, we have Eq. \eqref{stab-gda-2}. 

We can combine the above two cases together and derive
\[
\left\|\begin{pmatrix}
           \bw_{t}-\bw_{t}^{(i)} \\
           \bv_{t}-\bv_{t}^{(i)}
         \end{pmatrix}\right\|_2^2\leq \left\|\begin{pmatrix}
           \bw_{t-1}-\bw_{t-1}^{(i)} \\
           \bv_{t-1}-\bv_{t-1}^{(i)}
         \end{pmatrix}\right\|_2^2+8G^2\eta_t^2+4\sqrt{2}G\eta_t\left\|\begin{pmatrix}
           \bw_{t-1}-\bw_{t-1}^{(i)} \\
           \bv_{t-1}-\bv_{t-1}^{(i)}
         \end{pmatrix}\right\|_2\ibb_{[i_t=i]}.
\]
We apply the above inequality recursively and derive
\[
\left\|\begin{pmatrix}
           \bw_{t}-\bw_{t}^{(i)} \\
           \bv_{t}-\bv_{t}^{(i)}
         \end{pmatrix}\right\|_2^2 \leq 8G^2\sum_{j=1}^{t}\eta_j^2+4\sqrt{2}G\sum_{j=1}^{t}\eta_j\left\|\begin{pmatrix}
           \bw_{j-1}-\bw_{j-1}^{(i)} \\
           \bv_{j-1}-\bv_{j-1}^{(i)}
         \end{pmatrix}\right\|_2\ibb_{[i_j=i]}
\]
Let $\delta_t^{(i)}$ be defined in Eq. \eqref{delta-t-i}. It then follows that
\[
\big(\delta_t^{(i)}\big)^2  \leq 8G^2\sum_{j=1}^{t}\eta_j^2+4\sqrt{2}G\delta_t^{(i)}\sum_{j=1}^{t}\eta_j\ibb_{[i_j=i]}.
\]
We can apply Lemma \ref{lem:quad} with $x=\delta_t^{(i)}$ to show that
\[
\left\|\begin{pmatrix}
           \bw_{t}-\bw_{t}^{(i)} \\
           \bv_{t}-\bv_{t}^{(i)}
         \end{pmatrix}\right\|_2\leq \delta_t^{(i)} \leq 2\sqrt{2}G\Big(\sum_{j=1}^{t}\eta_j^2\Big)^{\frac{1}{2}}+4\sqrt{2}G\sum_{j=1}^{t}\eta_j\ibb_{[i_j=i]}.
\]
We can take an average over $i$ to derive
\begin{align*}
  \frac{1}{n}\sum_{i=1}^{n}\left\|\begin{pmatrix}
           \bw_{t}-\bw_{t}^{(i)} \\
           \bv_{t}-\bv_{t}^{(i)}
         \end{pmatrix}\right\|_2 & \leq \frac{2\sqrt{2}G}{n}\sum_{i=1}^{n}\Big(\sum_{j=1}^{t}\eta_j^2\Big)^{\frac{1}{2}}+\frac{4\sqrt{2}G}{n}\sum_{i=1}^{n}\sum_{j=1}^{t}\eta_j\ibb_{[i_j=i]} \\
  % & \leq 2\sqrt{2}G\Big(\frac{1}{n}\sum_{i=1}^{n}\sum_{k=1}^{t}\eta_k^2\Big)^{\frac{1}{2}}+\frac{2\sqrt{2}}{n}\sum_{i=1}^{n}\sum_{k=1}^{t}\eta_k\ibb_{[i_k=i]} \\
   & = 2\sqrt{2}G\Big(\sum_{j=1}^{t}\eta_j^2\Big)^{\frac{1}{2}}+\frac{4\sqrt{2}G}{n}\sum_{j=1}^{t}\eta_j.
\end{align*}
It follows from the convexity of a norm that
\begin{align*}
  \frac{1}{n}\sum_{i=1}^{n}\left\|\begin{pmatrix}
           \bar{\bw}_{T}-\bar{\bw}_{T}^{(i)} \\
           \bar{\bv}_{T}-\bar{\bv}_{T}^{(i)}
         \end{pmatrix}\right\|_2 & \leq \frac{2\sqrt{2}G}{n}\sum_{i=1}^{n}\Big(\sum_{j=1}^{t}\eta_j^2\Big)^{\frac{1}{2}}+\frac{4\sqrt{2}G}{n}\sum_{i=1}^{n}\sum_{j=1}^{t}\eta_j\ibb_{[i_j=i]} \\
  % & \leq 2\sqrt{2}G\Big(\frac{1}{n}\sum_{i=1}^{n}\sum_{k=1}^{t}\eta_k^2\Big)^{\frac{1}{2}}+\frac{2\sqrt{2}}{n}\sum_{i=1}^{n}\sum_{k=1}^{t}\eta_k\ibb_{[i_k=i]} \\
   & = 2\sqrt{2}G\Big(\sum_{j=1}^{t}\eta_j^2\Big)^{\frac{1}{2}}+\frac{4\sqrt{2}G}{n}\sum_{j=1}^{t}\eta_j.
\end{align*}
The proof is completed.
\end{proof}

Now, we can combine the stability bounds in Theorem~\ref{thm:stab-sgda} and Theorem~\ref{thm:stab-gen-min-max} to develop generalization bounds for weak PD risk bounds and primal population risk bounds.
\begin{proof}[Proof of Theorem~\ref{thm:gen-sgda}]
(a) Note Theorem~\ref{thm:stab-sgda} shows, for smooth case, that MC-SGDA is  on-average $\epsilon$-argument stable with  
$
\epsilon\leq 4G\big(\frac{1}{n}\sum_{j=1}^{T}\eta_j^2\big)^{1/2}+\frac{8\sqrt{2}G}{n}\sum_{j=1}^{T}\eta_j.
$ 
We can combine the above stability bound with Part (a) in Theorem~\ref{thm:stab-gen-min-max} and get the desired result. Part (b) can be proved in a similar way by combining Part (b) in Theorem~\ref{thm:stab-sgda} and Part (a) in Theorem~\ref{thm:stab-gen-min-max}.
\end{proof}

\begin{proof}[Proof of Theorem~\ref{thm:gen-sgda-primal}]
(a) For smooth case, Theorem~\ref{thm:stab-sgda} implies that MC-SGDA is on-average $\epsilon$-argument stable with  
$
\epsilon\leq 4G\big(\frac{1}{n}\sum_{j=1}^{T}\eta_j^2\big)^{1/2}+\frac{8\sqrt{2}G}{n}\sum_{j=1}^{T}\eta_j.
$ 
Plugging this stability bound back into Part (b) in Theorem~\ref{thm:stab-gen-min-max} yields the desired result. Part (b) can be directly proved by combining Part (b) in   Theorem~\ref{thm:stab-sgda} and Part (b) in Theorem~\ref{thm:stab-gen-min-max}. 
\end{proof}

\subsection{Optimization Error for MC-SGDA}\label{appendix-opt-sgda}
%Recall $\bar{\bw}_T=\frac{\sum_{j=1}^T \eta_j\bw_j}{\sum_{j=1}^T \eta_j}  \text{ and }   \bar{\bv}_T=\frac{\sum_{j=1}^T \eta_j\bv_j}{\sum_{j=1}^T \eta_j}.$ 
We now develop convergence rates on MC-SGDA~for convex-concave problems. We consider bounds both in expectation and with high probability. To this aim,  we decompose 
$\max_{\bv \in\V} F_S(\bar{\bw}_T,\bv) -   \min_{\bw\in \W} F_S(\bw, \bar{\bv}_T) $ into two parts: $ 
     \frac{1}{T}\sum_{j=1}^T F_S(\bw_j,  \bv_j) -   \min_{\bw\in \W} F_S(\bw, \bar{\bv}_T) $ and $  \max_{\bv \in\V} F_S(\bar{\bw}_T,\bv)- \frac{1}{T}\sum_{j=1}^T F_S(\bw_j , \bv_j),$ which are estimated separately.   
\begin{theorem}\label{thm:opt-sgda}
    Suppose Assumptions~\ref{ass:Markov-chain} and \ref{ass:lipschitz} hold.  Assume for all $z$, the function $(\bw,\bv)\mapsto f(\bw,\bv;z)$ is convex-concave.  Let $\A$ be MC-SGDA with $T$ iterations, and $\{\bw_j,  \bv_j  \}_{j=1}^T$ be the sequence produced by MC-SGDA with $\eta_j\equiv\eta$. Let $D_\bw$ and $D_\bv$ be the diameter of $\W$ and $\V$ respectively, and $D=D_\bw+D_\bv$. For any $j\in[T]$, let
    \begin{equation}\label{eq:def-kj-sgda}
        k_j\!=\!\min\!\Big\{\!\max\Big\{\Big\lceil \frac{\log(2C_PDn j^2)}{\log(1/\lambda(P))}\Big\rceil, K_P \Big\}, j\Big\}.
    \end{equation}
Then the following inequality holds 
\vspace*{-3mm}
  \begin{multline*}
    \E_\A\big[\max_{\bv \in\V} F_S(\bar{\bw}_T,\bv) -   \min_{\bw\in \W} F_S(\bw, \bar{\bv}_T) \big]\le G^2 \eta + \frac{D^2}{2T\eta} \\  
    + \frac{ 2GK_PD+12G^2\eta\sum_{j=1}^T k_j + G \sum_{j=K_P}^T 1/{j^2} }{ T}. 
\end{multline*}
Furthermore, suppose Assumption~\ref{ass:reversible} holds. Then selecting $\eta\asymp(T\log(T))^{-1/2}$ implies
\vspace*{-1mm}\[\E_\A\big[\max_{\bv \in\V} F_S(\bar{\bw}_T,\bv) - \min_{\bw\in \W} F_S(\bw, \bar{\bv}_T) \big]  = \O\big( \sqrt{\log(T)}/\big(\sqrt{T}   \log(1/\lambda(P))\big)\big).\]
\end{theorem}
\begin{proof}[Proof of Theorem~~\ref{thm:opt-sgda}] 
To estimate $ \E_\A\big[\max_{\bv \in\V} F_S(\bar{\bw}_T,\bv) -   \min_{\bw\in \W} F_S(\bw, \bar{\bv}_T) \big]$, we use the following decomposition
\begin{align}\label{eq:sgda-opt-decom}
     &\E_\A\big[\max_{\bv \in\V} F_S(\bar{\bw}_T,\bv) -   \min_{\bw\in \W} F_S(\bw, \bar{\bv}_T) \big]\nonumber\\
     &=  \E_\A\Big[ \frac{1}{T}\sum_{j=1}^T F_S(\bw_j,  \bv_j) -   \min_{\bw\in \W} F_S(\bw, \bar{\bv}_T) \Big] + \E_{\A}\Big[\max_{\bv \in\V} F_S(\bar{\bw}_T,\bv)- \frac{1}{T}\sum_{j=1}^T F_S(\bw_j , \bv_j)\Big]. 
\end{align}
Consider the first term in \eqref{eq:sgda-opt-decom}. Let $k_j=\min\Big\{\max\Big\{\Big\lceil \frac{\log(2C_P(D_\bw+D_\bv) n j^2)}{\log(1/\lambda(P))}\Big\rceil, K_P \Big\}, j\Big\}$ and $\bw^{*}_S=\arg\min_{\bw\in \W} F_S(\bw,\bar{\bv}_T)$.  The concavity of $F_S(\bw, \cdot )$ implies 
\begin{align}\label{eq:sgda-opt-w}
    &\E_\A\Big[ \frac{1}{T}\sum_{j=1}^T F_S(\bw_j,  \bv_j) -     F_S(\bw^{*}_S, \bar{\bv}_T) \Big]\nonumber\\
    &\le \E_\A \Big[ \frac{1}{T}\sum_{j=1}^T\big( F_S(\bw_j,  \bv_j) -    F_S(\bw^{*}_S,  \bv_j)\big) \Big]\nonumber\\
    &=\!\E_\A\!\Big[ \frac{1}{T}\!\sum_{j=1}^T\big( F_S(\bw_j,  \bv_j)\!-\!F_S(\bw_{j-k_j},  \bv_j)\big)\Big]\!\!+\!\! \E_\A\!\Big[\frac{1}{T}\!\sum_{j=1}^T\big( F_S(\bw_{j-k_j},  \bv_j)\!-\!    F_S(\bw_{j-k_j},  \bv_{j-k_j})\big) \Big] \nonumber\\
    & +\!\E_\A\!\Big[ \frac{1}{T}\!\sum_{j=1}^T\big( F_S(\bw_{j-k_j},  \bv_{j-k_j})\!-\!    F_S(\bw^{*}_S,  \bv_{j-k_j})\big) \Big] \!+\!\E_\A\!\Big[ \frac{1}{T}\!\sum_{j=1}^T\big( F_S(\bw^{*}_S,  \bv_{j-k_j})\!-\!    F_S(\bw^{*}_S,  \bv_j)\big) \Big]\nonumber\\
    &\le \frac{3G^2\eta}{T}\sum_{j=1}^T k_j + \E_\A \Big[ \frac{1}{T}\sum_{j=1}^T\big( F_S(\bw_{j-k_j},  \bv_{j-k_j}) -    F_S(\bw^{*}_S,   \bv_{j-k_j})\big) \Big],
\end{align}
where the last inequality used the Lipschitz continuity of $f(\cdot,\bv;z)$ and $f(\bw,\cdot;z)$  and the fact that  $\|\bw_{j}-\bw_{j-k_j}\|_2\le G\eta k_j$ and $\|\bv_{j}-\bv_{j-k_j}\|_2\le G\eta k_j$. 

Now, we turn to estimate the term  $\E_\A \big[ \frac{1}{T}\sum_{j=1}^T\big( F_S(\bw_{j-k_j},  \bv_{j-k_j}) -    F_S(\bw^{*}_S,   \bv_{j-k_j})\big) \big]$. 
Note that 
\begin{align}\label{eq:opt-sgda-hp-1}
    &\E_{i_j}\big[ f(\bw_{j-k_j},\bv_{j-k_j};z_{i_j}) - f(\bw^{*}_S,\bv_{j-k_j};z_{i_j}) | (\bw_0,\bv_0),\ldots,(\bw_{j-k_j},\bv_{j-k_j}), z_{i_1},\ldots,z_{i_{j-k_j}} \big]\nonumber\\
    &= \sum_{i=1}^n \big[ f(\bw_{j-k_j},\bv_{j-k_j};z_i) - f(\bw^{*}_S,\bv_{j-k_j};z_i)\big] \cdot [P^{k_j}]_{i_{j-k_j}, i }\nonumber\\
    &= \big( F_S(\bw_{j-k_j},\bv_{j-k_j}) - F_S( \bw^{*}_S, \bv_{j-k_j}) \big)\nonumber\\
    &\quad + \sum_{i=1}^n \Big( [P^{k_j}]_{i_{j-k_j}, i } -\frac{1}{n} \Big)  \cdot  \big[ f(\bw_{j-k_j},\bv_{j-k_j};z_{i }) - f(\bw^{*}_S,\bv_{j-k_j};z_{i })\big].
\end{align}
Summing over $j$ and taking total expectation we have
\begin{align}\label{eq:sgda-opt-1}
       &\sum_{j=1}^T \E_\A \big[ F_S(\bw_{j-k_j},\bv_{j-k_j}) -   F_S( \bw^{*}_S, \bv_{j-k_j}) \big]\nonumber\\
       &= \sum_{j=1}^T \E_{\A}\big[ f(\bw_{j-k_j},\bv_{j-k_j};z_{i_j}) - f(\bw^{*}_S,\bv_{j-k_j};z_{i_j})  \big] \nonumber\\
       &\quad  + \sum_{j=1}^T \E_\A \Big[ \sum_{i=1}^n \Big( \frac{1}{n}  - [P^{k_j}]_{i_{j-k_j}, i }\Big)  \cdot  \big[ f(\bw_{j-k_j},\bv_{j-k_j};z_{i}) - f(\bw^{*}_S,\bv_{j-k_j};z_{i })\big]\Big].
\end{align} 
Similar to before, 
according to MC-SGDA update rule \eqref{eq:MCSGDA}, for any $j$ and $1\le k_j\le j$
\begin{align*}
    &\|\bw_j - \bw_S^{*}\|_2^2\\
    &\le \| \bw_{j-1} -\eta  \partial_\bw f(\bw_{j-1},\bv_{j-1};z_{i_j}) - \bw_S^{*} \|_2^2 \nonumber \\  & =\|\bw_{j-1} -  \bw_S^{*} \|_2^2 -2\eta  \langle  \bw_{j-1} -  \bw_S^{*}, \partial_\bw f(\bw_{j-1},\bv_{j-1};z_{i_j}) \rangle + \eta^2  \|\partial_\bw f(\bw_{j-1},\bv_{j-1};z_{i_j})\|_2^2\nonumber\\
    &\le \|\bw_{j-1} -  \bw_S^{*} \|_2^2 -2\eta  \big( f(\bw_{j-1},\bv_{j-1};z_{i_j}) - f(\bw_S^{*},\bv_{j-1};z_{i_j}) \big) + G^2\eta^2  \nonumber\\
    &= \|\bw_{j-1} -  \bw_S^{*}\|_2^2 -2\eta  \big( f(\bw_{j-k_j},\bv_{j-k_j};z_{i_j}) - f(\bw_S^{*}, \bv_{j-k_j};z_{i_j}) \big)\nonumber\\
    &\quad + 2\eta  \big( f(\bw_{j-k_j},\bv_{j-k_j};z_{i_j})  - f(\bw_{j-k_j},\bv_{j-1};z_{i_j})\!+\! f(\bw_{j-k_j},\bv_{j-1};z_{i_j})- f(\bw_{j-1},\bv_{j-1};z_{i_j})   \big)  \nonumber\\
    &\quad +  2\eta  \big(  f(\bw_S^{*},\bv_{j-1};z_{i_j}) - f(\bw_S^{*}, \bv_{j-k_j};z_{i_j}) \big) + G^2\eta^2  \nonumber\\
    &\le \|\bw_{j-1} -  \bw_S^{*} \|_2^2 -2\eta  \Big( f(\bw_{j-k_j},\bv_{j-k_j};z_{i_j}) - f(\bw_S^{*}, \bv_{j-k_j};z_{i_j}) \Big)  + 6G^2\eta^2 k_j   +G^2 \eta^2, 
\end{align*}
where the second inequality is due to the convexity of $f(\cdot,\bv;z)$, and the last inequality used the fact that  $\|\bw_{j-k_j} -  \bw_{j-1} \|_2 \le   G \eta k_j$ and $\|\bv_{j-k_j} -  \bv_{j-1} \|_2 \le   G \eta k_j$.  
Rearranging the above inequality and taking a summation of both sides  over $j$, we get
\begin{align}\label{eq:sgda-opt-2}
   \sum_{j=1}^T \Big( f(\bw_{j-k_j},\bv_{j-k_j};z_{i_j}) - f(\bw_S^{*}, \bv_{j-k_j};z_{i_j}) \Big) \le \frac{ D_\bw^2 + 6G^2\eta^2 \sum_{j=1}^T k_j   + T G^2 \eta^2  }{2\eta} . 
\end{align}
 
Now, we consider the second term in \eqref{eq:sgda-opt-1}. 
Recall that $k_j=\min\Big\{\max\Big\{\Big\lceil \frac{\log(2C_P(D_\bw+D_\bv) n j^2)}{\log(1/\lambda(P))}\Big\rceil, K_P \Big\}, j\Big\}$. 
If $j\ge K_P$, then according to Lemma~\ref{lem:difference-stationary}, for any $i,i'\in [n]$ we have
\[ \left| \frac{1}{n} -[P^{k_j}]_{i,i'} \right| \le \frac{1}{2(D_\bw+D_\bv)nj^2}. \]
Combining this with  Assumption~\ref{ass:lipschitz} we get
\begin{align}\label{eq:sgda-opt-3}
    &\sum_{j=K_P}^T    \sum_{i=1}^n \Big( \frac{1}{n}- [P^{k_j}]_{i_{j-k_j}, i } \Big)  \cdot  \big[ f(\bw_{j-k_j},\bv_{j-k_j};z_{i }) - f(\bw^{*}_S,\bv_{j-k_j};z_{i })\big] \nonumber\\
    &\le GD_\bw \sum_{j=K_P}^T \sum_{i=1}^n \Big| [P^{k_j}]_{i_{j-k_j}, i } -\frac{1}{n} \Big|  \le G \sum_{j=K_P}^T \frac{1}{2j^2}. 
\end{align}
 For $j<K_P$, there holds
 \begin{align}\label{eq:sgda-opt-4}
     \sum_{j=1}^{K_P}   \sum_{i=1}^n \Big( \frac{1}{n}-[P^{k_j}]_{i_{j-k_j}, i } \Big)  \cdot  \big[ f(\bw_{j-k_j},\bv_{j-k_j};z_{i }) - f(\bw^{*}_S,\bv_{j-k_j};z_{i })\big] \le 2GK_P D_\bw,  
 \end{align}
where we use $\sum_{i=1}^n [P^{k_j}]_{i_{j-k_j},i}=1$ and the Lipschitz continuity of $f(\cdot, \bv)$. 
 Combining \eqref{eq:sgda-opt-3} and \eqref{eq:sgda-opt-4}  together, we get
 \begin{align}\label{eq:sgda-opt-5}
     &\sum_{j=1}^{T}   \sum_{i=1}^n \Big( \frac{1}{n}-[P^{k_j}]_{i_{j-k_j}, i }  \Big)  \cdot  \big[ f(\bw_{j-k_j},\bv_{j-k_j};z_{i }) - f(\bw^{*}_S,\bv_{j-k_j};z_{i })\big]\nonumber\\& \le 2GK_P D_\bw + G \sum_{j=K_P}^T \frac{1}{2j^2}.  
 \end{align}
Putting \eqref{eq:sgda-opt-2} and  \eqref{eq:sgda-opt-5} back into \eqref{eq:sgda-opt-1}, we  obtain
\begin{align*}
       &\sum_{j=1}^T \E_\A \big[ F_S(\bw_{j-k_j},\bv_{j-k_j}) -   F_S( \bw^{*}_S, \bv_{j-k_j}) \big] \\&\le  \frac{ D_\bw^2 + 6G^2\eta^2 \sum_{j=1}^T k_j   + T G^2 \eta^2  }{2\eta}  +  2GK_P D_\bw + G \sum_{j=K_P}^T \frac{1}{2j^2}.
\end{align*}  
Finally, plugging the above inequality back into \eqref{eq:sgda-opt-w}, we have
\begin{align*}
     &\E_\A\Big[ \frac{1}{T}\sum_{j=1}^T F_S(\bw_j,  \bv_j) -     \min_{\bw\in\W}F_S(\bw , \bar{\bv}_T) \Big]\\&\le  \frac{ 6G^2\eta \sum_{j=1}^T k_j  }{T} +  \frac{ 2GK_P D_\bw + G \sum_{j=K_P}^T \frac{1}{2j^2}}{T} + \frac{D_\bw^2}{2T\eta} + \frac{G^2\eta}{2}.
\end{align*}
In a similar way, we can show 
\begin{align*}
     &\E_\A\Big[  \max_{\bv \in \V}F_S(\bar{\bw}_T,\bv) -  \frac{1}{T}\sum_{j=1}^T F_S(\bw_j,  \bv_j) \Big]\\&\le  \frac{ 6G^2\eta \sum_{j=1}^T k_j  }{T} +  \frac{ 2GK_P D_\bv + G \sum_{j=K_P}^T \frac{1}{2j^2}}{T} + \frac{D^2_\bv}{2T\eta} + \frac{G^2\eta}{2}.
\end{align*}
Combining the above two bounds together, we get
\begin{align*}
     &\E_\A\Big[ \max_{\bv \in \V}F_S(\bar{\bw}_T,\bv)   -     \min_{\bw\in\W}F_S(\bw , \bar{\bv}_T) \Big]\\
     &\le   G^2\eta +  \frac{(D_\bw+D_\bv)^2}{2T\eta} +  \frac{  2GK_P (D_\bw+D_\bv)  +  12G^2\eta \sum_{j=1}^T k_j + G \sum_{j=K_P}^T \frac{1}{ j^2} }{T}  .
\end{align*} 
The first part of theorem is proved. 
Now, we turn to the second part of theorem. 
Let $K=\frac{1}{\sqrt{2C_P (D_\bw+D_\bv) n \lambda(P)^{K_P}} }$ and $\eta\asymp 1/\sqrt{T\log(T)}$. 
If $j< K$, we have
\[ \sum_{j=1}^{K-1} k_j \eta^2\le K K_P \eta^2= \frac{K_P}{T\log(T) \sqrt{2C_P (D_\bw+D_\bv) n \lambda(P)^{K_P}} }. \]
If $j\ge K$, there holds
\begin{align*}
     \sum_{j=K}^{T} k_j \eta^2&\le \frac{1}{\log( 1/ \lambda(P) )} \Big[ \sum_{j=K}^{T}  \log( 2C_P (D_\bw+D_\bv)   )   \eta^2\!+\! \sum_{j=K}^{T} \log(n) \eta^2 + 2\sum_{j=K}^{T} \log(j) \eta^2\Big]\!+\!T\eta^2\\
     &=\O\Big( \frac{1}{\log( 1/ \lambda(P) )} \Big).
\end{align*}
Combining the above two cases together, we get
\begin{align}\label{eq:kj-sum-sgda}
   \sum_{j=1}^T k_j \eta^2=\O\Big(  \frac{K_P}{T\log(T) \sqrt{ C_P   n \lambda(P)^{K_P}} } + \frac{1}{\log( 1/ \lambda(P) )}\Big).
\end{align} 
Then we obtain
\begin{align*} 
       &\E_\A\Big[ \max_{\bv \in \V}F_S(\bar{\bw}_T,\bv)   -     \min_{\bw\in\W}F_S(\bw , \bar{\bv}_T) \Big]\\
       &=\O\Big( \frac{K_P}{T}+  \frac{1+\sum_{j=1}^T k_j \eta^2}{T\eta} +   \eta     \Big)\\
      &=\O\Big(\frac{\sqrt{\log(T)}}{\sqrt{T}\log(1/\lambda(P))} + \frac{K_P}{T\min\{1, \sqrt{ C_P   n \lambda(P)^{K_P}T\log(T)}\}} \Big).
\end{align*}  
Note Assumption~\ref{ass:reversible} implies $K_P=0$. Then we get
\begin{align*} 
      \E_\A\Big[ \max_{\bv \in \V}F_S(\bar{\bw}_T,\bv)   -     \min_{\bw\in\W}F_S(\bw , \bar{\bv}_T) \Big]  =\O\Big(\frac{\sqrt{\log(T)}}{\sqrt{T}\log(1/\lambda(P))}  \Big).
\end{align*}  
This completes the proof. 
\end{proof}
%The following theorem shows high probability bounds for optimization error. % match the above bounds in expectation up to a constant factor.
\begin{theorem}[High-probability bound]\label{thm:opt-hp-sgda}
Suppose Assumptions~\ref{ass:Markov-chain}, \ref{ass:reversible} and \ref{ass:lipschitz} hold.  Assume for all $z$, the function $(\bw,\bv)\mapsto f(\bw,\bv;z)$ is convex-concave.  Let $\{\bw_j,  \bv_j  \}_{j=1}^T$ be produced MC-SGDA with $ \eta_j\equiv\eta\asymp 1/\sqrt{T\log(T)}$. 
Assume $\sup_{z\in \Z} f(\bw,\bv;z) \le B$ with some $B>0$ for any $\bw\in \W $ and $\bv\in \V$. 
Let $\gamma\in (0,1)$. Then with probability  $1-\gamma$ 
\vspace*{-3mm}
\begin{align*}
     \max_{\bv \in\V} F_S(\bar{\bw}_T,\bv) -   \min_{\bw\in \W} F_S(\bw, \bar{\bv}_T)=   \O\Big(    \frac{\sqrt{\log(T)}}{\sqrt{T}}\big(\frac{1}{\log(1/{\lambda(P)})} +B\sqrt{\log(1/\gamma)}\big)   \Big). 
\end{align*}
\end{theorem}
\begin{proof}[Proof of Theorem~\ref{thm:opt-hp-sgda}]
Note that     
    \begin{align}\label{eq:sgda-opt-hp-decom}
      &\max_{\bv \in\V} F_S(\bar{\bw}_T,\bv) -   \min_{\bw\in \W} F_S(\bw, \bar{\bv}_T)\nonumber \\
      &=  \Big[ \frac{1}{T}\sum_{j=1}^T F_S(\bw_j,  \bv_j) -   \min_{\bw\in \W} F_S(\bw, \bar{\bv}_T) \Big]   + \Big[   \max_{\bv \in\V} F_S(\bar{\bw}_T,\bv)- \frac{1}{T}\sum_{j=1}^T F_S(\bw_j , \bv_j) \Big]. 
\end{align}
Consider the first term in \eqref{eq:sgda-opt-hp-decom}. 
Let  $k_j=\min\Big\{\max\Big\{\Big\lceil \frac{\log(2C_P(D_\bw+D_\bv) n j^2)}{\log(1/\lambda(P))}\Big\rceil, K_P \Big\}, j\Big\}$ and $\bw^{*}_S=\arg\min_{\bw\in \W} F_S(\bw,\bar{\bv}_T)$. Similar to \eqref{eq:sgda-opt-w},  we can show
\begin{align}\label{eq:opt-sgda-hp-w}
       \frac{1}{T}\sum_{j=1}^T F_S(\bw_j,  \bv_j) -     F_S(\bw^{*}_S, \bar{\bv}_T) 
    &\le \frac{3G^2\eta}{T}\sum_{j=1}^T k_j +    \frac{1}{T}\sum_{j=1}^T\big( F_S(\bw_{j-k_j},  \bv_{j-k_j}) -    F_S(\bw^{*}_S,   \bv_{j-k_j})\big)  .
\end{align}

Let 
$ 
\xi_j =   f(\bw_{j-k_j},\bv_{j-k_j}; z_{i_j}) - f(\bw_S^{*},\bv_{j-k_j};z_{i_j} )   .$
Observe that  $|\xi_j-\E_{i_j}[ \xi_j]|\le 2 B $.  
Then, applying Lemma~\ref{lem:martingle} implies, with probability at least $1-\gamma/2$, that
\begin{equation}\label{eq:opt-sgda-hp-2}
	\sum_{j=1}^T \E_{i_j}[\xi_j] - \sum_{j=1}^T \xi_j \le 2B \sqrt{2T \log(2/\gamma)}. 
\end{equation}  
Combining \eqref{eq:opt-sgda-hp-1} and 	\eqref{eq:opt-sgda-hp-2} together, we get
\begin{align*}
	&\sum_{j=1}^T   [F_S(\bw_{j-k_j},  \bv_{j-k_j})\!-\!F_S(\bw^{*}_S, \bv_{j-k_j})]\!+\!\sum_{j=1}^T   \sum_{i=1}^n  \big( [P^{k_j}]_{i_{j-k_j},i}\!-\!  \frac{1}{n}\big) [ f(\bw_{j-k_j},  \bv_{j-k_j};z_{i})\!-\! f(\bw_S^{*},\bv_{j-k_j} ; z_{i} ) ]\\
	&= \sum_{j=1}^T \E_{i_j}[  f(\bw_{j-k_j},\bv_{j-k_j};z_{i_j})\!-\!f(\bw^{*}_S,\bv_{j-k_j} ; z_{i_j} )| \{\bw_{0},\bv_0\},\ldots,\{\bw_{j-k_j},\bv_{j-k_j}\}, z_{i_1},\ldots, z_{i_{j-k_j}}  ]\\  
	& \le  \sum_{j=1}^T   [f(\bw_{j-k_j},\bv_{j-k_j};z_{i_j}) - f(\bw^{*}_S,\bv_{j-k_j} ; z_{i_j} )] + 2B \sqrt{2T \log(2/\gamma)}
	\end{align*}
	with probability at least $1-\gamma/2$.
Putting \eqref{eq:sgda-opt-2} and \eqref{eq:sgda-opt-5} back into the above inequality, we obtain
\begin{align}\label{eq:opt-sgda-hp-3}	&\sum_{j=1}^T [F_S(\bw_{j-k_j},\bv_{j-k_j}) - F_S(\bw^{*}_S,\bv_{j-k_j})]\nonumber\\
&\le  \sum_{j=1}^T     \sum_{i=1}^n \Big(  \frac{1}{n}- [P^{k_j}]_{i_{j-k_j},i} \Big) [ f(\bw_{j-k_j},\bv_{j-k_j};z_{i})-f(\bw^{*}_S,\bv_{j-k_j} ; z_{i} ) ] \nonumber\\
&\quad  + \sum_{j=1}^t   [f(\bw_{j-k_j},\bv_{j-k_j};z_{i_j})-f(\bw^{*}_S,\bv_{j-k_j} ; z_{i_j} )]  + 2B \sqrt{2T \log(2/\gamma)} \nonumber\\
	& \le \frac{ D_\bw^2+ 6G^2\eta^2 \sum_{j=1}^T k_j + TG^2\eta^2 }{2\eta} + 2GK_P D_\bw + G\sum_{j=K_P}^T \frac{1}{2j^2} + 2B \sqrt{2T \log(2/\gamma)} . 
	\end{align}

Now, plugging \eqref{eq:opt-sgda-hp-3} back into \eqref{eq:opt-sgda-hp-w}, with probability at least $1-\gamma/2$, there holds
 \begin{align*}
     &\frac{1}{T}\sum_{j=1}^T F_S(\bw_j,  \bv_j) -    \min_{\bw\in\W} F_S(\bw, \bar{\bv}_T) \\
    &\le \frac{6G^2\eta \sum_{j=1}^T k_j +   2GK_P D_\bw + G\sum_{j=K_P}^T \frac{1}{2j^2}}{T}+  
    \frac{ D_\bw^2 }{2T\eta} +  \frac{G^2\eta}{2}  + \frac{ 2B\sqrt{2 \log(\frac{2}{\gamma})}}{\sqrt{T}} .
\end{align*}
In a similar way, we can show, with probability at least $1-\gamma/2$, that
 \begin{align*}
     &\max_{\bv \in\V} F_S(\bar{\bw}_T,\bv)- \frac{1}{T}\sum_{j=1}^T F_S(\bw_j , \bv_j) \\
    &\le \frac{6G^2\eta \sum_{j=1}^T k_j +   2GK_P D_\bv + G\sum_{j=K_P}^T \frac{1}{2j^2}}{T} + 
    \frac{ D_\bv^2 }{2T\eta} +  \frac{G^2\eta}{2} + \frac{ 2B\sqrt{2 \log(\frac{2}{\gamma})}}{\sqrt{T}} .
\end{align*}
Combining the above two inequalities together, with probability at least $1-\gamma$, we  get
\begin{align*}
    & \max_{\bv \in\V} F_S(\bar{\bw}_T,\bv)-    \min_{\bw\in\W} F_S(\bw, \bar{\bv}_T)\\
    &\le \frac{12G^2\eta \sum_{j=1}^T k_j\!+\!2GK_P(D_\bw\!+\!D_\bv)\!+\! G\sum_{j=K_P}^T \frac{1}{j^2}}{T}\!+\! 
    \frac{ (D_\bw+D_\bv)^2 }{2T\eta}\!+\!   G^2\eta\!+\!\frac{ 4B\sqrt{2 \log(\frac{2}{\gamma})}}{\sqrt{T}}.
\end{align*}
Further, if we select $\eta\asymp 1/{\sqrt{T\log(T)}}$, according to Eq.\eqref{eq:kj-sum-sgda} we have 
\begin{align*}
      &\max_{\bv \in\V} F_S(\bar{\bw}_T,\bv)-    \min_{\bw\in\W} F_S(\bw, \bar{\bv}_T)\\
      &=\O\Big(\frac{K_P}{T} + \eta + \frac{1+\sum_{j=1}^T k_j\eta^2}{T\eta} +    \frac{  B\sqrt{  \log(1/{\gamma})}}{\sqrt{T}} \Big)\\
      &=\O\Big(   \frac{\sqrt{\log(T)} }{\sqrt{T}}\Big( \frac{1}{\log(1/\lambda(P))}+  B\sqrt{ \log(1/\gamma) } \Big) + \frac{K_P}{T \min\{1, \sqrt{ C_P   n \lambda(P)^{K_P}T\log(T)}\} }   \Big)\\
      &=\O\Big(   \frac{\sqrt{\log(T)} }{\sqrt{T}}\Big( \frac{1}{\log(1/\lambda(P))} +    B\sqrt{ \log(1/\gamma) } \Big)   \Big) ,
\end{align*}
where in the last equality we used $K_P=0$ due to $P=P^\top$.
This completes the proof. 
\end{proof}

\subsection{Proofs of Theorem~\ref{thm:weakPD-risk} and Theorem \ref{thm:excess-primal}}\label{sec:thm10-11}
\begin{proof}[Proof of Theorem \ref{thm:weakPD-risk}]
We can choose $\eta$ such that $T\eta^2\leq1/(2L^2)$ and therefore Theorem \ref{thm:gen-sgda} applies.
According to part (a) of Theorem~\ref{thm:gen-sgda} we have
\[\triangle^w(\bar{\bw}_T,\bar{\bv}_T)-\triangle^w_\emp(\bar{\bw}_T,\bar{\bv}_T)\leq  \frac{4G^2\sqrt{T}\eta}{\sqrt{n}}+\frac{8\sqrt{2}G^2T\eta}{n} .\]
Combining the above inequality with Theorem~\ref{thm:opt-sgda} together, we get
\begin{align*}
    &\triangle^w(\bar{\bw}_T,\bar{\bv}_T)=\triangle^w(\bar{\bw}_T,\bar{\bv}_T)-\triangle^w_\emp(\bar{\bw}_T,\bar{\bv}_T)\\
    &\le \frac{4G^2\sqrt{T}\eta}{\sqrt{n}}+\frac{8\sqrt{2}G^2T\eta}{n} + G^2 \eta + \frac{D^2}{2T\eta}   
    + \frac{2GK_PD+12G^2\eta\sum_{j=1}^T k_j + G \sum_{j=K_P}^T 1/{j^2} }{ T} 
\end{align*}
where $D=D_\bw+D_\bv$. 
If we choose $T\asymp n$ and $\eta\asymp  ( T \log(T))^{-\frac{1}{2}}$, according to \eqref{eq:kj-sum-sgda} we get 
\[ \triangle^w(\bar{\bw}_T,\bar{\bv}_T)= \O\Big( \frac{\log(n)}{\sqrt{n}\log(1/\lambda(P))}  + \frac{K_P}{ n \min\{ 1, n \sqrt{ \log(n) C_P \lambda(P)^{K_P} } \} } \Big) .\]
Note Assumption~\ref{ass:reversible} implies $K_P=0$, the proof of part (a) is completed. 

Part (b) can be proved in a similar way (e.g., by combining part (b) of Theorem~\ref{thm:gen-sgda} and Theorem~\ref{thm:opt-sgda} together). We omit the proof for brevity.   
\end{proof}

\begin{proof}[Proof of Theorem \ref{thm:excess-primal}]
We use the following decomposition
\begin{align*}
    R(\bar{\bw}_T) - R(\bw^{*}) = & \Big(  R(\bar{\bw}_T) - R_S(\bar{\bw}_T) \Big) + \Big(  R_S(\bar{\bw}_T) - F_S( \bw^{*},  \bar{\bv}_T ) \Big)\nonumber \\
    &+ \Big(  F_S( \bw^{*},  \bar{\bv}_T ) - F(\bw^{*},  \bar{\bv}_T) \Big)+ \Big(  F(\bw^{*},  \bar{\bv}_T) - R(\bw^{*})  \Big). 
\end{align*}
Note that $ F(\bw^{*},  \bar{\bv}_T)\le  F(\bw^{*},  {\bv}^{*})$. Then we have
\begin{align*}
    R(\bar{\bw}_T) - R(\bw^{*}) \le  & \Big(  R(\bar{\bw}_T) - R_S(\bar{\bw}_T) \Big) + \Big(  R_S(\bar{\bw}_T) - F_S( \bw^{*},  \bar{\bv}_T ) \Big)\nonumber\\ &+ \Big(  F_S( \bw^{*},  \bar{\bv}_T ) - F(\bw^{*},  \bar{\bv}_T) \Big). 
\end{align*}
Taking the expectation on both sides gives
\begin{align}\label{eq:decom-excess-sgda}
    \E_{S,\A}[R(\bar{\bw}_T) - R(\bw^{*})] \leq & \E_{S,\A}[  R(\bar{\bw}_T) - R_S(\bar{\bw}_T)] + \E_{S,\A}[ R_S(\bar{\bw}_T) - F_S( \bw^{*},  \bar{\bv}_T ) ] \nonumber\\
    & + \E_{S,\A}[ F_S( \bw^{*},  \bar{\bv}_T ) - F(\bw^{*},  \bar{\bv}_T)  ]. 
\end{align}
According to part (a) of  Theorem~\ref{thm:gen-sgda-primal} we know
\[\E_{S,\A}[  R(\bar{\bw}_T) - R_S(\bar{\bw}_T)]\le 4G^2(1+L/\rho)\Big(\frac{\sqrt{T}\eta}{\sqrt{n}}+\frac{2\sqrt{2}T\eta}{n} \Big).\]
Similarly, the stability bound in Theorem~\ref{thm:stab-sgda} also implies 
\[\E_{S,\A}[ F_S( \bw^{*},  \bar{\bv}_T ) - F(\bw^{*},  \bar{\bv}_T) ]\le 4G^2(1+L/\rho)\Big(\frac{\sqrt{T}\eta}{\sqrt{n}}+\frac{2\sqrt{2}T\eta}{n} \Big).\]
According to Theorem~\ref{thm:opt-sgda}, we know
\begin{align*}
    \E_{S,\A}[ R_S(\bar{\bw}_T) - F_S( \bw^{*},  \bar{\bv}_T ) ]&\le  \E_\A\big[\max_{\bv \in\V} F_S(\bar{\bw}_T,\bv) -   \min_{\bw\in \W} F_S(\bw, \bar{\bv}_T) \big]\\
    &\le G^2 \eta + \frac{D^2}{2T\eta}  
    + \frac{ 2GK_PD+12G^2\eta\sum_{j=1}^T k_j + G \sum_{j=K_P}^T 1/{j^2} }{ T}, 
\end{align*}   
where $D=D_\bw+D_\bv$. 
Putting the above three inequalities back into Eq. \eqref{eq:decom-excess-sgda}, we obtain
\begin{align*}
    \E_{S,\A}[R(\bar{\bw}_T) - R(\bw^{*})] \le & 8G^2(1+L/\rho)\Big(\frac{\sqrt{T}\eta}{\sqrt{n}}+\frac{2\sqrt{2}T\eta}{n} \Big) +  G^2 \eta + \frac{D^2}{2T\eta} \\
    &
    + \frac{ 2GK_PD+12G^2\eta\sum_{j=1}^T k_j + G \sum_{j=K_P}^T 1/{j^2} }{ T}.
\end{align*}
If we choose $T\asymp n$ and $\eta \asymp (T\log(T))^{-1/2}$, combining the above estimation with \eqref{eq:kj-sum-sgda}  implies
\begin{align*}
    \E_{S,\A}[R(\bar{\bw}_T) - R(\bw^{*})] =\O\Big( \frac{(L/\rho)\log(n)}{ \sqrt{n}\log(\lambda(P))}   + \frac{K_P}{ n \min\{1, n\sqrt{\log(n) C_P \lambda(P)^{K_P}}\} }\Big). 
\end{align*}
The above result combines with $K_P=0$ complete the proof. 
\end{proof}

\bibliographystyle{plain}
\bibliography{learning.bib}

\end{document}